\documentclass[twocolumn,journal]{IEEEtran}
\usepackage{siunitx}
\usepackage{makecell}
\usepackage{easyReview}
\usepackage{amssymb}
\usepackage{amsmath}
\usepackage{amsfonts}
\usepackage{graphicx}
\usepackage{color}
\usepackage{caption}
\usepackage{multirow}
\usepackage{diagbox}
\usepackage{lineno}
\usepackage{subcaption}
\usepackage{mathrsfs}
\usepackage{amscd}
\usepackage{url}
\usepackage{enumitem}
\usepackage{booktabs}
\usepackage{tabularx}
\usepackage{makecell}
\usepackage{bm}
\usepackage{bbm}
\graphicspath{{Fig/}}  
\usepackage{tikz}
\usetikzlibrary{shadows, positioning, fit, backgrounds, calc, shapes.geometric, arrows.meta}
\usepackage{dashrule}
\newtheorem{theorem}{Theorem}[section]

\newtheorem{algorithm}{Algorithm}[section]

\newtheorem{defi}{Definition}[section]

\newtheorem{lem}{Lemma}[section]

\newtheorem{remark}{Remark}[section]

\newenvironment{proof}[1][Proof]{\noindent\textbf{#1.} }{\ \hfill\rule{0.3em}{0.5em}}

  \newcommand{\A}{\mathcal{A}}
\newcommand{\B}{\mathcal{B}}

\newcommand{\E}{\mathcal{E}}   
\newcommand{\F}{\mathcal{F}}

\newcommand{\mH}{\mathcal{H}}

  \newcommand{\bL}{\mathbb{L}}

\newcommand{\N}{\mathcal{N}}

\newcommand{\R}{\mathbb{R}}    
\newcommand{\mS}{\mathcal{S}}

  \newcommand{\W}{\mathcal{W}}
\newcommand{\X}{\mathcal{X}}
\newcommand{\Y}{\mathcal{Y}}
\newcommand{\Z}{\mathcal{Z}}

\newcommand{\rank}{\operatorname{rank}}

\usepackage{algorithm}
\usepackage{algpseudocode}
\usepackage[algo2e,ruled,linesnumbered]{algorithm2e}

\usepackage{hyperref}

\ifCLASSINFOpdf
\else
\fi
\begin{document}

\title{Hyperspectral Anomaly Detection via Group Sparse Low-Rank Tensor Factorization With Automatic Anomaly Grouping}

%
%
%

\author{
	Quan Yu,
	Yu-Hong Dai,
	Xiongjun Zhang\thanks{*Corresponding author: Xiongjun Zhang.}%
	\thanks{This work was supported in part by the National Natural Science Foundation of China under Grant 12171189, Grant 12021001, and Grant 92473208; in part by the Hubei Provincial Natural Science Foundation of China under Grant
		2025AFB966; and in part by the Fundamental Research Funds for the Central
		Universities under Grant XJ2026000601.}
	\thanks{Quan Yu, and Xiongjun Zhang are with the School of Mathematics and Statistics, Central China Normal University, Wuhan 430079, China, and also with the Key Laboratory of Nonlinear Analysis and Applications (Ministry of Education), Central China Normal University (e-mail: quanyu@ccnu.edu.cn; xjzhang@ccnu.edu.cn).}
	\thanks{Yu-Hong Dai is with the State Key Laboratory of Mathematical Sciences, Academy of Mathematics and Systems Science, Chinese Academy of Sciences, Beijing 100190, China, and also with the School of Mathematical Sciences, University of Chinese Academy of Sciences, Beijing 100049, China (e-mail: dyh@lsec.cc.ac.cn).}
}

%
%

\markboth{}
{Shell \MakeLowercase{\textit{et al.}}: Bare Demo of IEEEtran.cls for IEEE Journals}
%



\maketitle

\begin{abstract}
	Low-rank tensor modeling has become an effective tool for hyperspectral anomaly detection. However, existing methods still suffer from high computational cost and limited flexibility in characterizing spatially structured anomalies. To address these issues, this paper proposes a hyperspectral anomaly detection method based on group sparse low-rank tensor factorization with automatic anomaly grouping (GSAA). Specifically, the low tubal rank background is characterized by imposing group sparsity on tensor factors, which provides an efficient alternative to direct tensor rank regularization. For anomaly modeling, a latent grouping map is introduced to build an automatic anomaly grouping penalty, allowing anomaly groups to be adaptively inferred from the data rather than predefined at the pixel level. To further exploit complementary spectral and spatial information, GSAA is applied in both domains, and the resulting detection maps are fused to form a spectral--spatial version of GSAA, termed GSAA-SS. An efficient linearized alternating direction method of multipliers algorithm with convergence guarantee is developed to solve the resulting model. Experimental results on five real hyperspectral datasets demonstrate that the proposed method achieves superior detection performance and competitive computational efficiency compared with several state-of-the-art methods.
\end{abstract}

\begin{IEEEkeywords}
	Hyperspectral anomaly detection, low-rank tensor factorization, group sparsity, automatic anomaly grouping, spectral--spatial fusion.
\end{IEEEkeywords}

%
\IEEEpeerreviewmaketitle

\section{Introduction}
\IEEEPARstart{H}{yperspectral} images (HSIs) provide rich spectral signatures by recording hundreds of narrow and contiguous bands for each spatial location \cite{LKC15}. Owing to this spectral resolution, HSIs have been widely used in denoising \cite{TLZ23,ZWZ26}, image fusion \cite{YBLC26,SYY26}, anomaly detection \cite{QWS25,YB24}, and classification \cite{HGY21}. Among these applications, hyperspectral anomaly detection (HAD) is of particular interest in public safety, surveillance, and defense, where anomalous targets are expected to be identified without prior knowledge of their spectral signatures or class labels \cite{SWZD22}. The main difficulty of HAD lies in suppressing complex and heterogeneous backgrounds while preserving weak and spatially structured anomalies.

Existing HAD methods can be roughly grouped into statistical methods, deep learning based methods, and low-rank modeling based methods. Statistical methods usually assume a prescribed distribution for background pixels and detect anomalies by measuring the deviation from the estimated background distribution. The Reed--Xiaoli (RX) detector \cite{RY90} is a representative method, and many variants have been developed to improve background estimation or detection robustness \cite{GZR14,MVDC13,HSG26}. However, simple parametric assumptions are often insufficient for real HSIs, whose backgrounds can be highly nonlinear and heterogeneous \cite{HZZJ19}. Deep learning based methods improve representation ability by exploiting neural networks to learn spectral--spatial features. Representative studies include autoencoder (AE) based models and their variants for unsupervised detection \cite{BCKA15,FMM22,TZL25}, convolutional neural network based methods for supervised detection \cite{LWD17}, and several self-supervised architectures \cite{THH25}, such as blind-spot learning, pixel-shuffle downsampling blind-spot reconstruction, and nonlocal--local feature-coupled schemes \cite{GWZ23,WZG23,WRS25}. Despite their promising performance, these methods usually require substantial training cost and careful tuning of network architectures and hyperparameters.

Low-rank modeling provides another important route for HAD. It is based on the observation that hyperspectral backgrounds usually exhibit strong spectral--spatial correlations and can be approximated by low-rank structures, whereas anomalies appear as sparse deviations from the background. Low-rank decomposition based methods model an HSI as the superposition of a low-rank background component and a sparse anomaly component, often under the framework of robust principal component analysis. Representative examples include matrix based formulations \cite{SLC26,SLL14,ZDZW16} and tensor based extensions \cite{LLQ22,WLG25}. Low-rank representation based methods further introduce a predefined or learned dictionary to characterize the background subspace, and have been developed in both matrix based \cite{CW20,XWL16} and tensor based \cite{WWH23,YDB26} forms. Since HSIs are naturally third order tensors, tensor based methods are generally more suitable than matrix based methods for preserving the intrinsic spectral--spatial structure of the data.

Although tensor based low-rank methods have achieved encouraging results, 
two key issues have yet to be adequately addressed.
 First, most existing methods impose low-rank constraints directly on the background tensor through tensor nuclear norm type surrogates or related nonconvex penalties \cite{YB24,WWH23,HWL23,QSZ23}. The resulting optimization usually involves repeated singular value decompositions (SVDs), which become computationally expensive for large-scale HSIs. Second, many methods characterize the anomaly component by the mixed $\ell_2$--$\ell_1$ norm or its variants, where each spatial pixel together with its spectral vector is treated as an independent group \cite{YDB26,FKZW26}. This strategy preserves spectral grouping, but it fixes the spatial partition at the pixel level and therefore cannot flexibly describe anomalous regions with unknown shapes and extents.

To address these limitations, this paper proposes an HAD method based on group sparse low-rank tensor factorization with automatic anomaly grouping, termed GSAA. Instead of directly regularizing the background tensor, GSAA imposes group sparsity on tensor factors to characterize the low tubal rank background with lower computational cost. Instead of predefining each pixel as an independent anomaly group, GSAA introduces a latent grouping map to infer the grouping structure of anomalies from the data. The proposed GSAA model is further applied in both spectral and spatial domains, and the resulting detection maps are fused to construct a spectral--spatial version of GSAA, termed GSAA-SS. An efficient linearized alternating direction method of multipliers (LADMM) algorithm is developed for solving the proposed model, and convergence analysis is provided. The main contributions of this paper are summarized as follows.
\begin{itemize}
	\item We propose a group sparse tensor factorization strategy for low-rank background modeling. By imposing group sparsity on tensor factors, the proposed formulation implicitly characterizes the low tubal rank structure, avoiding direct tensor rank regularization and repeated SVD computations. Its connection to tensor Schatten-$p$ regularization is further established to justify the low-rank modeling ability.
	\item We develop an automatic anomaly grouping penalty for adaptive anomaly modeling. By introducing a latent grouping map, the proposed penalty learns the grouping structure of anomalies from the data instead of relying on a fixed pixel-level partition, making it more suitable for spatially clustered anomalous targets.
	\item We construct a GSAA-SS detector by applying GSAA in the spectral and spatial domains and fusing the corresponding detection maps. The proposed model is optimized by an efficient LADMM algorithm with convergence guarantee, and experiments on real hyperspectral datasets validate its detection accuracy and computational efficiency.
\end{itemize}

The remainder of this paper is organized as follows. Section \ref{Sec:PK} introduces the preliminaries of tensor algebras and related notations.  We  present the  GSAA-SS model for HAD in Section \ref{Sec:GSAA}. 
An LADMM based optimization algorithm with convergence guarantee is developed to solve the resulting model in Section \ref{Sec:Alg}. Section \ref{Sec:Exp} reports the experimental results on several real hyperspectral datasets to show the effectiveness of GSAA-SS. Finally, the conclusions are drawn in Section \ref{Sec:con}. The proof of the main theorem is deferred to the supplementary material.

\section{Preliminaries}\label{Sec:PK}
This section summarizes the notation and tensor tools used in the proposed model. For a positive integer $n$, let $[n]:=\{1,2,\ldots,n\}$. Scalars, vectors, matrices, and tensors are denoted by lowercase letters ($a$), bold lowercase letters ($\bm{a}$), uppercase letters ($A$), and calligraphic letters ($\A$), respectively. The fields of real and complex numbers are denoted by $\R$ and $\mathbb{C}$. For a matrix $X\in\mathbb{R}^{n_1\times n_2}$, $\nabla_1 X$ and $\nabla_2 X$ denote the first-order forward finite-difference operators in the vertical and horizontal directions, respectively.

For a third order tensor $\X\in \R^{n_1\times n_2\times n_3}$, its $(i,j,k)$th entry is denoted by $\X_{ijk}$ or $\X(i,j,k)$, and its $k$th frontal slice is denoted by $X^{(k)}$. For two tensors $\X,\Y\in \R^{n_1\times n_2\times n_3}$, their inner product is
defined as
$\langle \X,\Y \rangle=\sum_{i=1}^{n_1}\sum_{j=1}^{n_2}\sum_{k=1}^{n_3}\X_{ijk}\Y_{ijk}$,
and the Frobenius norm is defined as $\|\X\|=\sqrt{\langle \X,\X\rangle}$. The $\ell_0$-norm of $\X$ counts the number of its nonzero entries. We use $\bar{\X}$ to denote the discrete Fourier transform (DFT) of $\X$ along the third mode, i.e., $\bar{\X}=\mathrm{fft}(\X,[\;],3)$, and $\X=\mathrm{ifft}(\bar{\X},[\;],3)$ denotes the inverse transform.

\begin{defi}\textbf{(f-diagonal tensor)} \cite{KM11}
	A tensor is called f-diagonal if each of its frontal slices is a diagonal matrix. 
\end{defi}
\begin{defi}\textbf{(conjugate transpose)} \cite{KM11}
	The conjugate transpose of a tensor $\mathcal{Z} \in \mathbb{R}^{n_{1} \times n_{2} \times n_{3}}$, denoted by $\mathcal{Z}^{\top}$, is defined as the tensor obtained by taking the conjugate transpose of each frontal slice and then reversing the order of the transposed frontal slices from the second to the last.
\end{defi}
\begin{defi}\textbf{(identity tensor)} \cite{KM11}
	The identity tensor $\mathcal{I} \in \mathbb{R}^{n \times n \times n_{3}}$ is defined as a tensor whose first frontal slice is the identity matrix, while all other frontal slices consist entirely of zeros.
\end{defi}
\begin{defi}\textbf{(orthogonal tensor)} \cite{KM11}
	A tensor $\mathcal{P} \in \mathbb{R}^{n \times n \times n_{3}}$ is said to be orthogonal if  $\mathcal{P}^{\top} * \mathcal{P} = \mathcal{P} * \mathcal{P}^{\top} = \mathcal{I}$, where $\mathcal{I}$ is the identity tensor. 
\end{defi}

\begin{defi}\label{def:T-pro}\textbf{(t-product \cite{KM11})} 
	For tensors $\X\in \mathbb{R}^{n_1\times r\times n_3}$ and $\Y\in \mathbb R^{r\times n_2\times n_3}$, the t-product is defined as
	$$\X\ast\Y:=\operatorname{Fold}\big(\operatorname{bcirc}(\X)\ \cdot \operatorname{Unfold}(\Y)\big) \in \mathbb{R}^{n_1\times n_2\times n_3}.$$
	Here,
	$$\operatorname{bcirc}({\X}) = \left[ {\begin{array}{cccc}
			{X^{(1)}}&{X^{(n_3)}}& \cdots &{X^{(2)}}\\
			{X^{(2)}}&{X^{(1)}}& \cdots &{X^{(3)}}\\
			\vdots & \vdots & \ddots & \vdots \\
			{X^{(n_3)}}&{X^{({n_3} - 1)}}& \cdots &{X^{(1)}}
	\end{array}} \right],$$
	$ \operatorname{Unfold}(\Y) = \big[Y^{(1)};Y^{(2)}; \ldots ;Y^{(n_3)}\big]\in\mathbb{R}^{n_3r\times n_2}$,
	and its inverse operator ``$\operatorname{Fold}$" is defined by $\operatorname{Fold}(\operatorname{Unfold}(\Y)) = \Y$.
\end{defi}

\begin{theorem}\textbf{(t-SVD \cite{KM11})} 
	Any tensor $\mathcal{Z} \in \mathbb{R}^{n_1 \times n_2 \times n_3}$ admits a tensor singular value decomposition (t-SVD) of the form
	$
	\mathcal{Z}=\mathcal{U}_{\Z} * \mathcal{S}_{\Z} * \mathcal{V}_{\Z}^{\top}
	$,
	where $\mathcal{U}_{\Z} \in \mathbb{R}^{n_{1} \times n_{1} \times n_{3}}$ and $\mathcal{V}_{\Z} \in \mathbb{R}^{n_{2} \times n_{2} \times n_{3}}$ are orthogonal tensors, and $\mathcal{S}_{\Z} \in \mathbb{R}^{n_{1} \times n_{2} \times n_{3}}$ is an $ f $-diagonal tensor.
\end{theorem}

\begin{defi}\textbf{(tubal rank \cite{KBHH13})}
	The tensor tubal rank is the number of nonzero singular tubes in $\mS_{\Z}$, i.e., $\operatorname{rank}_t(\mathcal{Z})=\#\{i: \mathcal{S}_{\Z}(i, i,:) \neq 0\}$,
	where $\mS_{\Z}$ is obtained from the t-SVD of $\mathcal{Z}=\mathcal{U}_{\Z} * \mathcal{S}_{\Z} * \mathcal{V}_{\Z}^{\top}$.
\end{defi}

\begin{defi}\textbf{(mode-$i$ unfolding and folding)}
	Let $\Z\in\R^{n_1\times n_2\times n_3}$ be a third order tensor. Its mode-$i$ unfolding is the matrix $Z_{(i)}=\operatorname{unfold}_{(i)}(\Z) \in \R^{n_i \times \prod_{s\neq i}n_s}$, whose columns are mode-$i$ fibers arranged lexicographically over the remaining modes. The inverse operation is denoted by $\operatorname{fold}_{(i)}(Z_{(i)})=\Z$.
\end{defi}

\begin{defi}\textbf{(tensor Schatten-$p$ norm \cite{KXL18})}
	For a tensor $\mathcal{Z} \in \mathbb{R}^{n_{1} \times n_{2} \times n_{3}}$ with t-SVD $\mathcal{Z}=\mathcal{U}_{\Z} * \mathcal{S}_{\Z} * \mathcal{V}_{\Z}^{\top}$, the tensor Schatten-$p$ norm is defined as $\|\mathcal{Z}\|_{S_p}^p=\frac{1}{n_3}\sum_{k=1}^{n_3}\sum_{i=1}^{\min\{n_1,n_2\}}\big(\bar{\mathcal{S}}_{\Z}(i, i, k)\big)^p$.
\end{defi}

\section{A GSAA-SS model for HAD}\label{Sec:GSAA}
This section presents the proposed GSAA-SS model for HAD. We first revisit the standard low-rank and sparse decomposition view of HAD, which separates an observed HSI into a structured background and a sparse anomaly component. We then develop two modeling components: a group sparse tensor factorization for efficient low-rank background modeling, and an automatic anomaly grouping (AAG) penalty for adaptive anomaly modeling. These two components are integrated into a unified GSAA model, which is further applied in both the spectral and spatial domains.
Then the resulting detection maps are fused to obtain the final GSAA-SS detector.

\subsection{General Tensor Based Detection Framework}
A broad class of tensor based HAD models can be written as
\begin{equation}\label{HAD}
	\mathop{\arg\min}\limits_{\Z,\E}~    \rank(\Z) +  \gamma\|\E\|_{\operatorname{sparse}}, \quad
	\mbox{\rm s.t.}\quad \mH = \Z + \E.
\end{equation}
Here, $\mH$ denotes the observed hyperspectral tensor, $\Z$ represents the background, and $\E$ denotes the anomaly component. Model \eqref{HAD} follows a common assumption in HAD: the background contains strong spectral--spatial correlation and can be approximated by a low-rank tensor, whereas anomalies appear as sparse deviations from this structured background.

The practical effectiveness of \eqref{HAD} depends critically on how the low-rank and sparse terms are regularized. For the background component, directly minimizing $\rank(\Z)$ is NP-hard. Existing tensor methods therefore replace it with tractable surrogates such as the weighted nuclear norm \cite{WWH23}, the $\varepsilon$-shrinkage tensor nuclear norm \cite{HWL23}, and unified nonconvex penalty functions \cite{YB24,QSZ23}. Although these surrogates are effective in modeling low-rank structure, they typically require repeated SVDs, which constitute a major computational bottleneck for large-scale HSIs.

For the anomaly component, many methods exploit the fact that each spatial location is associated with a full spectral vector and replace $\|\E\|_{\operatorname{sparse}}$ with the mixed $\ell_2$--$\ell_1$ norm $\|\E\|_{2,1}:=\sum_{ij}\|\E(i,j,:)\|$ or its nonconvex variants \cite{YDB26,WLZ26,LLKW26}. This choice preserves spectral grouping, but it fixes the spatial grouping a priori by treating each pixel as an independent group. Such a finest partition assumption is often too restrictive for HAD, where anomalous targets usually occupy spatially clustered regions with unknown shapes and extents.

The proposed GSAA framework addresses these two issues in a coordinated manner. We first derive a group sparse factorization for efficient low-rank background modeling, and then introduce an automatic anomaly grouping mechanism that learns the spatial grouping structure of anomalies directly from the data.

\subsection{Group Sparse Low-Rank Tensor Factorization}
To model the low-rank background efficiently, we factorize the background tensor and impose group sparsity on the tensor factors. The key idea is that the active lateral slices of the factors determine the effective tubal rank of the reconstructed background. Therefore, low-rank structure can be encouraged through factor sparsity instead of direct rank regularization on $\Z$.
\begin{theorem}\label{Thm:gs}
	For any tensor $\Z\in\R^{n_1\times n_2 \times n_3}$ with $r:=\operatorname{rank}_t\big(\Z\big) \leq d \leq \min \big\{n_1, n_2\big\}$, we have
	\begin{itemize}
		\item[(1)] $\operatorname{rank}_t(\Z)=\frac{1}{2}\min_{\Z = \X*\Y^{\top}} \big(\|\X\|_{F,0} + \|\Y\|_{F,0}\big)=\min_{\Z = \X*\Y^{\top}} \|\X\|_{F,0}=\min_{\Z = \X*\Y^{\top}} \|\Y\|_{F,0}$;
		\item[(2)] $\frac{1}{p}\|\Z\|_{S_p}^p\le\min_{\Z = \X*\Y^{\top}} \big(\frac{1}{p_1}\|\X\|_{F,p_1} + \frac{1}{p_2}\|\Y\|_{F,p_2}\big) \le \big(\frac{n_3^{1-p_1/2}}{p_1} + \frac{n_3^{1-p_2/2}}{p_2}\big) \|\Z\|_{S_p}^p$  for any $p_1, p_2 \in (0, 2]$ satisfying $1/p = 1/p_1+1/p_2$.
	\end{itemize}
	Here $ \|\Z\|_{F,p} = \sum_{j}\|\Z(:,j,:)\|^p $, $\X\in\R^{n_1\times d \times n_3}$ and $\Y\in\R^{n_2\times d \times n_3}$.
\end{theorem}

\begin{remark}
	Theorem \ref{Thm:gs} explains why group sparsity on tensor factors can serve as a low-rank prior. The first statement shows that the tubal rank of $\Z$ can be exactly characterized by the number of active groups in $\X$ and $\Y$. The second statement further connects group sparse factor penalties with tensor Schatten-$p$ regularization, which supports the low-rank modeling ability of the proposed factorized formulation.
\end{remark}
\begin{remark}
	The factorized formulation is also computationally favorable. Since low-rankness is imposed through factor regularization, the optimization avoids direct tensor rank regularization and the associated repeated SVD computations. Moreover, the auxiliary dimension $d$ is updated adaptively and, in practice, approaches the tubal rank $r$, which is usually much smaller than $\min\{n_1,n_2\}$. This leads to a more efficient background model for large HSIs.
\end{remark}

\subsection{Automatic Anomaly Grouping}\label{sec:aag}
Existing structured sparsity models for HAD usually assume that anomaly groups are known as a prior. 
A common choice is the mixed $\ell_2/\ell_1$ norm
\[
\|\E\|_{F,1}^{\B}=\sum_{l=1}^L\sqrt{|\B_l|}\|\E_{\B_l}\|,
\]
where $\{\B_l\}_{l=1}^L$ is a prescribed block partition. In hyperspectral applications, the most common partition treats the full spectral vector at each spatial location as a single group, namely,
\[
\B=\{\{(1,1,:)\},\{(1,2,:)\},\ldots,\{(n_1,n_2,:)\}\},
\]
which yields $\|\E\|_{2,1}$. This design preserves spectral grouping, but it fixes the spatial partition at the pixel level and cannot adapt to anomalous regions with unknown spatial extent.

To overcome this limitation, we introduce an automatic anomaly grouping (AAG) penalty. The spectral dimension is always kept within each group, while the spatial partition is learned from the data. Specifically, we define
\begin{equation}
	\Psi\big(\E\big)=\min_{L,\B_l}~\sum_{l=1}^L\sqrt{|\B_l|/n_3}\|\E_{\B_l}\|,
\end{equation}
where each anomaly block is defined by $\B_l=\Omega_l\times[n_3]$, and $\{\Omega_l\}_{l=1}^L$ forms a partition of the spatial index set $[n_1]\times[n_2]$. In contrast to conventional structured sparsity penalties, the spatial partition is optimized rather than prescribed in advance.

The direct optimization of \(\Psi(\E)\) is combinatorial because the spatial partition is unknown. To obtain a tractable formulation, we rewrite each block term using an auxiliary scalar shared by all pixels in the same spatial group. Following the variational representation in \cite[Lemma 1]{KK22}, define
$$\phi(e, \vartheta):= \begin{cases}\frac{|e|^2}{2 \vartheta}+\frac{\vartheta}{2}, & \text{if}~\vartheta>0; \\ 0, & \text{if}~e=0 ~\text{and}~ \vartheta=0; \\ \infty, & \text{otherwise}.\end{cases}$$
Then each block penalty admits the representation
\begin{equation}
	\sqrt{|\B_l|/n_3}\|\E_{\B_l}\| = \min_{\vartheta_l \in \mathbb{R}}~ \sum_{(i,j)\in \Omega_l} \phi\big(\|\E_{ij:}\|, \vartheta_l\big),
\end{equation}
where the minimum is attained at $ \vartheta_l = \sqrt{n_3/|\B_l|}\|\E_{\B_l}\| $. Therefore, the original combinatorial penalty can be equivalently rewritten as
\begin{equation*}
	\Psi\big(\E\big)=\min_{L,\B_l,\vartheta_l}~\sum_{l=1}^L\sum_{(i,j)\in \Omega_l} \phi\big(\|\E_{ij:}\|, \vartheta_l\big).
\end{equation*}

This representation leads to a latent grouping map. We introduce $\varTheta \in \R^{n_1\times n_2}$ by setting $\varTheta_{ij}=\vartheta_l$ for all $(i,j)\in\Omega_l$. Thus, pixels in the same spatial group share the same value in $\varTheta$, and connected constant regions of $\varTheta$ encode the learned spatial partition. Boundaries between neighboring groups correspond to nonzero entries in $\nabla_1\varTheta$ and $\nabla_2\varTheta$. Therefore, sparsity of these finite differences controls the complexity of the learned grouping structure.

Based on this observation, we relax the combinatorial partition optimization as
$$
\Psi_{\bm{\alpha}}\big(\E\big)=\min_{\varTheta \in \mathbb{S}_{\bm{\alpha}} } \sum_{i=1}^{n_1}\sum_{j=1}^{n_2} \phi\big(\big\|\E_{ij:}\big\|, \varTheta_{ij}\big),
$$
where
$
\mathbb{S}_{\bm{\alpha}}
:=\{\varTheta:\|\nabla_1\varTheta\|_0 \leq \alpha_1, \|\nabla_2\varTheta\|_0 \leq \alpha_2\}
$,
and $\bm{\alpha} = (\alpha_1, \alpha_2) \in \mathbb{Z}_+^2$. The parameters $\alpha_1$ and $\alpha_2$ control the number of vertical and horizontal changes in $\varTheta$, respectively, and thus determine the flexibility of the spatial grouping map.

The AAG penalty interpolates between two meaningful extremes, as formalized below.
\begin{theorem}\label{Thm:LOP}
	For a tensor $\E \in \R^{n_1\times n_2 \times n_3}$, we have
	\begin{itemize}
		\item[(1)] $\Psi_{(0,0)}(\E)=\sqrt{n_1n_2}\|\E\|$;
		\item[(2)] $\lim_{(\alpha_1,\alpha_2) \to (\infty,\infty)}\Psi_{(\alpha_1,\alpha_2)}(\E)=\|\E\|_{2,1}$.
	\end{itemize}
\end{theorem}

\begin{remark}
	Theorem \ref{Thm:LOP} shows that $\Psi_{\bm{\alpha}}$ bridges the coarsest and finest spatial partitions. Thus, it preserves spectral group sparsity while allowing the spatial grouping pattern to adapt to the data.
\end{remark}
\begin{remark}
	Compared with $\|\cdot\|_{2,1}$ based detectors, AAG estimates both the anomaly tensor $\E$ and the grouping map $\varTheta$. The piecewise constant structure of $\varTheta$ encodes learned spatial groups and provides a complementary spatially coherent anomaly response.
\end{remark}

\subsection{Proposed GSAA-SS Model}
By combining the group sparse low-rank background factorization with the AAG penalty, we obtain the unified GSAA model:
\begin{equation}\label{model:OBP_GS}
	\begin{array}{cl}
		\mathop{\arg\min}\limits_{\X,\Y,\E,\varTheta}&    \|\X\|_{F,p} + \|\Y\|_{F,p} + \gamma\varPhi(\E, \varTheta)      \\
		\mbox{\rm s.t.}& \mH = \X*\Y^{\top} + \E,~ \varTheta \in \mathbb{S}_{\bm{\alpha}}.
	\end{array}
\end{equation}
where $\varPhi(\E, \varTheta) =  \sum_{i=1}^{n_1}\sum_{j=1}^{n_2} \phi(\|\E_{ij:}\|, \varTheta_{ij})$.

In model \eqref{model:OBP_GS}, the factors $\X$ and $\Y$ model the low-rank background through group sparse regularization, whereas the pair $(\E,\varTheta)$ captures the anomaly component together with its adaptively learned spatial grouping structure.

To further exploit the complementary spectral and spatial information in HSIs, we extend the unified GSAA model to a spectral--spatial detection framework, termed GSAA-SS. Specifically, GSAA is separately applied in the spectral and spatial domains of the same HSI $\mH\in\mathbb{R}^{n_1\times n_2\times n_3}$, and the resulting detection maps are then fused. In the spectral branch, a matrix version of GSAA is applied to the mode-3 unfolding of $\mH$ to characterize the common background spectral subspace. Pixels that are poorly represented by this subspace yield large residuals and are therefore identified as spectral anomalies. In the spatial branch, principal component analysis (PCA) \cite{RS02} is first used to reduce the spectral dimension and obtain a compact tensor $\tilde{\mH}\in\mathbb{R}^{n_1\times n_2\times b}$ with $b\ll n_3$, while preserving the main spatial structure \cite{JC16}. GSAA is then applied to $\tilde{\mH}$ within the t-product framework to capture its low-rank background spatial structure. The overall workflow is illustrated in Fig.~\ref{fig:flow} and consists of the following three stages.

\par (1) \textbf{Spectral-domain detection:} 
The HSI is first unfolded along the spectral mode into the matrix $H_{(3)}$, and a matrix version of GSAA is then applied to the unfolded data:
\begin{equation*}
	\begin{array}{cl}
		\mathop{\arg\min}\limits_{X,Y,E,\varTheta}&    \|X\|_{F,p} + \|Y\|_{F,p} + \gamma\varPhi(\operatorname{fold}_{(3)}(E), \varTheta)      \\
		\mbox{\rm s.t.}& H_{(3)} = XY^{\top} + E,~ \varTheta \in \mathbb{S}_{\bm{\alpha}}.
	\end{array}
\end{equation*}
The optimized AAG map in this branch is denoted by $\varTheta_{spe}$.

\par (2) \textbf{Spatial domain detection:} For the tensor $\tilde{\mH}$ obtained by PCA, GSAA is applied within the t-product framework:
\begin{equation*}
	\begin{array}{cl}
		\mathop{\arg\min}\limits_{\X,\Y,\E,\varTheta}&    \|\X\|_{F,p} + \|\Y\|_{F,p} + \gamma\varPhi(\E, \varTheta)      \\
		\mbox{\rm s.t.}& \tilde{\mH} = \X*\Y^{\top} + \E,~ \varTheta \in \mathbb{S}_{\bm{\alpha}}.
	\end{array}
\end{equation*}
The optimized AAG map in this branch is denoted by $\varTheta_{spa}$.

\par (3) \textbf{Spectral--spatial fusion:} The AAG maps $\varTheta_{spe}$ and $\varTheta_{spa}$ obtained from the spectral and spatial branches, respectively, are fused via element-wise multiplication:
$$
\varTheta_{fus}=\varTheta_{spa}\odot \varTheta_{spe},
$$
where $\odot$ denotes the Hadamard product.

\begin{figure}[htbp]
	\centering
	\includegraphics[width=1\linewidth]{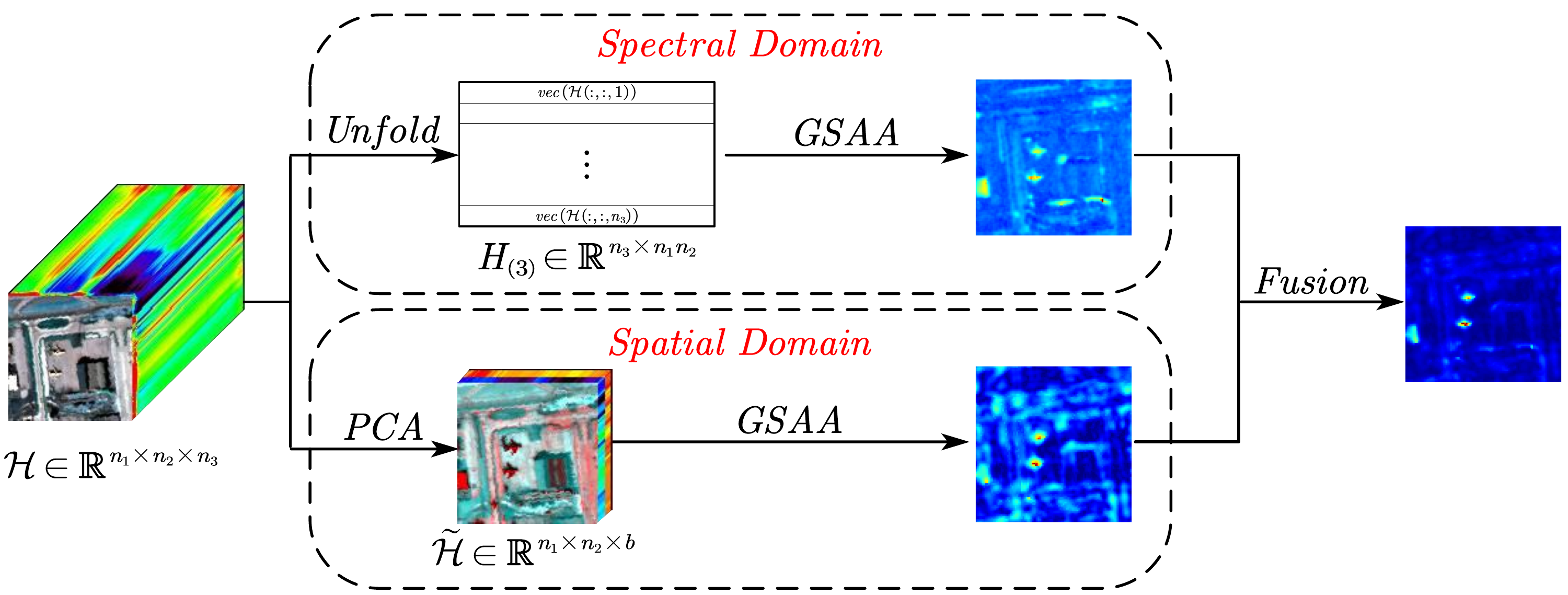}
	\caption{Flow chart of GSAA-SS.}
	\label{fig:flow}
\end{figure}

\section{Optimization by Linearized Alternating Direction Method of Multipliers}\label{Sec:Alg}
This section develops an efficient LADMM algorithm for solving the proposed GSAA model and provides convergence analysis.

\subsection{LADMM Algorithm}
To account for Gaussian noise in real HSI observations, the constrained GSAA model in \eqref{model:OBP_GS} is converted into the following penalized formulation:
\begin{multline}\label{model:OBP_GS_Noise}
	\mathop{\arg\min}\limits_{\X,\Y,\E,\varTheta\in \mathbb{S}_{\bm{\alpha}}}    \|\X\|_{F,p} + \|\Y\|_{F,p} + \gamma\varPhi(\E, \varTheta)\\+\frac{\lambda}{2}\big\|\X*\Y^{\top} + \E-\mH\big\|^2.     
\end{multline}

By introducing four auxiliary variables $\N$, $\Upsilon$, $U$ and $V$, problem \eqref{model:OBP_GS_Noise} can be rewritten as:
\begin{equation}\label{pro:av}
	\begin{array}{cl}
		\mathop{\arg\min}\limits_{\X,\Y,\E,\varTheta,\N,\Upsilon,U,V}&    \big\|\X\big\|_{F,p} + \big\|\Y\big\|_{F,p} + \gamma\varPhi(\E, \varTheta)+\frac{\lambda}{2}\big\|\X\\
		&*\Y^{\top} + \N-\mH\big\|^2  + \delta_{\mathbb{S}_1}\big(U\big) + \delta_{\mathbb{S}_2}\big(V\big)          \\
		\mbox{\rm s.t.}& \N = \E, \Upsilon = \varTheta, ~U=\nabla_1\Upsilon,~V=\nabla_2\Upsilon,
	\end{array}
\end{equation}
where $ \delta_{\mathbb{S}}(\cdot) $ is the indicator function, $ \mathbb{S}_1 = \big\lbrace U:\|U\|_0 \leq \alpha_1\rbrace $, and $ \mathbb{S}_2 = \lbrace V:\|V\|_0 \leq \alpha_2\big\rbrace  $. The augmented Lagrangian function of \eqref{pro:av} is given by
\begin{equation*}
	\begin{aligned}
		&\bL(\X,\Y,\E,\varTheta,\N,\Upsilon,U,V; \{\Lambda_u\}_{u=1}^4, \{\beta_u\}_{u=1}^2)\\
		=\ &\|\X\|_{F,p} + \|\Y\|_{F,p} + \gamma\varPhi(\E, \varTheta)+\frac{\lambda}{2}\big\|\X*\Y^{\top} + \N-\mH\big\|^2 \\&+ \delta_{\mathbb{S}_1}\big(U\big) + \delta_{\mathbb{S}_2}\big(V\big)+\big\langle \Lambda_1,  \E-\N\big\rangle + \frac{\beta_1}{2}\big\| \E-\N\big\|^2\\
		&+\big\langle \Lambda_2,  \varTheta-\Upsilon\big\rangle + \frac{\beta_1}{2}\big\|\varTheta-\Upsilon\big\|^2+\big\langle \Lambda_3,  \nabla_1\Upsilon-U\big\rangle \\
		&+ \frac{\beta_2}{2}\big\|\nabla_1\Upsilon-U\big\|^2+\big\langle \Lambda_4,  \nabla_2\Upsilon-V\big\rangle + \frac{\beta_2}{2}\big\|\nabla_2\Upsilon-V\big\|^2.
	\end{aligned}
\end{equation*}
where $\Lambda_u, u\in[4]$ are the Lagrange multipliers, and $\beta_u>0, u\in[2]$ are the penalty parameters. The variables are updated alternately, and the resulting subproblems admit closed-form or proximal solutions as described below.

\subsubsection{Update $\mathcal{X}^{t+1}$ and $\mathcal{Y}^{t+1}$}
The sub-problem to update $\mathcal{X}$ is
$$
\mathop{\arg\min}_{\X}~ \|\X\|_{F,p} + \lambda f(\X,\Y^t),
$$
where $f(\X,\Y) = \frac{1}{2}\|\X*\Y^{\top} + \N^t-\mH\|^2$.
To address the aforementioned problem, we linearize the term $f(\X,\Y^t)$ at the current iterate point $\X^t$. Consequently, the original problem can be reformulated in a relaxed manner as
\begin{multline*}
\mathop{\arg\min}_{\X}~\|\X\|_{F,p} + \lambda\big\langle\nabla_{\X}f(\X^t,\Y^t), \X-\X^t\big\rangle\\ + \frac{\lambda l_{\X}^t}{2}\|\X-\X^t\|^2.
\end{multline*}
Then the update of $\X$ is given by
\begin{equation}\label{opt:X}
	\X^{t+1} = \operatorname{prox}_{1/(\lambda l_{\X}^t)\|\cdot\|_{F,p}}\big(\X^t-\nabla_{\X}f(\X^t,\Y^t)/l_{\X}^t\big),
\end{equation}
where $\nabla_{\X}f(\X^t,\Y^t) = (\X^t*{\Y^t}^{\top} + \N^t-\mH) * \Y^t$ and $ l_{\X}^t=\|\Y^t\|_2^2+\varepsilon $ with $ \varepsilon>0 $.

The sub-problem to update $\mathcal{Y}$ is
$$
\mathop{\arg\min}_{\Y}~ \|\Y\|_{F,p} + \lambda f(\X^{t+1},\Y).
$$
Analogous to the update of $\mathcal{X}$, the update of $\mathcal{Y}$ is given by
\begin{equation}\label{opt:Y}
	\Y^{t+1} = \operatorname{prox}_{1/(\lambda l_{\Y}^t)\|\cdot\|_{F,p}}\big(\Y^t-\nabla_{\Y}f(\X^{t+1},\Y^t)/l_{\Y}^t\big),
\end{equation}
where $\nabla_{\Y}f(\X^{t+1},\Y^t) = (\X^{t+1}*{\Y^t}^{\top} + \N^t-\mH)^{\top} * \X^{t+1}$ and $ l_{\Y}^t=\|\X^{t+1}\|_2^2+\varepsilon $.

\subsubsection{Update $\E^{t+1}$ and $\varTheta^{t+1}$}
The sub-problem to update $(\E, \varTheta)$ is
\begin{equation}\label{opt:E-Theta}
	\begin{aligned}
		&\mathop{\arg\min}\limits_{\E, \varTheta}~  \gamma\varPhi(\E, \varTheta) +  \big\langle\Lambda_1^t,  \E-\N^t\big\rangle + \frac{\beta_1^t}{2}\|\E-\N^t\|^2\\
		&+\big\langle\Lambda_2^t,  \varTheta-\Upsilon^t\big\rangle + \frac{\beta_1^t}{2}\|\varTheta-\Upsilon^t\|^2\\
		=&\operatorname{prox}_{\gamma/\beta_1^t\varPhi}(\N^t - \Lambda_1^t/\beta_1^t, \Upsilon^t-\Lambda_2^t/\beta_1^t).
	\end{aligned}
\end{equation}

\subsubsection{Update $\N^{t+1}$}
The sub-problem to update $\N$ is
\begin{equation}\label{sub:N}
	\begin{aligned}
		&\mathop{\arg\min}\limits_{\N}~\frac{\lambda}{2}\|\X^{t+1}*{\Y^{t+1}}^{\top} + \N-\mH\|^2 \\&
		+\langle \Lambda_1^t,  \E^{t+1}-\N\rangle + \frac{\beta_1^t}{2}\| \E^{t+1}-\N\|^2\\
		=	& (\lambda(\mH - \X^{t+1}*{\Y^{t+1}}^{\top}) + \Lambda_1^t + \beta_1^t\E^{t+1})/(\lambda + \beta_1^t).
	\end{aligned}
\end{equation}

\subsubsection{Update $\Upsilon^{t+1}$}
The sub-problem to update $\Upsilon$ is
\begin{equation*}
	\begin{aligned}
		&\mathop{\arg\min}\limits_{\Upsilon}~\big\langle \Lambda_2^t,  \varTheta^{t+1}-\Upsilon\big\rangle + \frac{\beta_1^t}{2}\big\|\varTheta^{t+1}-\Upsilon\big\|^2
		\\
		&+\big\langle \Lambda_3^t,  \nabla_1\Upsilon-U^t\big\rangle + \frac{\beta_2^t}{2}\big\|\nabla_1\Upsilon-U^t\big\|^2\\&+\big\langle \Lambda_4^t,  \nabla_2\Upsilon-V^t\big\rangle + \frac{\beta_2^t}{2}\big\|\nabla_2\Upsilon-V^t\big\|^2.
	\end{aligned}
\end{equation*}
To optimize the problem, we reformulate it as the following linear system
\begin{equation}\label{ls}
	(\beta_1^tI+\beta_2^t\nabla_1^{\top}\nabla_1  +\beta_2^t\nabla_2^{\top}\nabla_2)\Upsilon=\chi_0 + \nabla_1^{\top}\chi_1 + \nabla_2^{\top}\chi_2,
\end{equation}
where $\chi_0 = \beta_1^t\varTheta^{t+1}+\Lambda_2^t$, $\chi_1 = \beta_2^tU^t - \Lambda_3^t$ and $\chi_2=\beta_2^tV^t - \Lambda_4^t$.
The operator $\nabla_u^{\top}\nabla_u$ corresponds to a block-circulant matrix, which can be diagonalized using a two-dimensional fast Fourier transform matrix. Applying the Fourier transform to both sides of equation \eqref{ls} and employing the convolution theorem, as shown in \cite{YB24,KF09}, we can readily derive the closed-form solution for $\Upsilon^{t+1}$ as follows
\begin{equation}\label{opt:Upsilon}
	\Upsilon^{t+1} = \F^{-1}\Big(\frac{\F(\chi_0)+\sum_{u=1}^2\F(\nabla_u)^{\top}\odot\F(\chi_u)}{\beta_1^t\bm{1}+\sum_{u=1}^2\beta_2^t|\F(\nabla_u)|^2}\Big),
\end{equation}
where $ \bm{1} $ represents the matrix with all elements as 1, $ \odot $ is the element-wise multiplication, $\F(\cdot)$ is the Fourier transform, and $|\cdot|^2$ is the element-wise square operation.

\subsubsection{Update $U^{t+1}$}
The sub-problem to update $U$ is
\begin{equation}\label{opt:U}
	\begin{aligned}
		&\mathop{\arg\min}\limits_{U \in \mathbb{S}_1}~ \langle \Lambda_3^t,  \nabla_1\Upsilon^{t+1}-U\rangle + \frac{\beta_2^t}{2}\|\nabla_1\Upsilon^{t+1}-U\|^2\\
		=&P_{\mathbb{S}_1}(\nabla_1\Upsilon^{t+1}+\Lambda_3^t/\beta_2^t).
	\end{aligned}
\end{equation}

\subsubsection{Update $V^{t+1}$}
The sub-problem to update $V$ is
\begin{equation}\label{opt:V}
	\begin{aligned}
		&\mathop{\arg\min}\limits_{V \in \mathbb{S}_2}~ \langle \Lambda_4^t,  \nabla_2\Upsilon^{t+1}-V\rangle + \frac{\beta_2^t}{2}\|\nabla_2\Upsilon^{t+1}-V\|^2\\
		=&P_{\mathbb{S}_2}(\nabla_2\Upsilon^{t+1}+\Lambda_4^t/\beta_2^t).
	\end{aligned}
\end{equation}

The complete LADMM procedure is summarized in Algorithm \ref{Alg:LADMM}.
\begin{algorithm}[htbp]
	\caption{The LADMM algorithm for problem \eqref{pro:av}}\label{Alg:LADMM}
	\KwIn{$\gamma>0$, $\lambda>0$, $p$, $\{\alpha_u\}_{u=1}^2$, $\{\beta_u^0, \rho_u\}_{u=1}^2$, and $\X^0$, $\Y^0$, $\E^0$, $\varTheta^0$,  $\N^0$, $\Upsilon^0$, $U^0$, $V^0$, $\{\Lambda_u^0\}_{u=1}^4$.}
	\KwOut{$\varTheta^{t+1}$.}
	$t\leftarrow 0$;
	
	\While{not converged}
	{
		Update $\X^{t+1}$ according to \eqref{opt:X}.
		
		Update $\Y^{t+1}$ according to \eqref{opt:Y}.
		
		Update $(\E^{t+1}, \varTheta^{t+1})$ according to \eqref{opt:E-Theta}.
		
		Update $\N^{t+1}$ according to \eqref{sub:N}.
		
		Update $\Upsilon^{t+1}$ according to \eqref{opt:Upsilon}.
		
		Update $U^{t+1}$ according to \eqref{opt:U}.
		
		Update $V^{t+1}$ according to \eqref{opt:V}.
		
		Update multipliers $\Lambda_u^{t+1}$ and penalty parameters $\beta_u^{t+1}$  according to 
		\begin{equation}\label{ADMM-2}
			\left\{\begin{array}{l}
				\Lambda_1^{t+1} = \Lambda_1^t + \beta_1^t\big(\E^{t+1}-\N^{t+1}\big);\\
				\Lambda_2^{t+1} = \Lambda_2^t + \beta_1^t\big(\varTheta^{t+1}-\Upsilon^{t+1}\big);\\
				\Lambda_3^{t+1} = \Lambda_3^t + \beta_2^t\big(\nabla_1\Upsilon^{t+1}-U^{t+1}\big);\\
				\Lambda_4^{t+1} = \Lambda_4^t + \beta_2^t\big(\nabla_2\Upsilon^{t+1}-V^{t+1}\big);\\
				\beta_1^{t+1} = \rho_1 \beta_1^t,~ \beta_2^{t+1} = \rho_2 \beta_2^t.
			\end{array}\right.
		\end{equation}
		
		Remove the zero lateral slices of $\X^{t+1}$ and $\Y^{t+1}$.
		
		$t\leftarrow t+1$\;
	}
\end{algorithm}

The remaining implementation details are the proximal mappings used in Algorithm \ref{Alg:LADMM}.
\par (1) $\operatorname{prox}_{\eta\|\cdot\|_{F,p}}(\cdot)$: Following \cite[Theorem 4.1]{YB24}, the proximal mapping is defined as
$$
[\operatorname{prox}_{\eta\|\cdot\|_{F,p}}(\Z)](:,j,:)=\begin{cases}\operatorname{prox}_{\eta |\cdot|^p}\big(z_j\big) \frac{\Z(:,j,:)}{z_j}, & \text{if}~z_j \neq 0; \\ 0, & \text{if}~z_j=0,\end{cases}
$$
where $z_j = \|\Z(:,j,:)\|$.
The proximal mapping of $|\cdot|^p$ considered here is classical, as given in \cite{MS12}.

\par (2) $\operatorname{prox}_{\eta\varPhi}(\cdot, \cdot)$: We first note that $\operatorname{prox}_{\eta\varPhi}(\tilde{\E}, \tilde{\varTheta})$ is separable, i.e.,
$$
[\operatorname{prox}_{\eta\varPhi}(\tilde{\E}, \tilde{\varTheta})]_{ij} = \operatorname{prox}_{\eta\tilde{\phi}}(\tilde{\E}_{ij:}, \tilde{\varTheta}_{ij}),
$$
where $\tilde{\phi}(\E_{ij:}, \varTheta_{ij}) = \phi(\|\E_{ij:}\|, \varTheta_{ij})$.
To further simplify the computation, we introduce the following lemma.
\begin{lem}
	For a given vector $\tilde{\bm{x}}$ and scalar $\tilde{y}$, we have
	$$
	\operatorname{prox}_{\eta\tilde{\phi}}(\tilde{\bm{x}}, \tilde{y}) = \begin{cases}(\frac{z_1}{\|\tilde{\bm{x}}\|}\tilde{\bm{x}},z_2), & \text{if}~\tilde{\bm{x}} \neq \bm{0}; \\ (\bm{0},\max\{\tilde{y}-\eta/2, 0\}), & \text{if}~\tilde{\bm{x}}=\bm{0},\end{cases}
	$$
	where $(z_1, z_2) \in \operatorname{prox}_{\eta\phi}(\|\tilde{\bm{x}}\|, \tilde{y})$.
\end{lem}
\begin{proof}
	It is clear that $\operatorname{prox}_{\eta\tilde{\phi}}(\tilde{\bm{x}}, \tilde{y}) = (\bm{0},\max\{\tilde{y}-\eta/2, 0\})$ when $\tilde{\bm{x}}=\bm{0}$. Therefore, we only need to consider $\tilde{\bm{x}} \neq \bm{0}$ in the following. From $(z_1, z_2) \in \operatorname{prox}_{\eta\phi}(\|\tilde{\bm{x}}\|, \tilde{y})$, one has
	\begin{equation*}
		\begin{aligned}
			&\eta\tilde{\phi}(\frac{z_1}{\|\tilde{\bm{x}}\|}\tilde{\bm{x}},z_2) + \frac{1}{2}\|\frac{z_1}{\|\tilde{\bm{x}}\|}\tilde{\bm{x}} - \tilde{\bm{x}}\|^2 + \frac{1}{2}(z_2-\tilde{y})^2\\
			=&\eta\phi(z_1,z_2) + \frac{1}{2}(z_1-\|\tilde{\bm{x}}\|)^2+ \frac{1}{2}(z_2-\tilde{y})^2\\
			\le & \eta \phi(\|\bm{x}\|,y)+ \frac{1}{2}(\|\bm{x}\|-\|\tilde{\bm{x}}\|)^2+ \frac{1}{2}(y-\tilde{y})^2\\
			\le & \eta\tilde{\phi}(\bm{x},y) + \frac{1}{2}\|\bm{x}-\tilde{\bm{x}}\|^2+ \frac{1}{2}\|y-\tilde{y}\|^2,
		\end{aligned}
	\end{equation*}
	which completes the proof.
\end{proof}

By this lemma, computing $\operatorname{prox}_{\eta\varPhi}(\tilde{\E}, \tilde{\varTheta})$ reduces to evaluating $\operatorname{prox}_{\eta\phi}(\|\tilde{\bm{x}}\|, \tilde{y})$, whose closed form is given in \cite[Example 2.4]{CM20}:
$$
\begin{aligned}
	&\operatorname{prox}_{\eta \phi}(\|\tilde{\bm{x}}\|, \vartheta) = \\&\begin{cases}(0,0), & \text{if}~ 2 \eta \vartheta+\|\tilde{\bm{x}}\|^2 \leq \eta^2; \\
		\big(0, \tilde{y}-\frac{\eta}{2}\big), & \text{if}~\|\tilde{\bm{x}}\|=0~\text{and}~2 \tilde{y}>\eta ; \\
		\big(\|\tilde{\bm{x}}\|-\eta s, \tilde{y}+\eta \frac{s^2-1}{2}\big), & \text{otherwise}.\end{cases}
\end{aligned}
$$
Let $a = \frac{2}{\eta}\tilde{y}+1$, $b = -\frac{2}{\eta}\|\tilde{\bm{x}}\|$, and $c = -\frac{b^2}{4}-\frac{a^3}{27}$. The variable $s$ is computed by
$$
s= \begin{cases}\sqrt[3]{-\frac{b}{2}+\sqrt{-c}}+\sqrt[3]{-\frac{b}{2}-\sqrt{-c}}, & \text{if}~ c<0; \\ 2 \sqrt[3]{-\frac{b}{2}}, & \text{if}~c=0; \\ 2 \sqrt[3]{\sqrt{\frac{b^2}{4}+c} \cos \big(\frac{\arctan (-2 \sqrt{c} / b)}{3}\big),} & \text{if}~ c>0.\end{cases}
$$

\par (3) $P_{\mathbb{S}}(\cdot)$: Let $\bm{x} \in \mathbb{R}^n$ be an arbitrary vector and $ \alpha > 0  $ a given radius. The projection of $ \bm{x} $ onto the $\ell_0$-ball,  defined as $ \mathbb{S} = \{\bm{z}\in \mathbb{R}^n|\|\bm{z}\|_0 \leq \alpha\} $, is denoted $P_{\mathbb{S}}(\bm{x}) $. This projection can be expressed concisely and computed efficiently as follows
\[
P_{\mathbb{S}}(\bm{x}) = 
\begin{cases} 
	\bm{x}, & \text{if}~ \|\bm{x}\|_0 \leq \alpha; \\
	\text{sign}(\bm{x}) \odot \max\big(|\bm{x}| - \varpi, 0\big), & \text{otherwise},
\end{cases}
\]
where $\varpi$ is defined as the ($\alpha$+1)-th largest absolute value in the descendingly sorted vector $|\bm{x}|$ components.

\subsection{Computation Complexity}
In each iteration of Algorithm \ref{Alg:LADMM}, the dominant cost comes from t-products and FFT operations. Updating $\X$ and $\Y$ requires two t-products with complexity $\mathcal{O}((d n_{1}+d n_{2}+n_{1}n_{2})n_{3}\log n_{3}+dn_{1}n_{2}n_{3})$, while updating $\N$ requires one t-product with complexity $\mathcal{O}(d(n_{1}+n_{2})n_{3}\log n_{3}+dn_{1}n_{2}n_{3})$. The update of $\Upsilon$ is solved by FFT and inverse FFT with complexity $\mathcal{O}(n_{1}n_{2}\log(n_{1}n_{2}))$. The updates of $\E$ and $\varTheta$ cost $\mathcal{O}(n_{1}n_{2}n_{3})$, and the projections for $U$ and $V$ cost $\mathcal{O}(n_{1}n_{2})$. Therefore, the overall per-iteration complexity is $\mathcal{O}((d(n_{1}+n_{2})+n_{1}n_{2})n_{3}\log n_{3} + dn_{1}n_{2}n_{3} + n_{1}n_{2}\log(n_{1}n_{2}))$.

\subsection{Convergence Analysis}
For convergence analysis, denote $\W = (\W_1, \W_2)$, where $\W_1 = (\X,\Y,\E,\varTheta,\N,\Upsilon,U,V)$ collects the primal variables and $\W_2 = (\Lambda_1,\dots, \Lambda_4)$ collects the dual variables. The following result establishes subsequential convergence of Algorithm \ref{Alg:LADMM} for problem \eqref{pro:av}.
\begin{theorem}\label{Thm:CA}
	Let $\{\W^t\}_{t=1}^{\infty}$ be the sequence produced by Algorithm \ref{Alg:LADMM}. Assume that the sequence  $\{\W_2^t\}_{t=1}^{\infty}$ is bounded. Then every accumulation point of $\{\W^t\}_{t=1}^{\infty}$ is a Karush--Kuhn--Tucker (KKT) point of the optimization problem \eqref{pro:av}.
\end{theorem}


\section{Numerical Experiments}\label{Sec:Exp}
This section evaluates the proposed GSAA-SS detector on five real hyperspectral scenes.  All experiments were conducted on a computer with an Intel Core i5-12500H CPU (2.50 GHz) and 16 GB RAM. All methods were implemented in MATLAB R2022a and applied to the raw hyperspectral data without additional preprocessing.

The proposed method is compared with ten representative hyperspectral anomaly detectors, including RX \cite{RY90}, RPCA \cite{SLL14}, LRASR \cite{XWL16}, Turbo-GoDec \cite{SLC26}, PTA \cite{LLQ22}, TPCA \cite{CYW18}, PCA-TLRSR \cite{WWH23}, RGAE \cite{FMM22}, GAED \cite{XAJZ22}, and LCRS \cite{HYZY26}. RX is a classical statistical detector. RPCA is a matrix low-rank decomposition method that separates sparse anomalies from a low-rank background and then applies RX to the anomaly component. LRASR is a matrix representation based detector using a dictionary driven background representation. Turbo-GoDec extends the classical GoDec framework by incorporating a spatial cluster sparsity prior for anomalies. PTA and TPCA are tensor based low-rank detectors, where PTA additionally introduces total variation regularization to exploit spatial smoothness. PCA-TLRSR is a tensor representation based method that employs PCA for spectral redundancy reduction and subsequent tensor representation modeling. RGAE and GAED are deep learning based detectors. LCRS is an efficient saliency based detector that combines multi-scale spectral local contrast with three-dimensional spectral residual saliency.

For quantitative evaluation, let $P_D$ denote the probability of detection and $P_F$ denote the false-alarm rate. We use the area under the curve (AUC) of the receiver operating characteristic (ROC) curve as the primary metric, where the ROC curve plots $P_D$ against $P_F$. A larger AUC indicates better detection accuracy.

\subsection{Dataset Description}
The experiments use five real hyperspectral scenes selected from two public data sources, which are described in detail below.
\par (1) Airport-Beach-Urban Dataset\footnote{\url{http://xudongkang.weebly.com/data-sets.html}}: 
The Airport-Beach-Urban (ABU) dataset contains hyperspectral images from airport, beach, and urban scenes. The spatial size of each image is either \(100 \times 100\) or \(150 \times 150\), and the number of spectral bands is approximately 100 or 200 depending on the sensor. Four scenes are selected in this paper, including two airport scenes, one beach scene, and one urban scene. Their pseudo-color images and ground-truth maps are shown in Fig. \ref{fig:HSIs}(a)--(d).

\par (2) San Diego Dataset\footnote{\url{https://aviris.jpl.nasa.gov/data/index.html}}:
The San Diego dataset contains an airport scene captured by the Airborne Visible/Infrared Imaging Spectrometer (AVIRIS). It was acquired over San Diego with a spatial resolution of 3.5 m/pixel and contains 189 spectral bands covering wavelengths from 370 to 2510 nm. Following common experimental settings, a \(100 \times 100\) subimage is used. Its pseudo-color image and ground-truth map are shown in Fig. \ref{fig:HSIs}(e).

\begin{figure}[htbp]
	\centering
	\begin{subfigure}[b]{1\linewidth}
		\begin{minipage}{0.2\textwidth}
			\centering
			\includegraphics[width=1\textwidth]{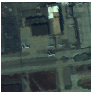}\vspace{0pt}
			\includegraphics[width=1\textwidth]{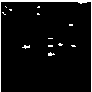}
			\caption{}
		\end{minipage}\hfill
		\begin{minipage}{0.2\textwidth}
			\centering
			\includegraphics[width=1\textwidth]{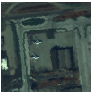}\vspace{0pt}
			\includegraphics[width=1\textwidth]{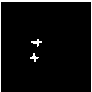}
			\caption{}
		\end{minipage}\hfill
		\begin{minipage}{0.199\textwidth}
			\centering
			\includegraphics[width=1\textwidth]{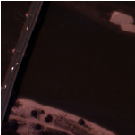}\vspace{0pt}
			\includegraphics[width=1\textwidth]{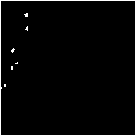}
			\caption{}
		\end{minipage}\hfill
		\begin{minipage}{0.199\textwidth}
			\centering
			\includegraphics[width=1\textwidth]{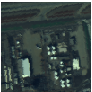}\vspace{0pt}
			\includegraphics[width=1\textwidth]{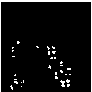}
			\caption{}
		\end{minipage}\hfill
		\begin{minipage}{0.2\textwidth}
			\centering
			\includegraphics[width=1\textwidth]{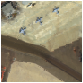}\vspace{0pt}
			\includegraphics[width=1\textwidth]{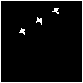}
			\caption{}
		\end{minipage}
	\end{subfigure}
	\vfill
	\caption{Pseudo-color images and ground-truth maps: (a) Airport1; (b) Airport2; (c) Beach; (d) Urban; (e) San Diego.}
	\label{fig:HSIs}
\end{figure}

\subsection{Detection Performance}
Fig. \ref{fig:2D} shows the detection maps produced by all compared methods on the five test scenes. The proposed GSAA-SS detector provides clearer anomaly responses while maintaining stronger background suppression in most cases. On Airport2, RX, RPCA, LRASR, Turbo-GoDec, TPCA, RGAE, and LCRS miss part of the anomalous targets, whereas PTA, PCA-TLRSR, and GAED recover more targets but also retain visible background interference, especially in the lower-left region. GSAA-SS better separates the target region from the background. On Urban, RX, RPCA, LRASR, Turbo-GoDec, and LCRS suppress the background effectively but also weaken several anomaly pixels. PTA, TPCA, PCA-TLRSR, RGAE, and GAED preserve more anomaly responses but introduce stronger false responses in the upper-right and lower-left areas. In contrast, GSAA-SS produces compact and prominent anomaly responses with fewer background artifacts.

\begin{figure*}[htbp]
	\centering
	\begin{subfigure}[b]{\linewidth}
		\centering
		\includegraphics[width=1\textwidth]{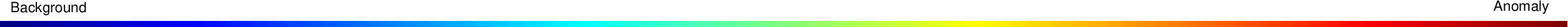}  
	\end{subfigure}
	\begin{subfigure}[b]{\linewidth}
		\centering
		\begin{subfigure}[b]{0.09\linewidth}
			\includegraphics[width=\linewidth]{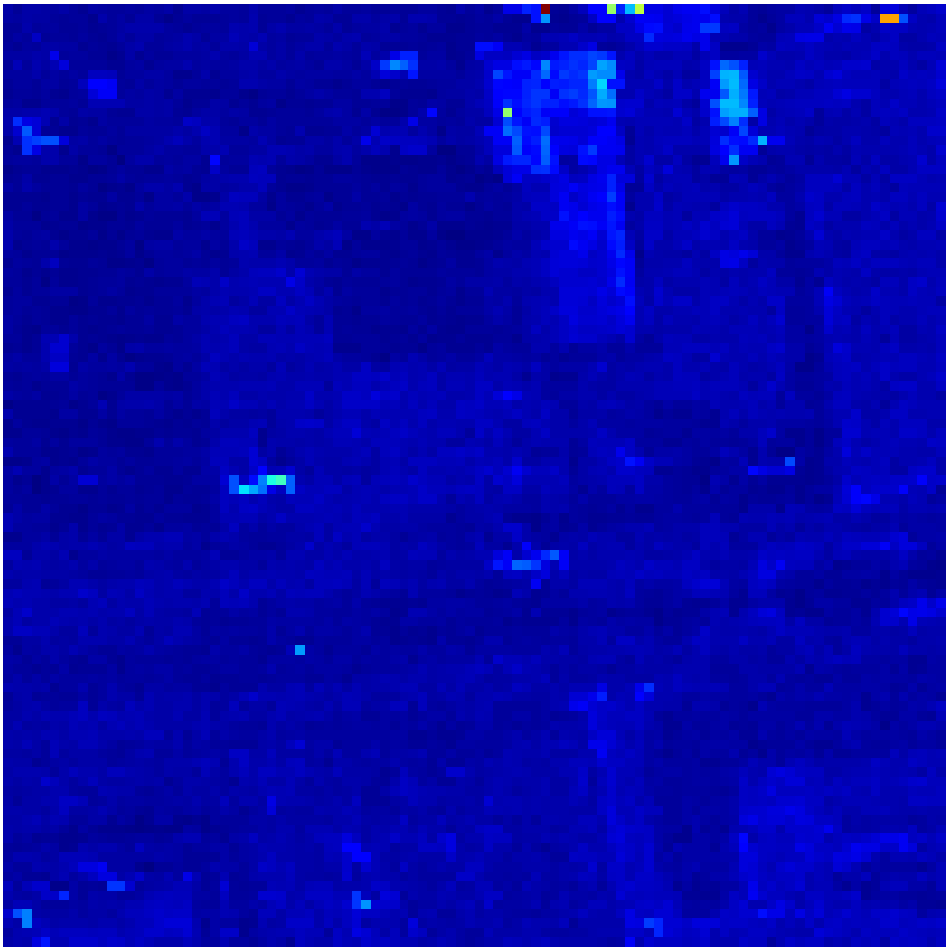}
			\includegraphics[width=\linewidth]{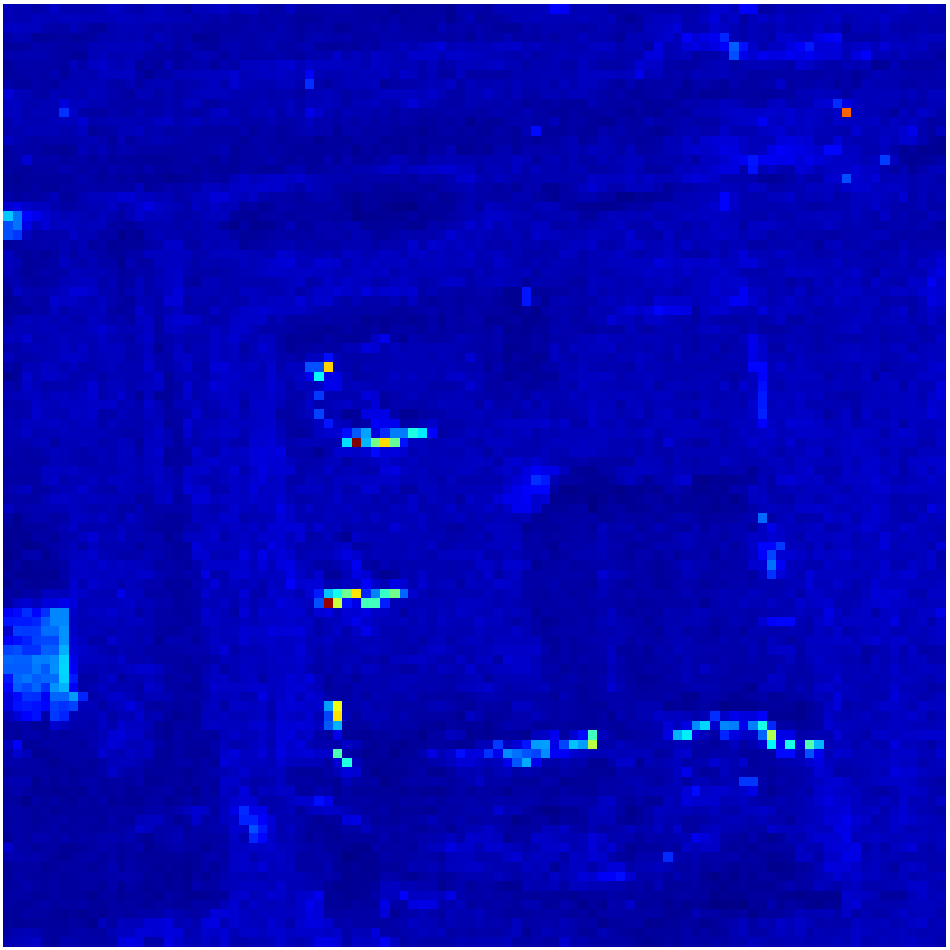}
			\includegraphics[width=\linewidth]{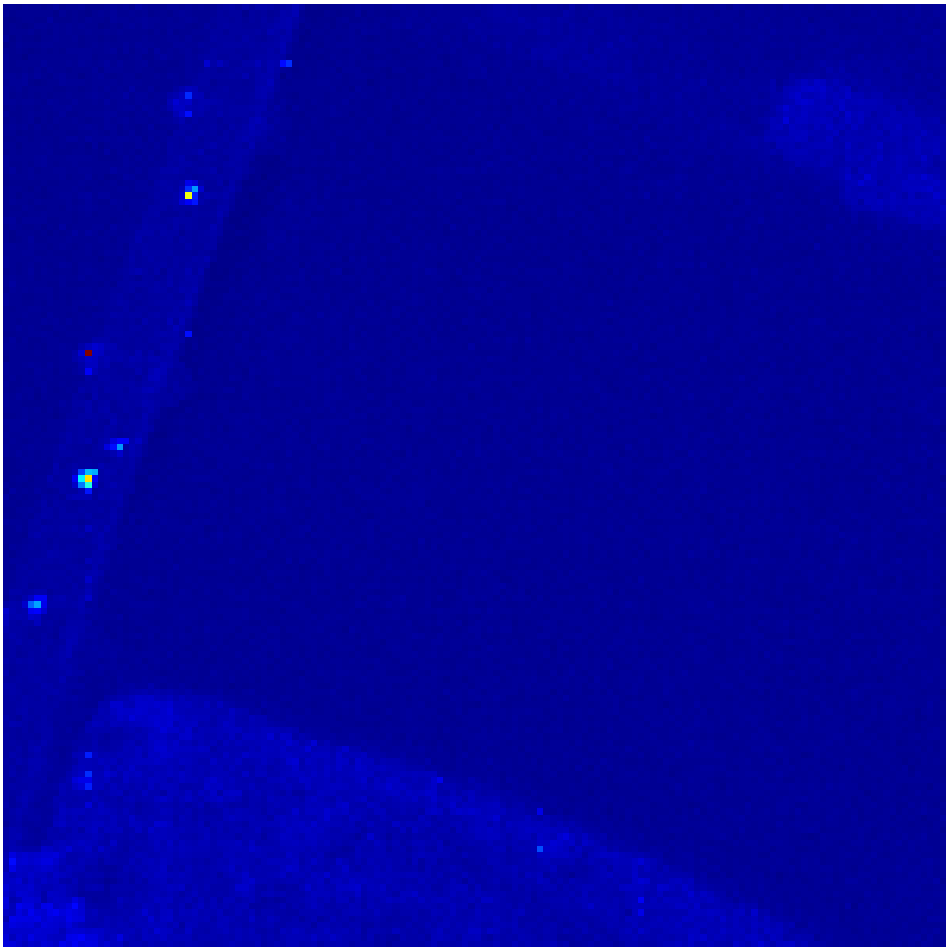}
			\includegraphics[width=\linewidth]{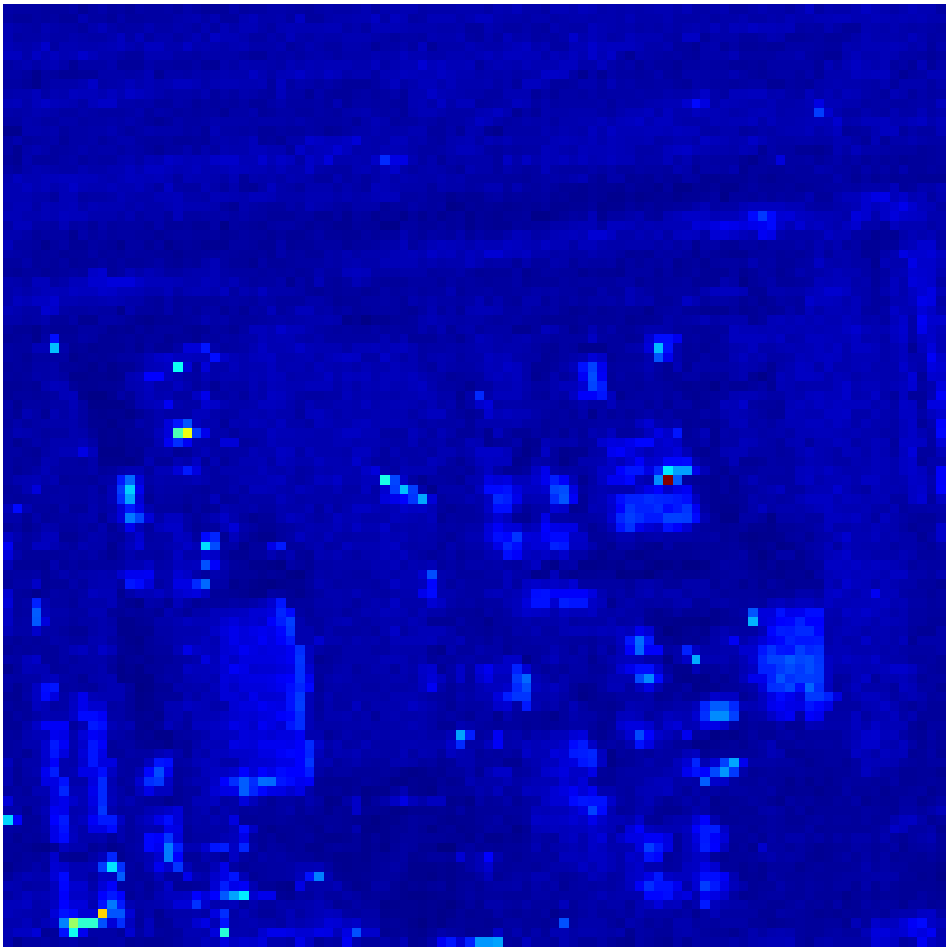}
			\includegraphics[width=\linewidth]{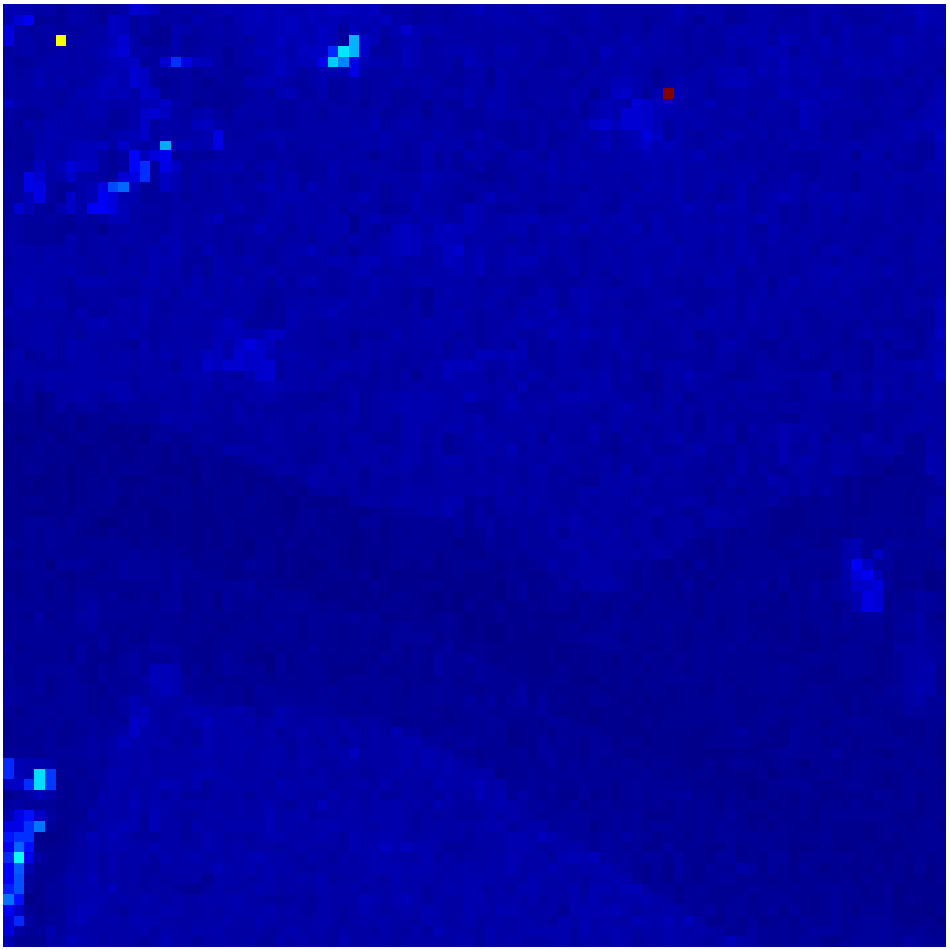}
			\caption{}
		\end{subfigure}\hfill
		\begin{subfigure}[b]{0.09\linewidth}
			\includegraphics[width=\linewidth]{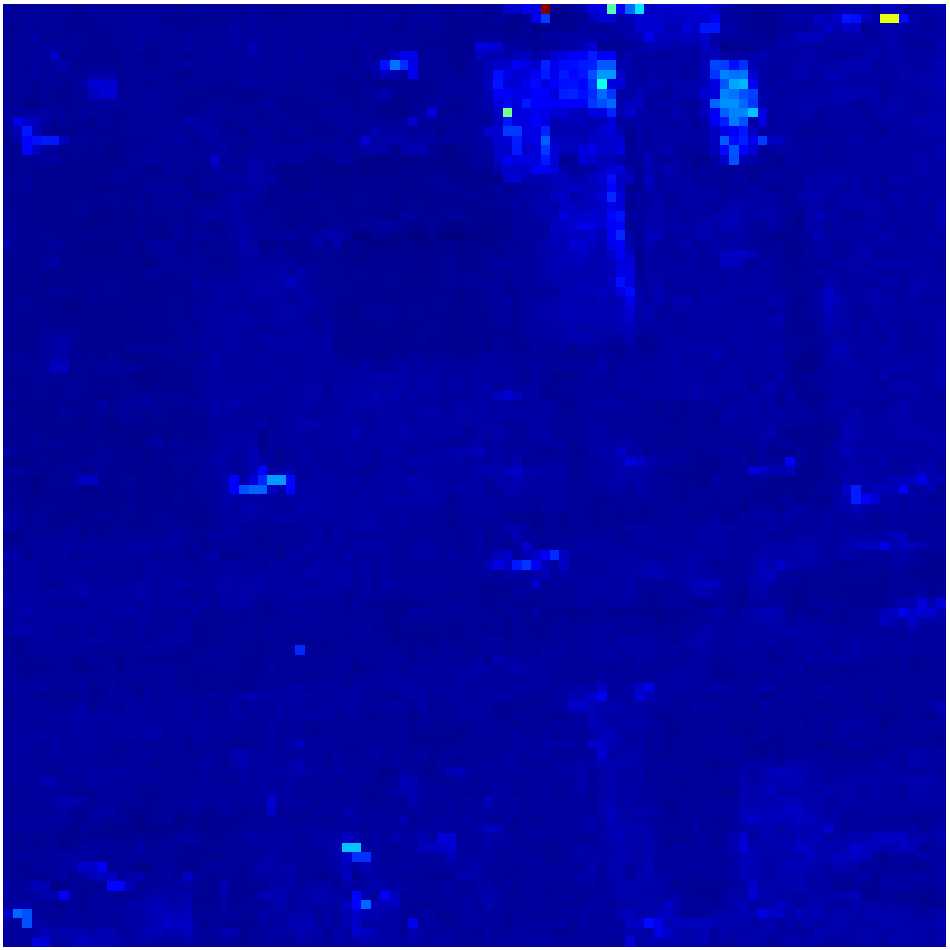}
			\includegraphics[width=\linewidth]{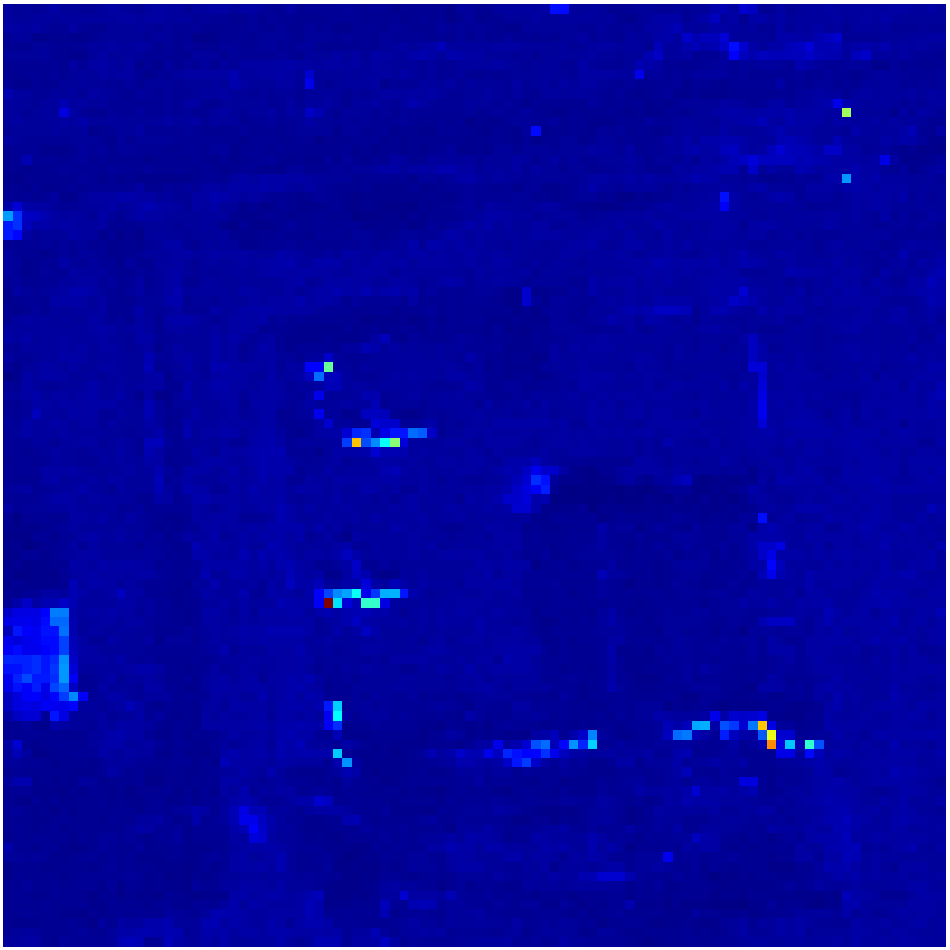}
			\includegraphics[width=\linewidth]{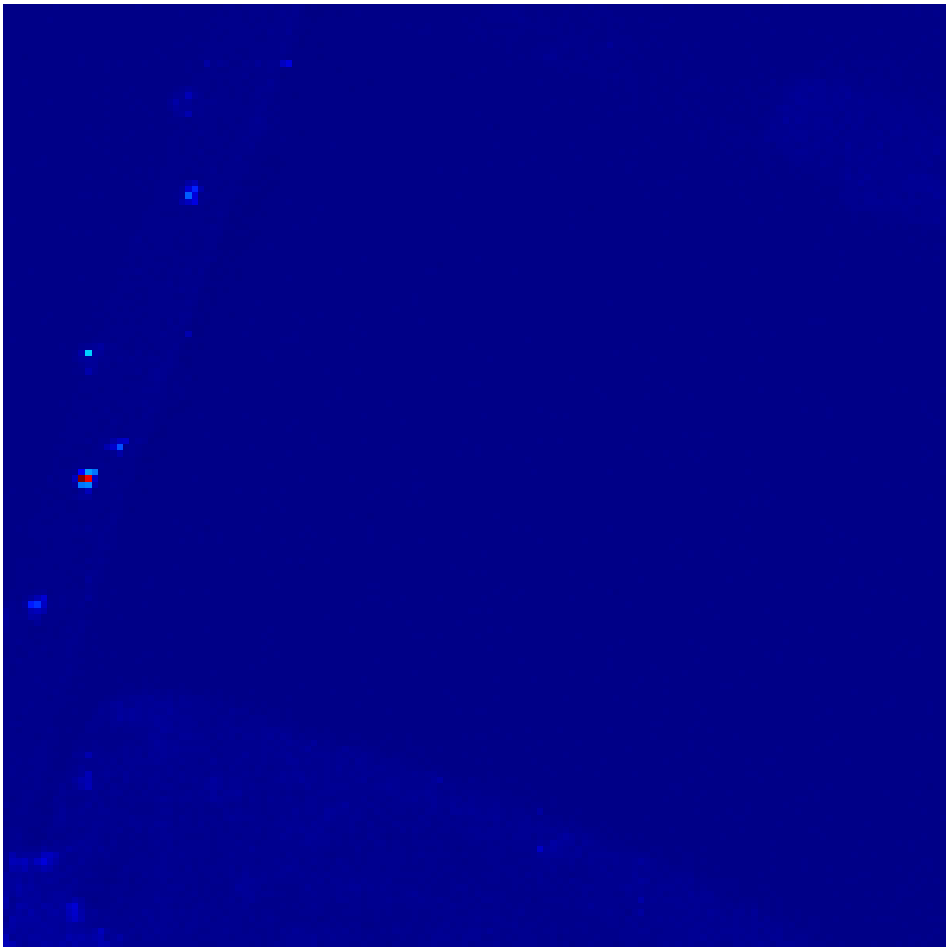}
			\includegraphics[width=\linewidth]{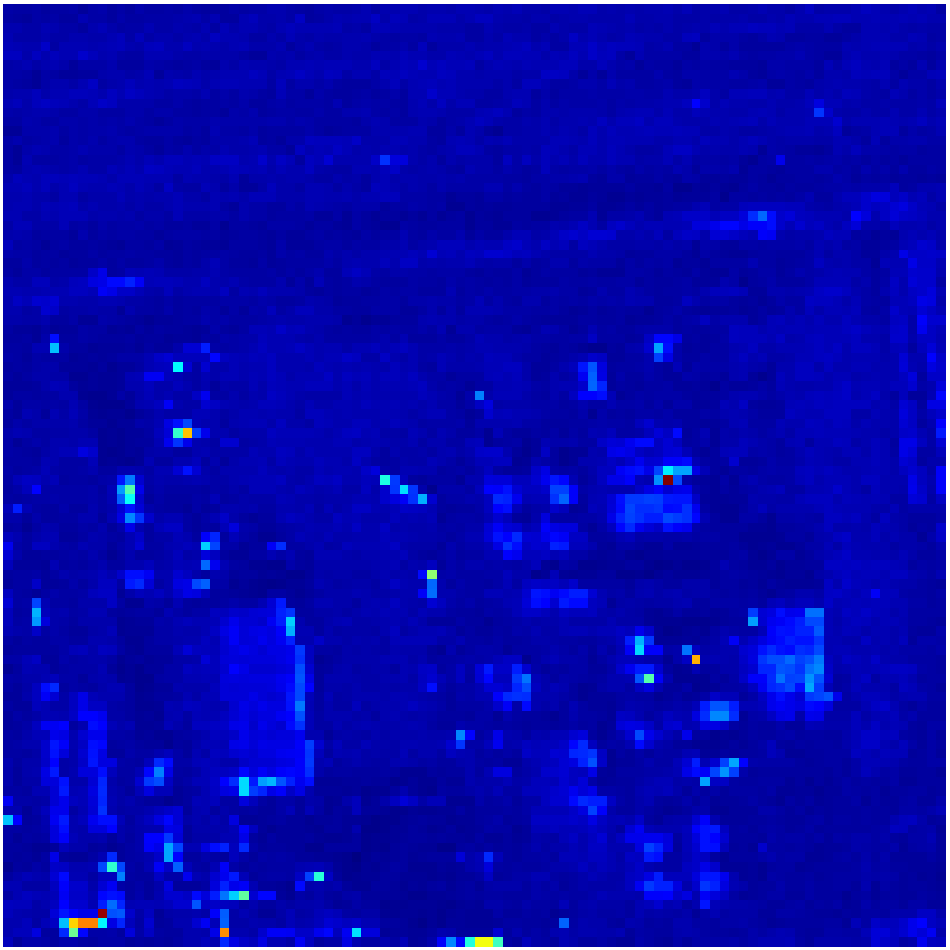}
			\includegraphics[width=\linewidth]{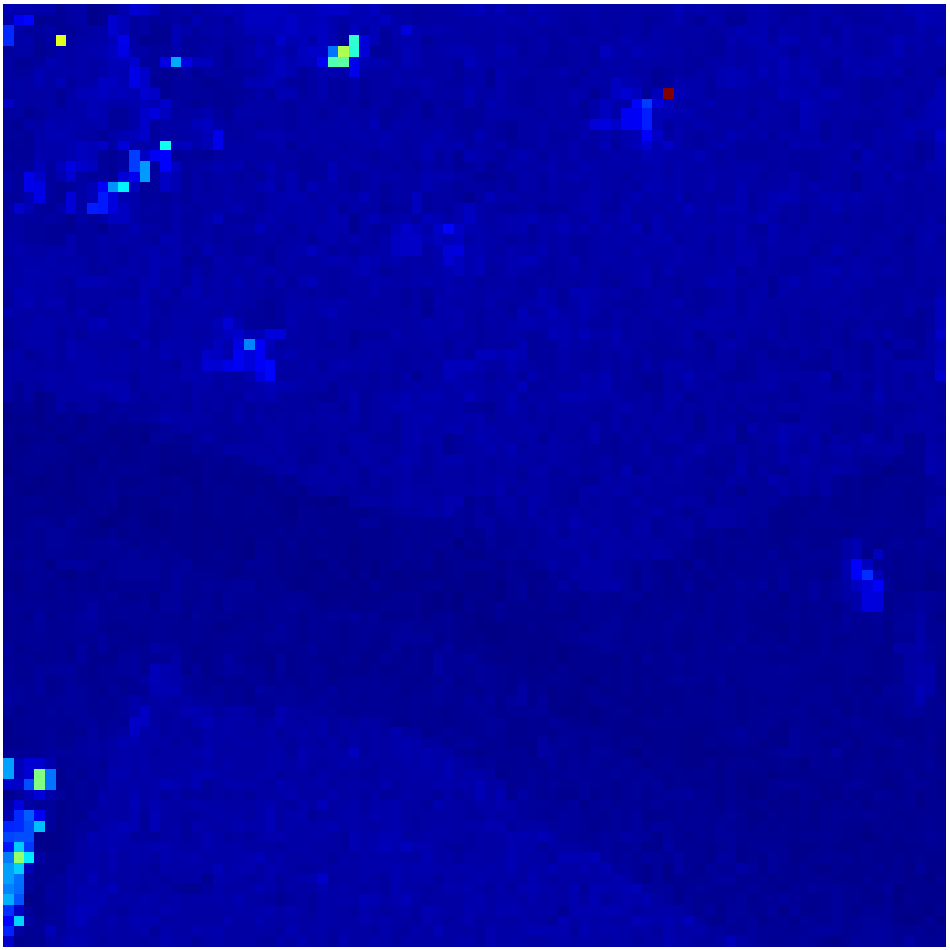}
			\caption{}
		\end{subfigure}\hfill
		\begin{subfigure}[b]{0.09\linewidth}
			\includegraphics[width=\linewidth]{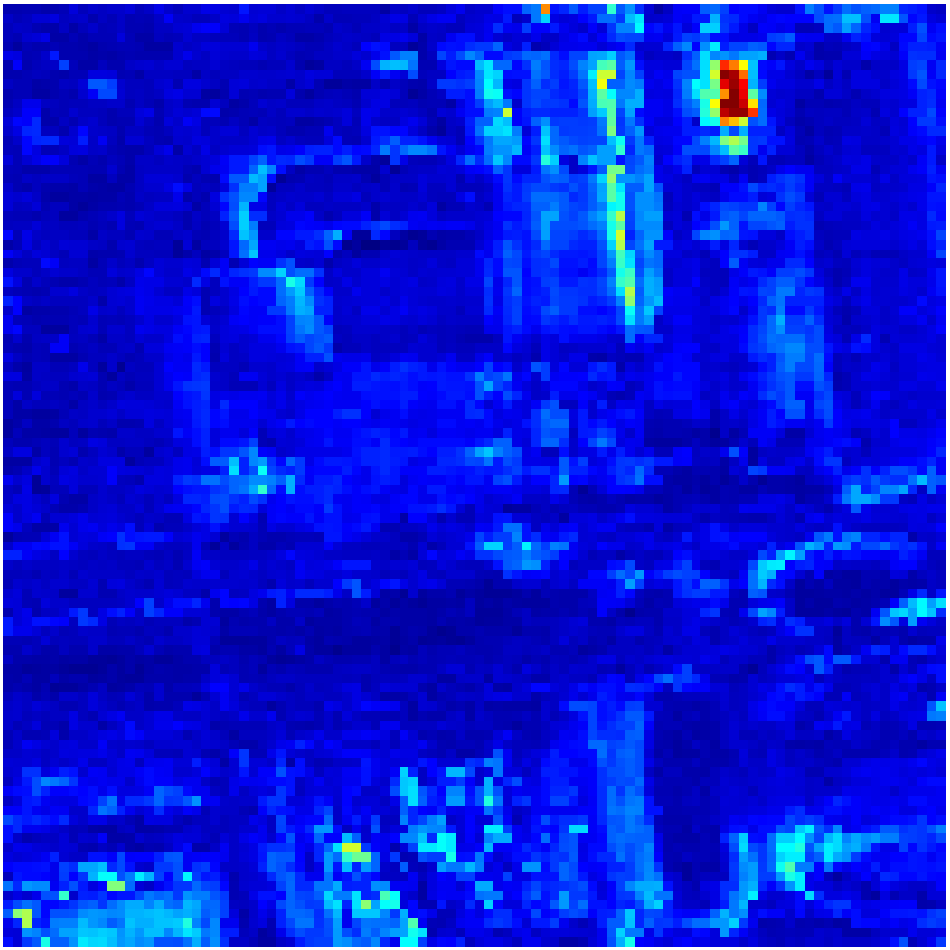}
			\includegraphics[width=\linewidth]{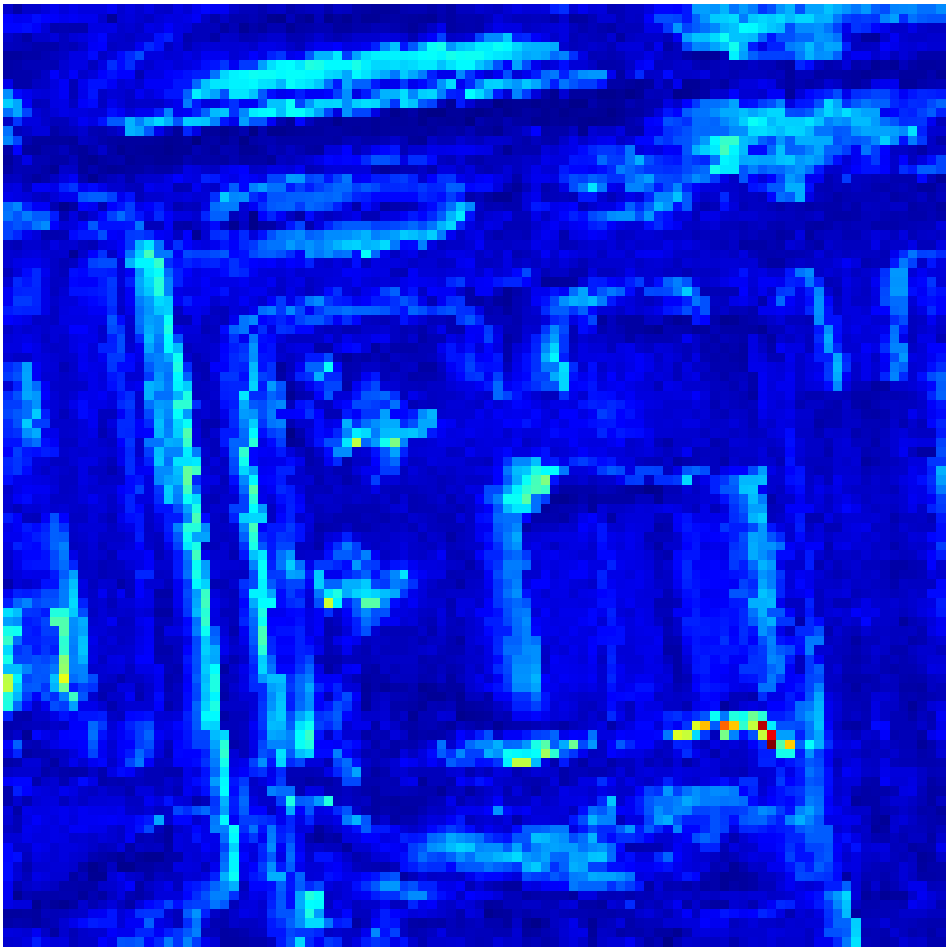}
			\includegraphics[width=\linewidth]{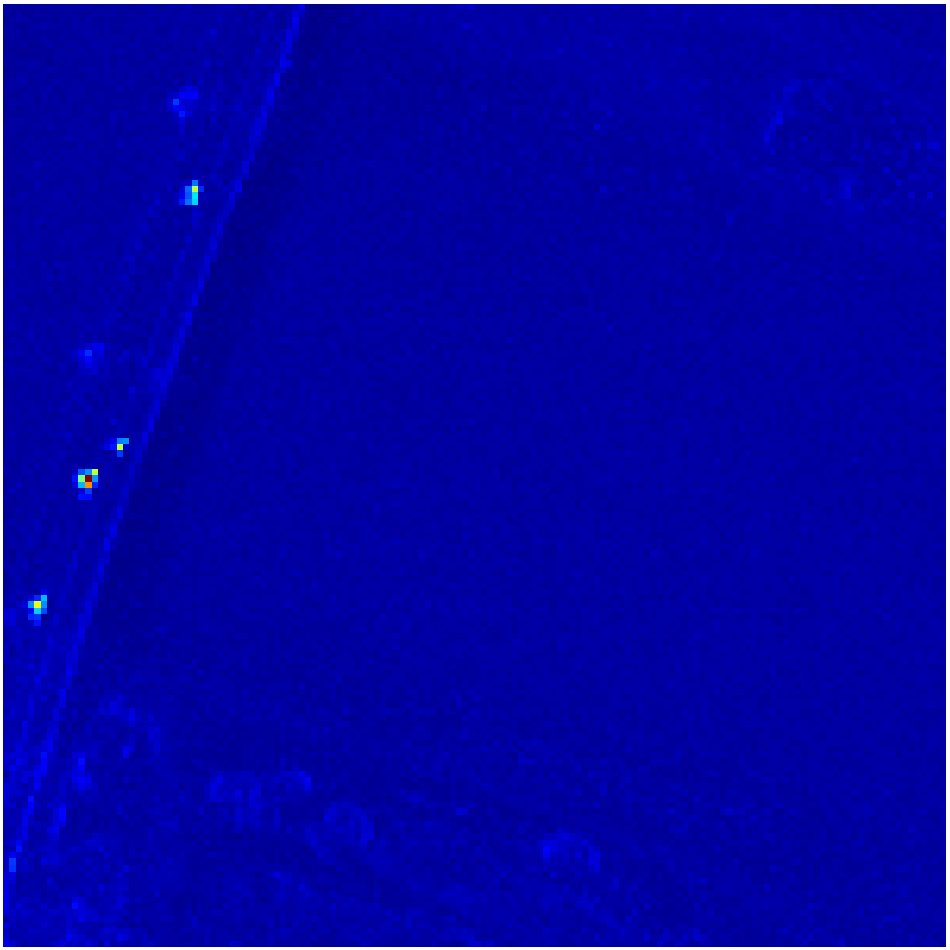}
			\includegraphics[width=\linewidth]{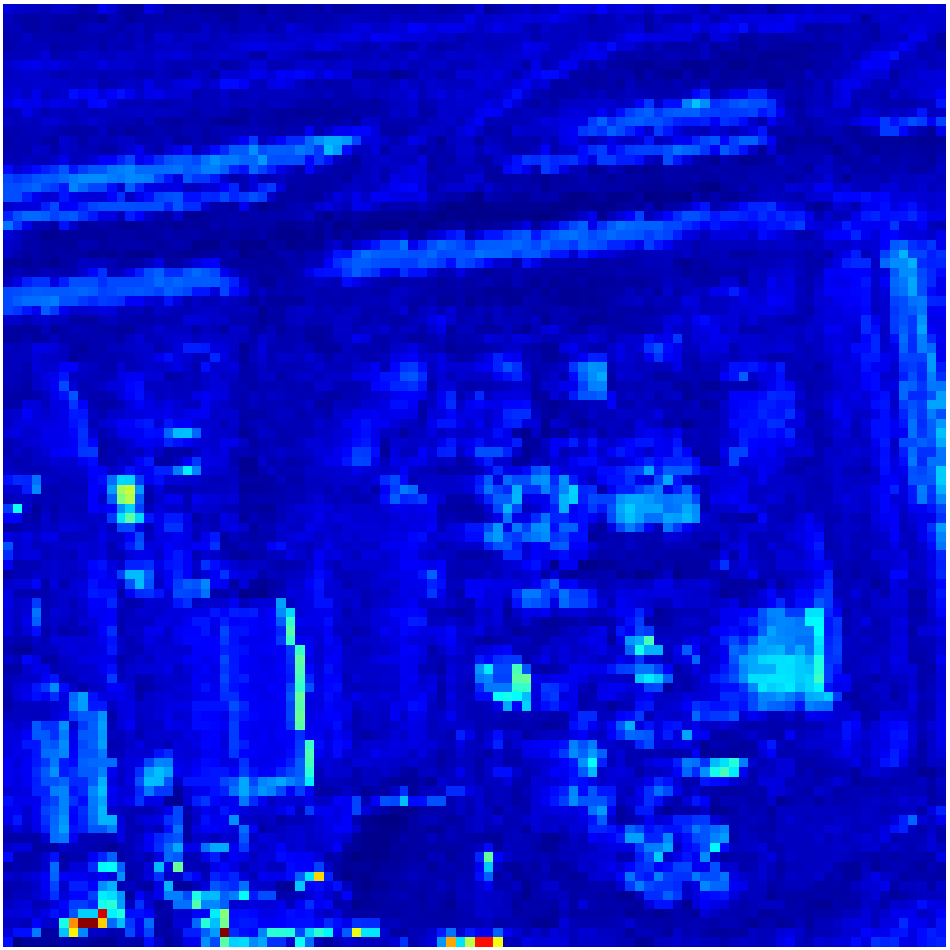}
			\includegraphics[width=\linewidth]{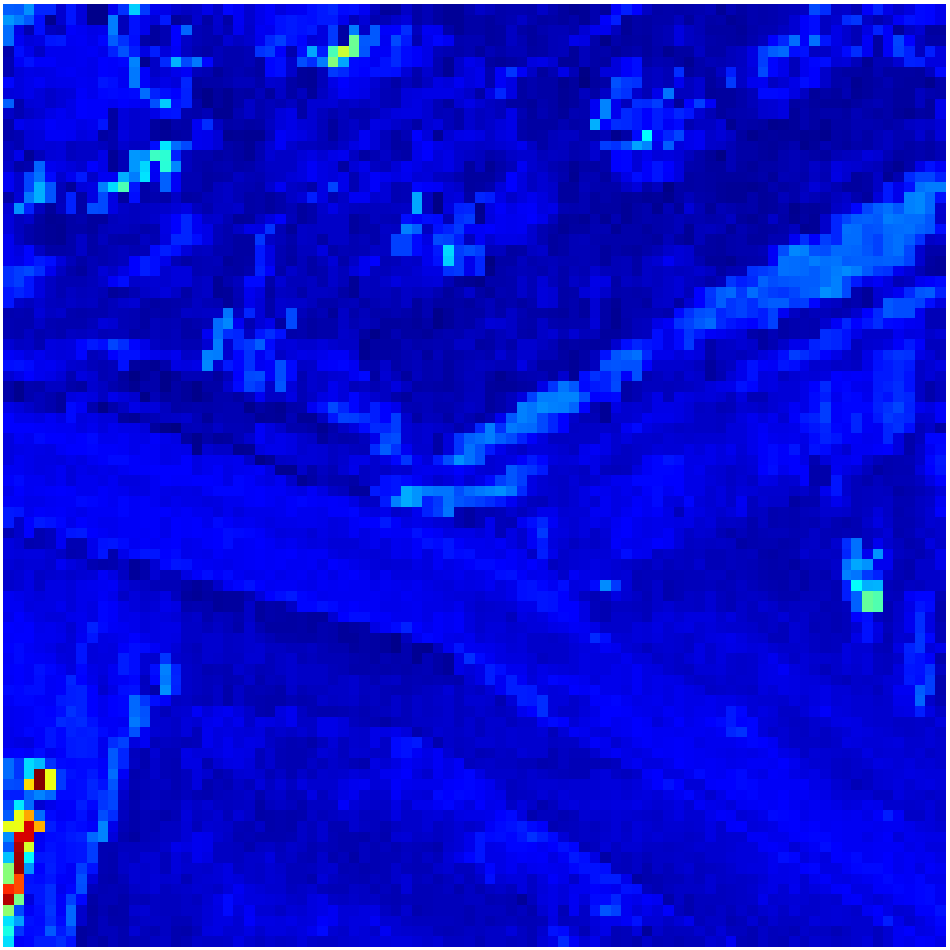}
			\caption{}
		\end{subfigure}\hfill
		\begin{subfigure}[b]{0.09\linewidth}
			\includegraphics[width=\linewidth]{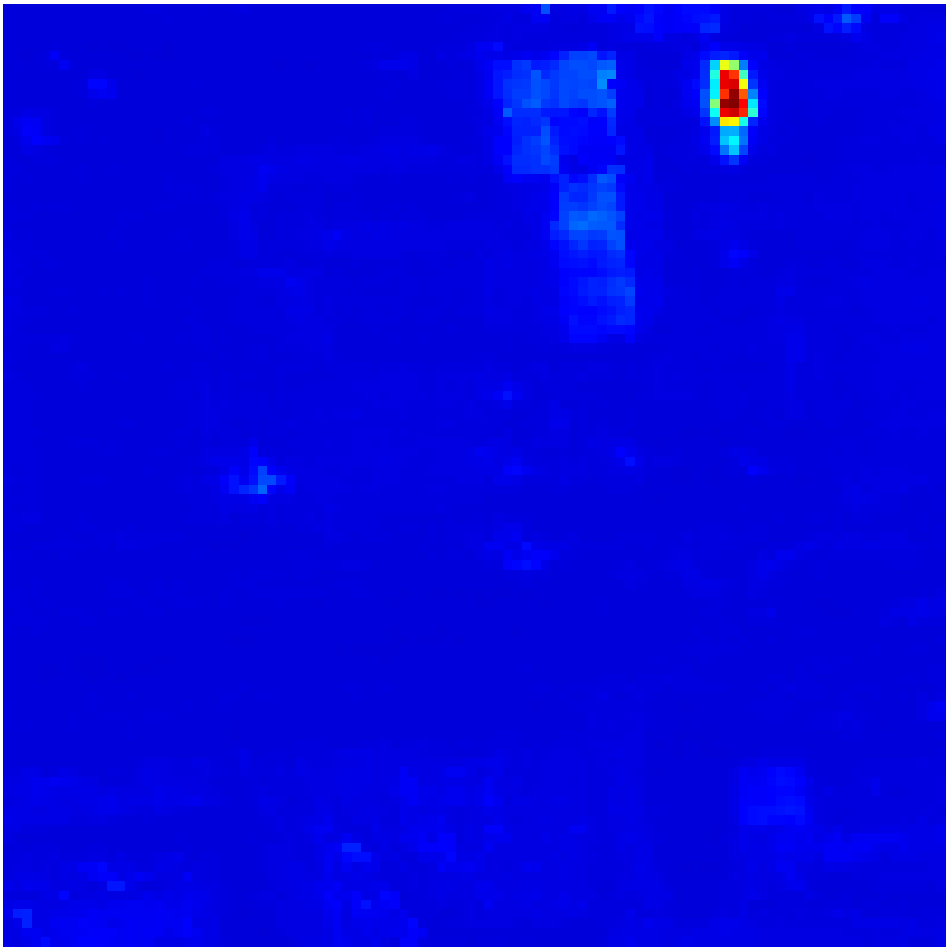}
			\includegraphics[width=\linewidth]{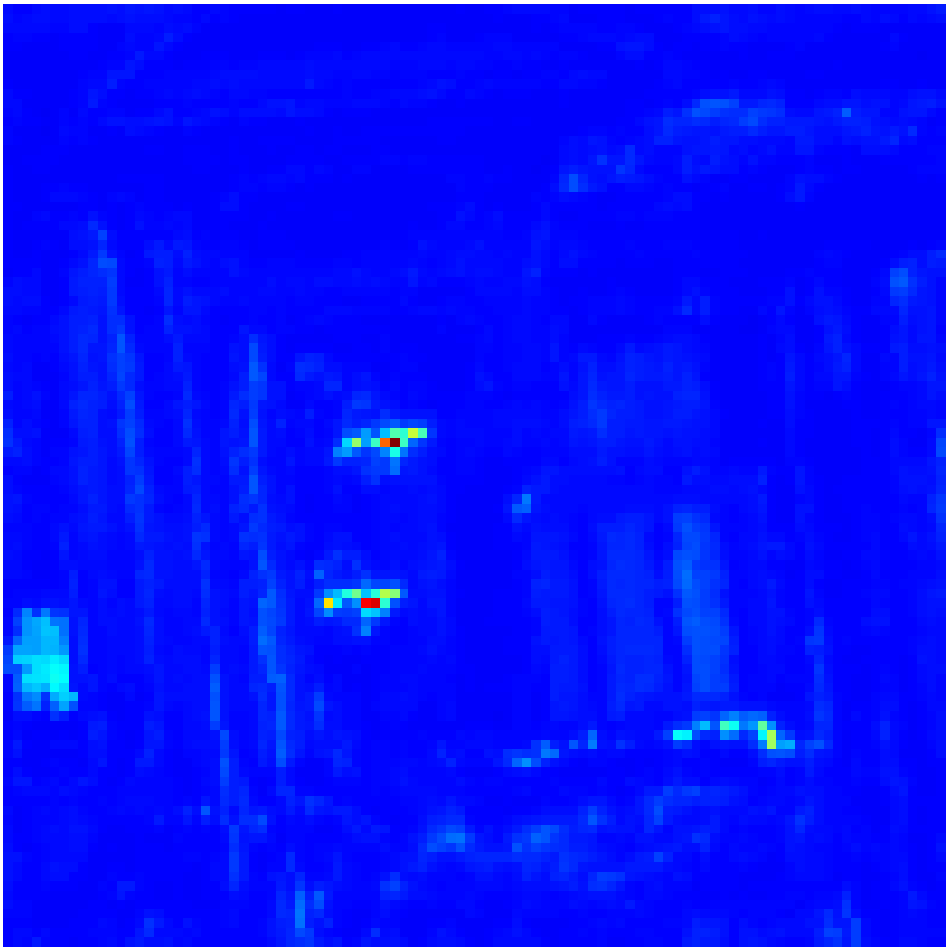}
			\includegraphics[width=\linewidth]{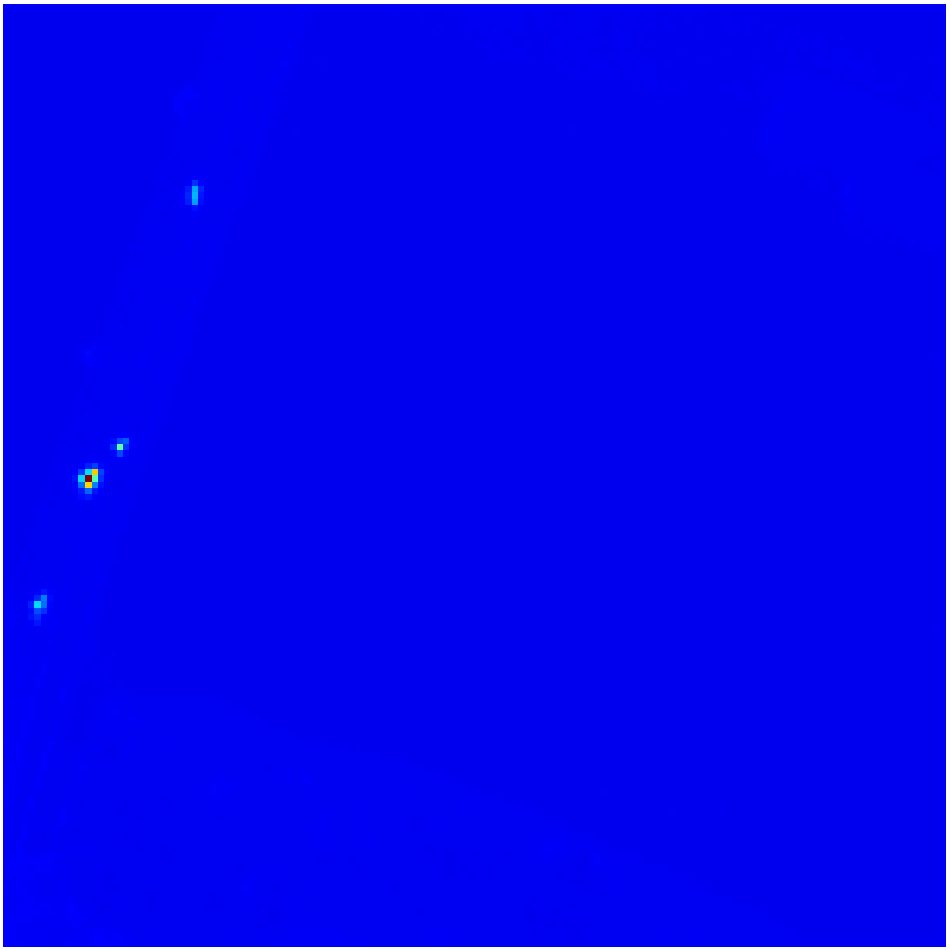}
			\includegraphics[width=\linewidth]{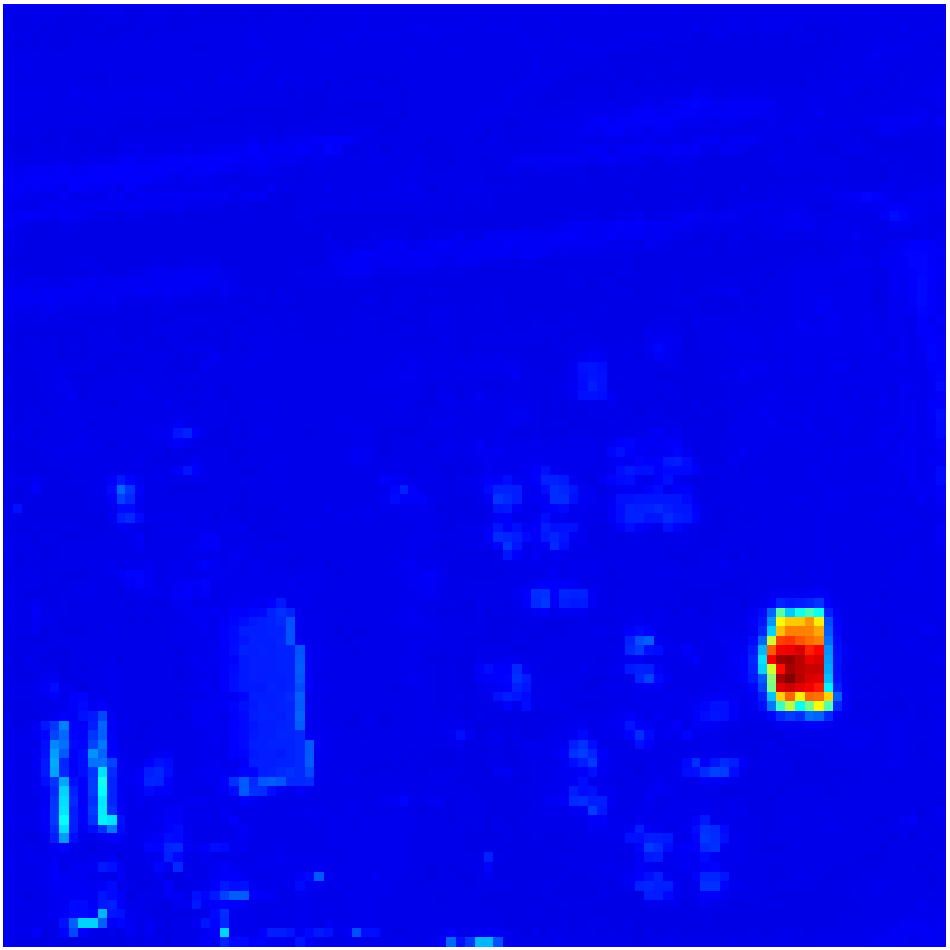}
			\includegraphics[width=\linewidth]{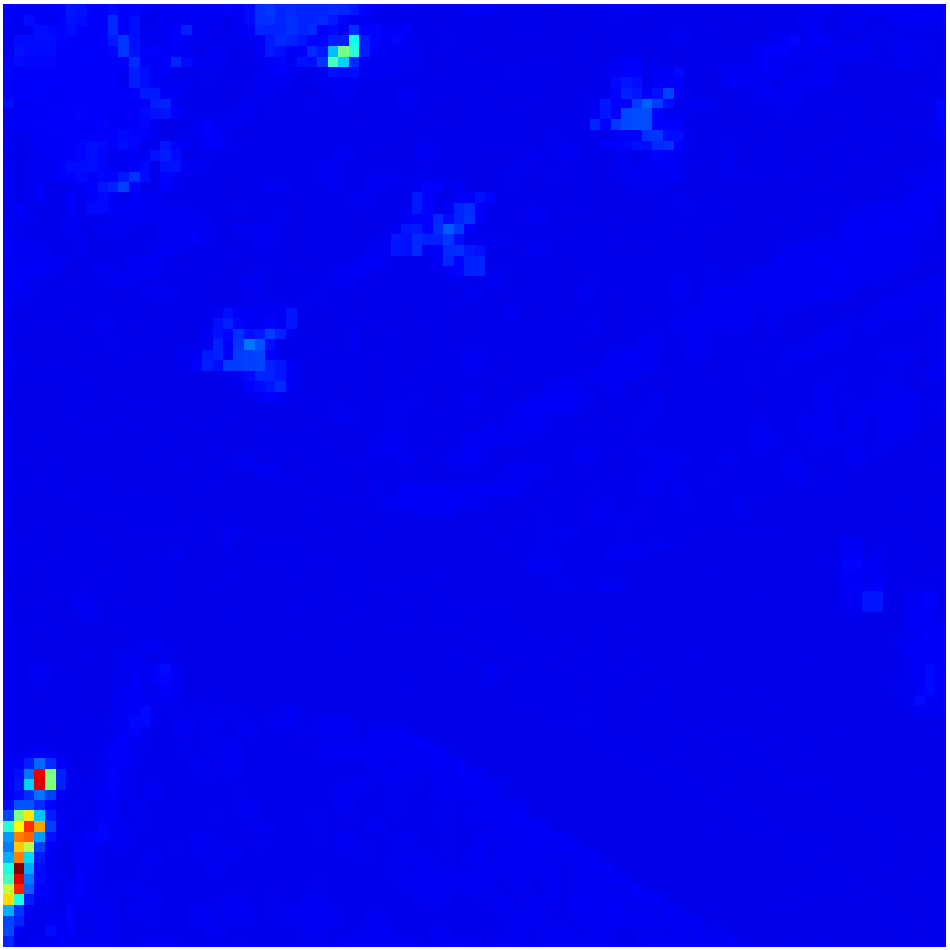}
			\caption{}
		\end{subfigure}\hfill
		\begin{subfigure}[b]{0.09\linewidth}
			\includegraphics[width=\linewidth]{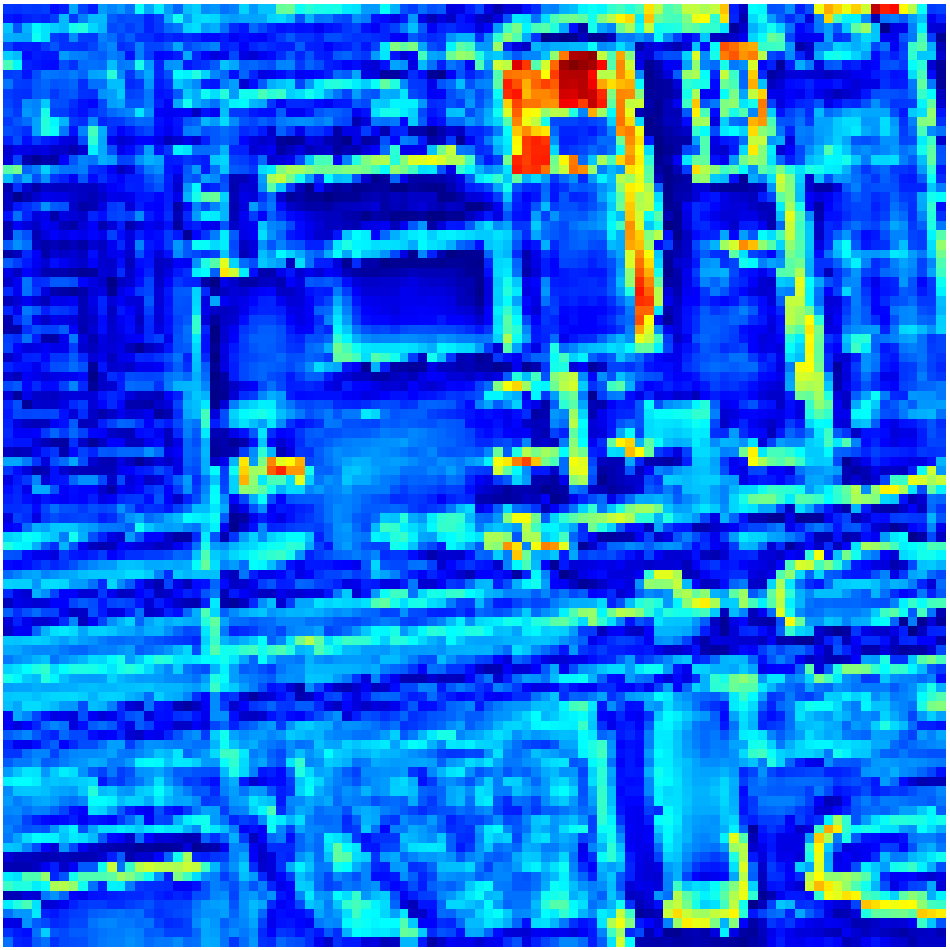}
			\includegraphics[width=\linewidth]{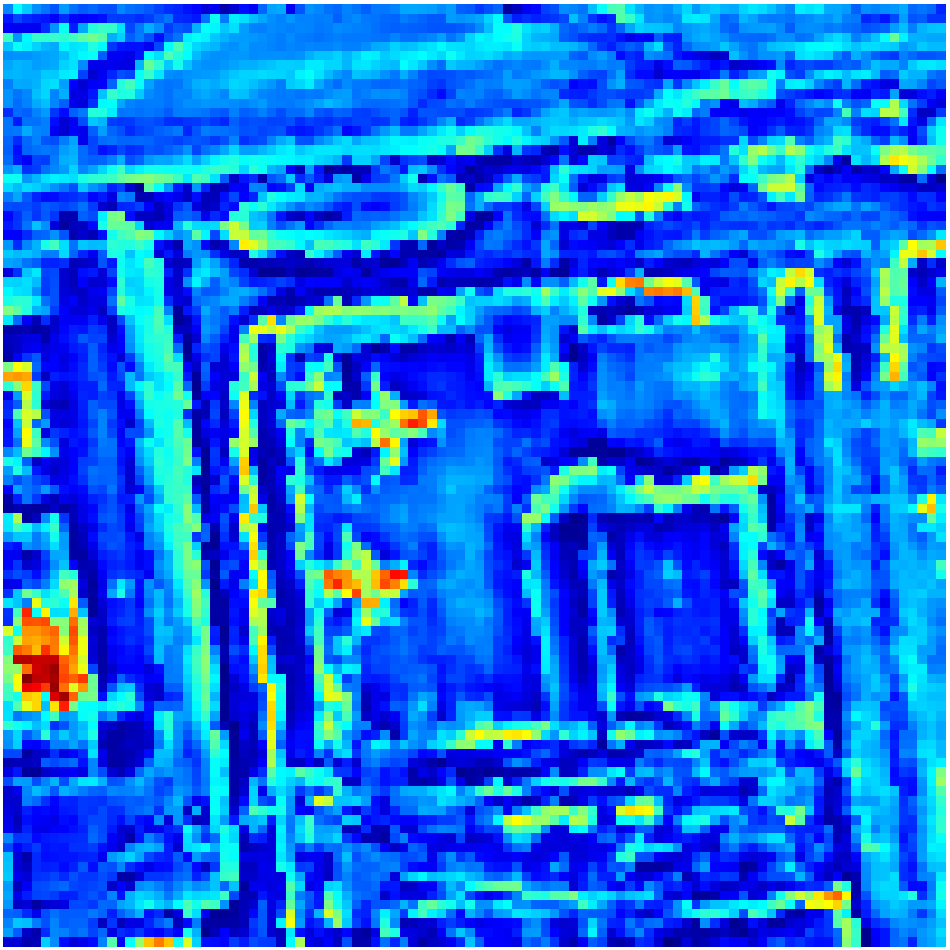}
			\includegraphics[width=\linewidth]{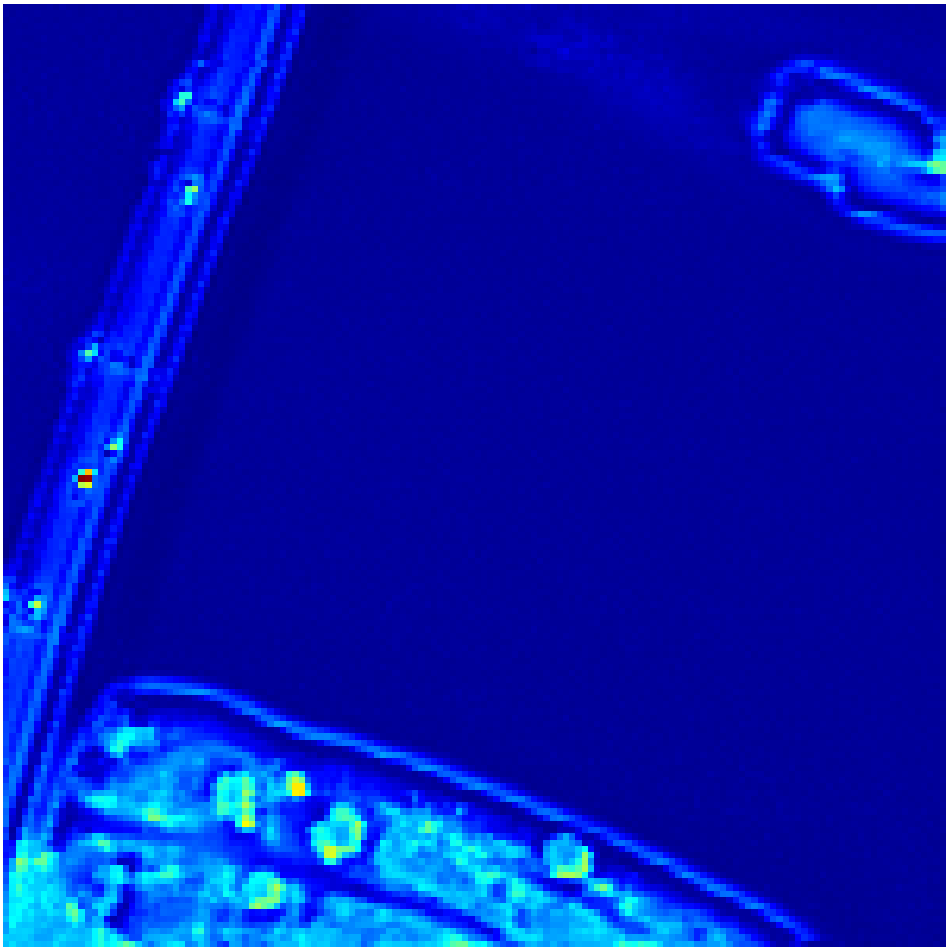}
			\includegraphics[width=\linewidth]{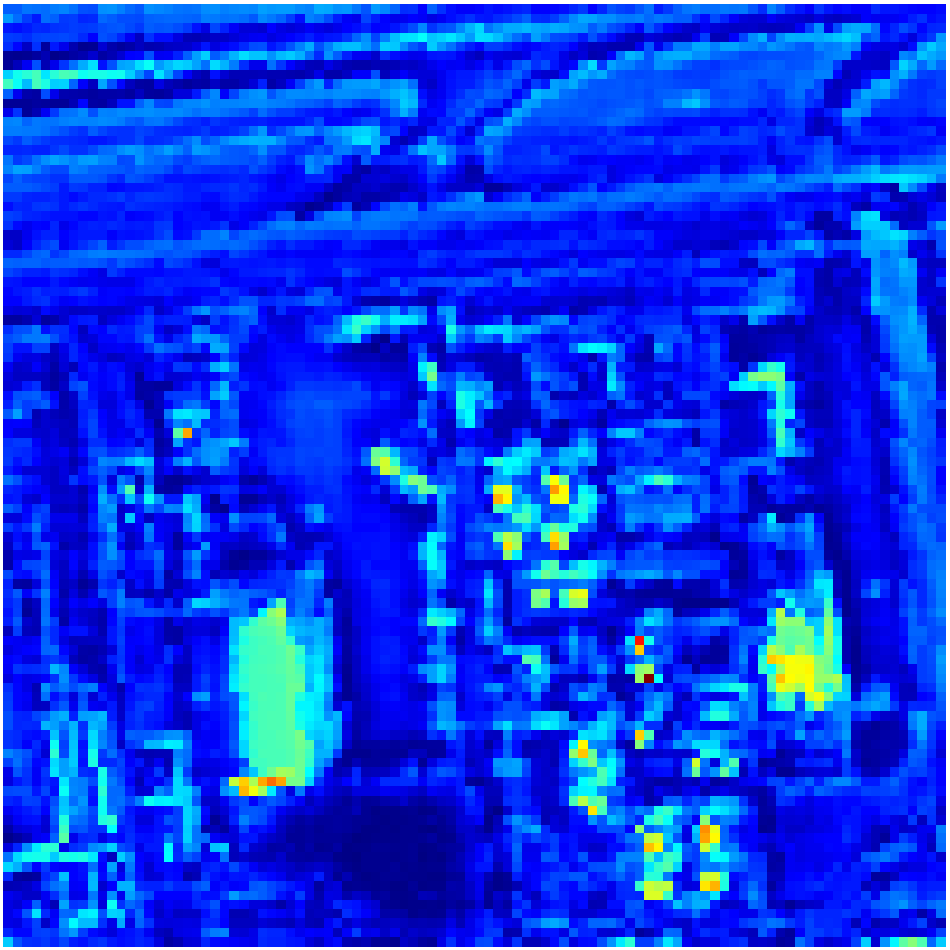}
			\includegraphics[width=\linewidth]{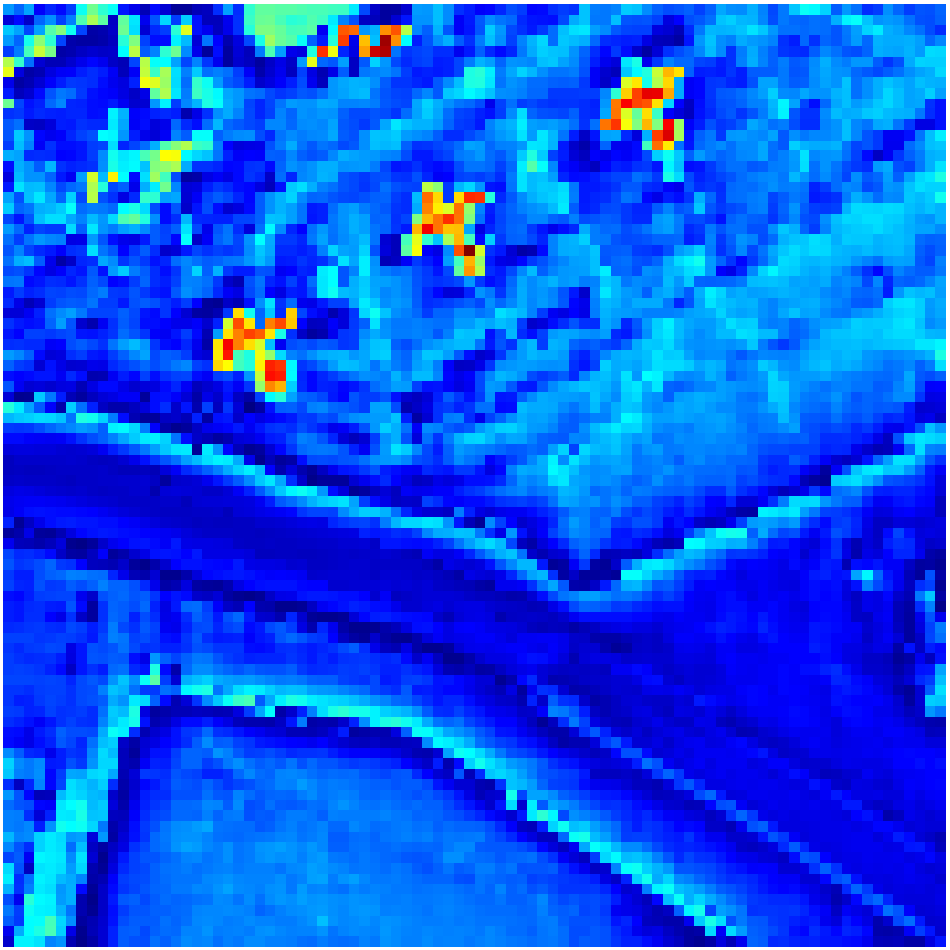}
			\caption{}
		\end{subfigure}\hfill
		\begin{subfigure}[b]{0.09\linewidth}
			\includegraphics[width=\linewidth]{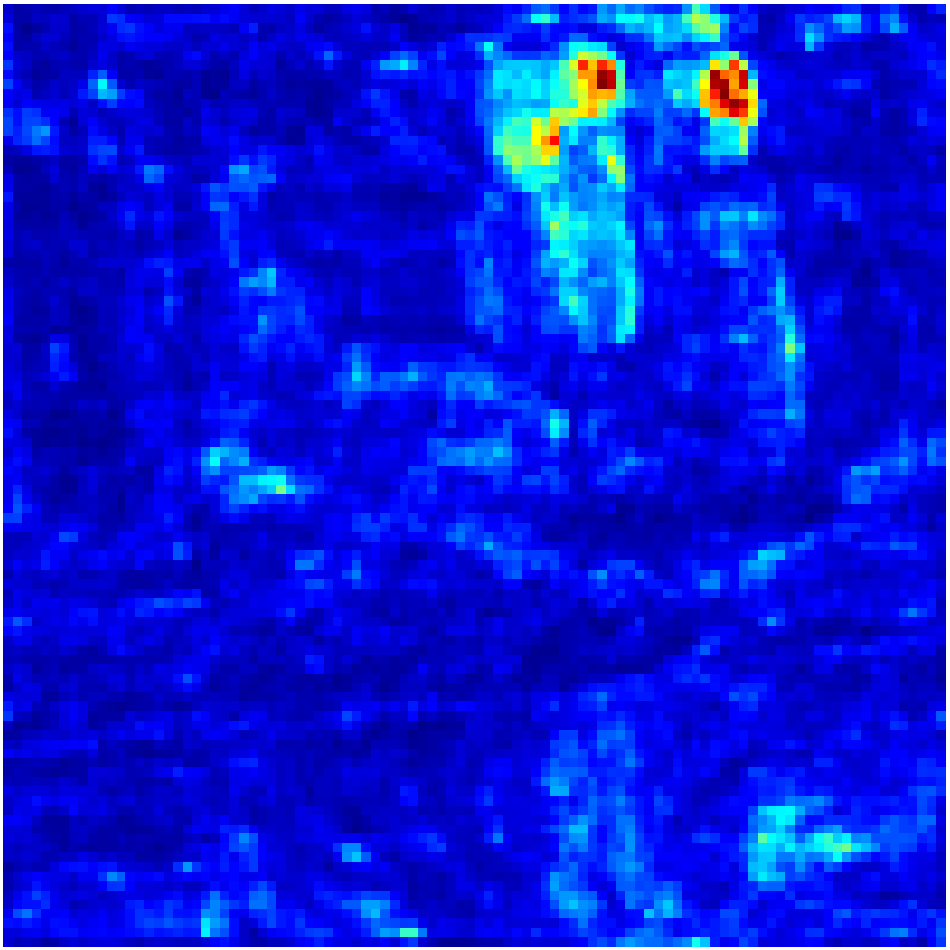}
			\includegraphics[width=\linewidth]{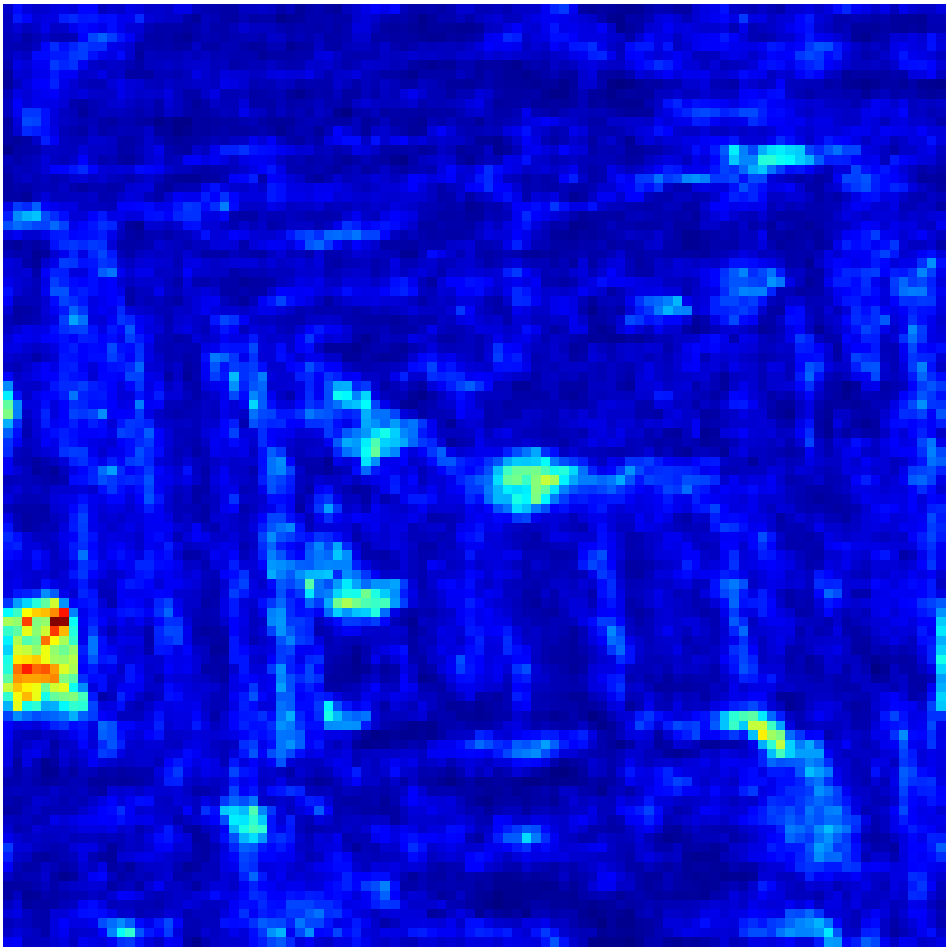}
			\includegraphics[width=\linewidth]{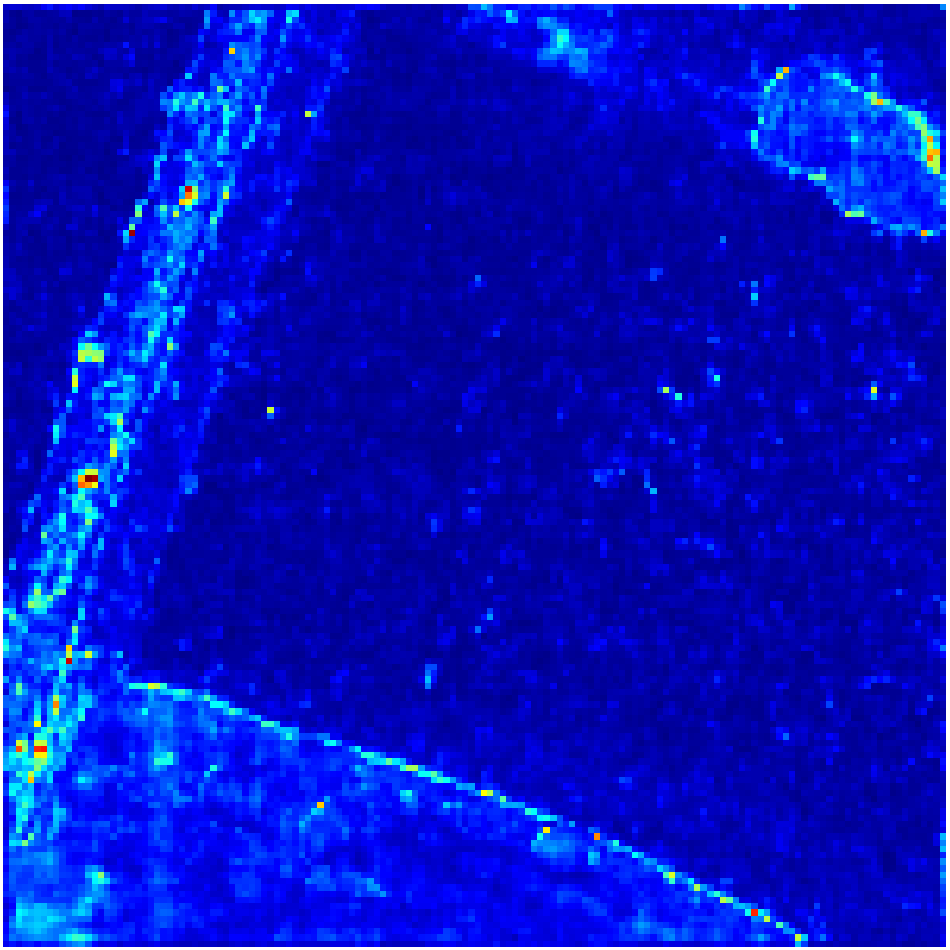}
			\includegraphics[width=\linewidth]{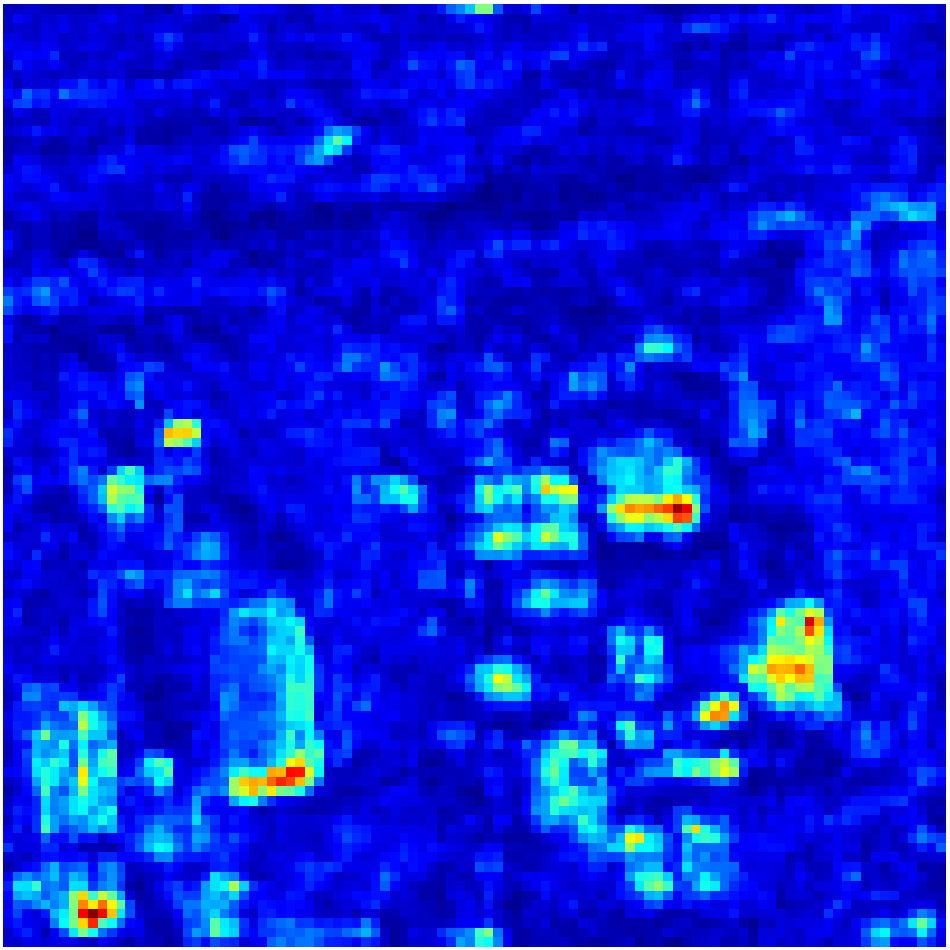}
			\includegraphics[width=\linewidth]{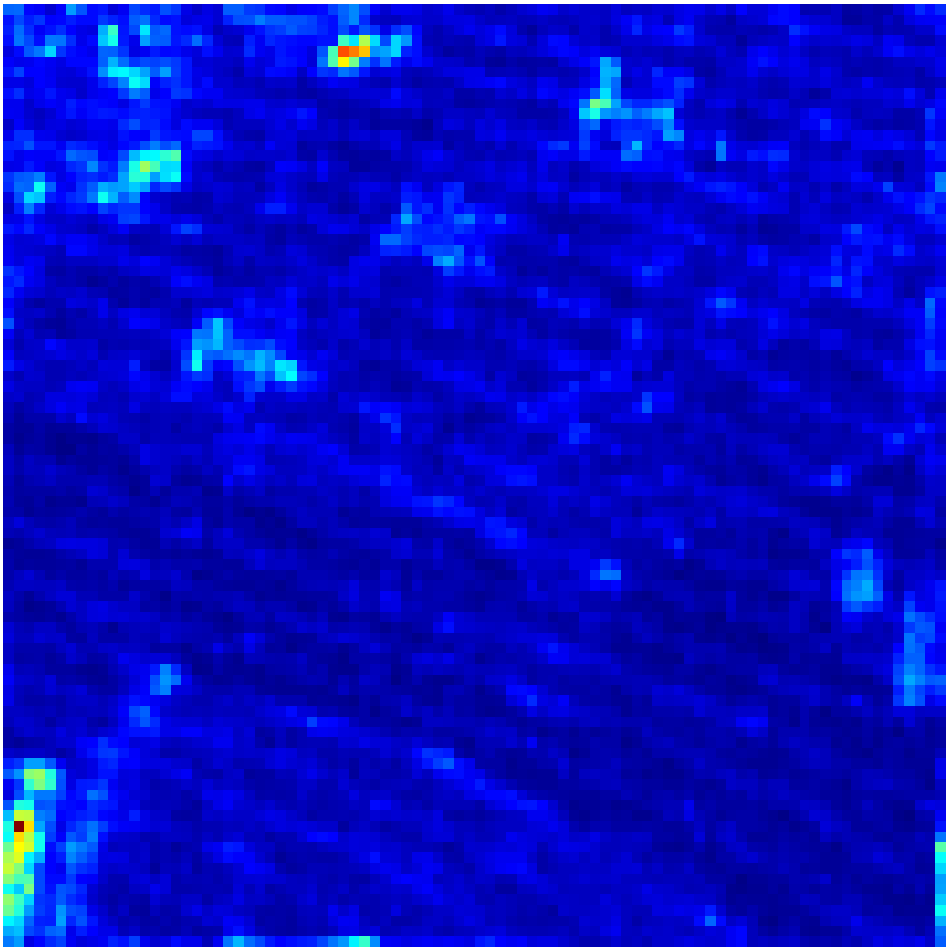}
			\caption{}
		\end{subfigure}\hfill
		\begin{subfigure}[b]{0.09\linewidth}
			\includegraphics[width=\linewidth]{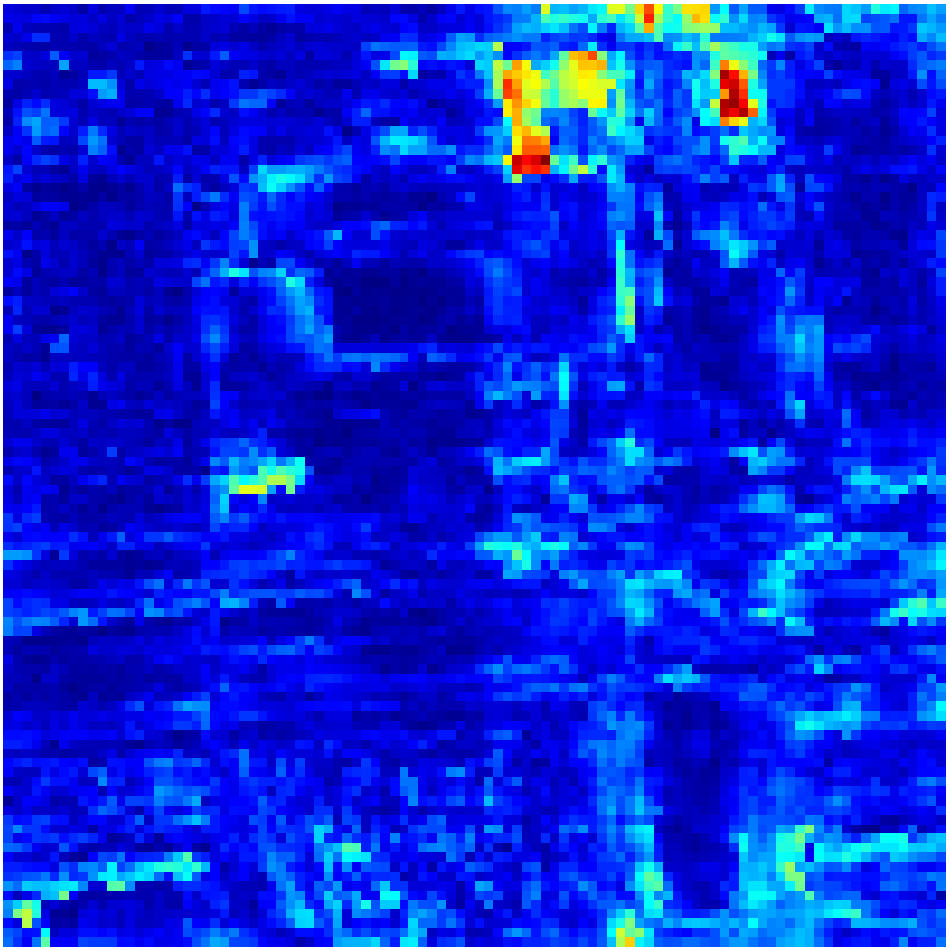}
			\includegraphics[width=\linewidth]{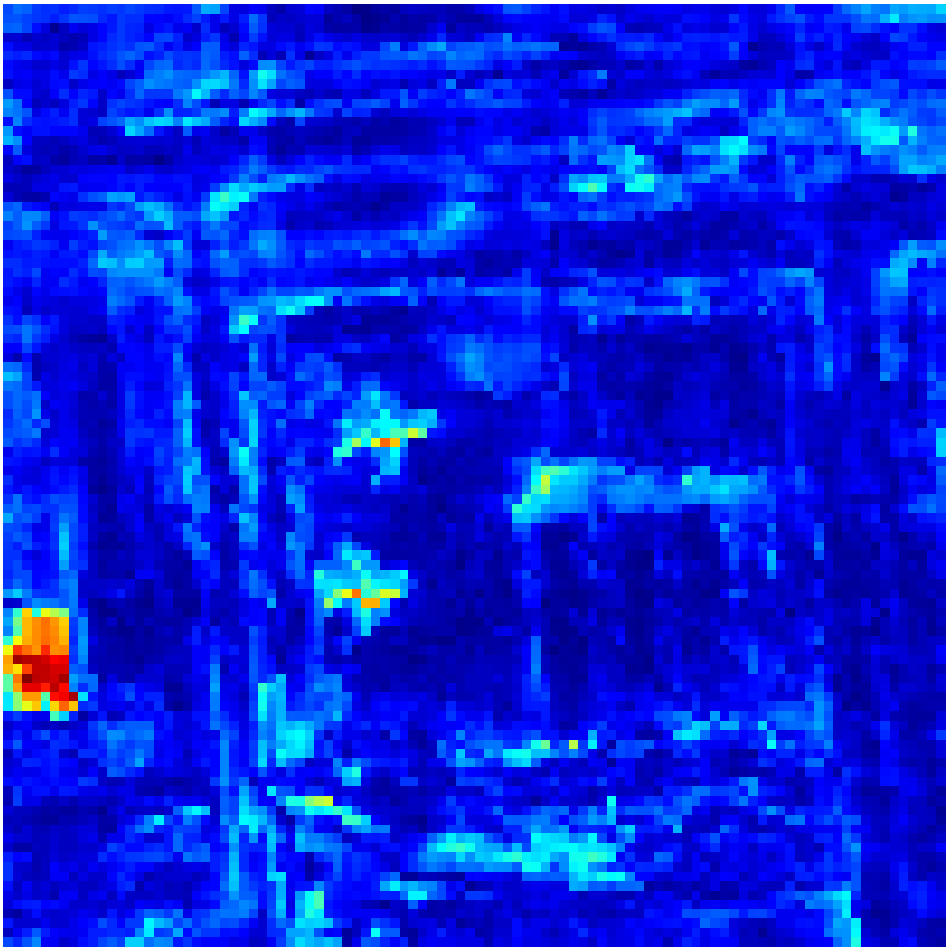}
			\includegraphics[width=\linewidth]{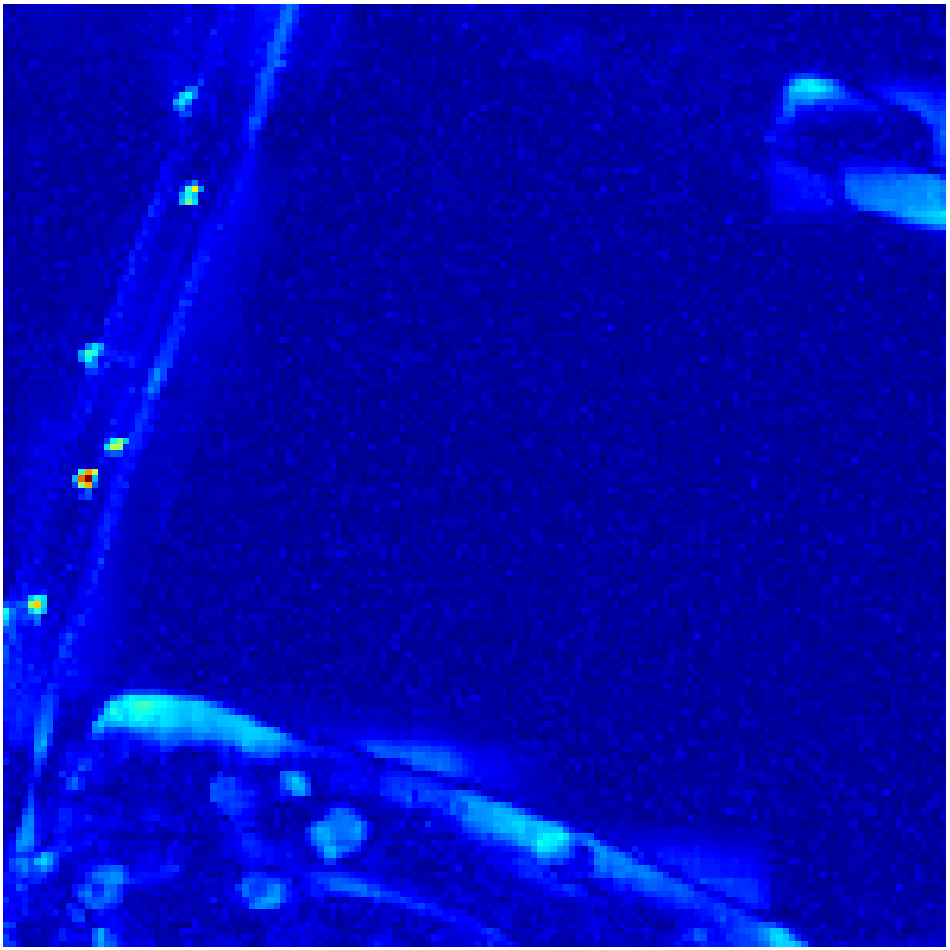}
			\includegraphics[width=\linewidth]{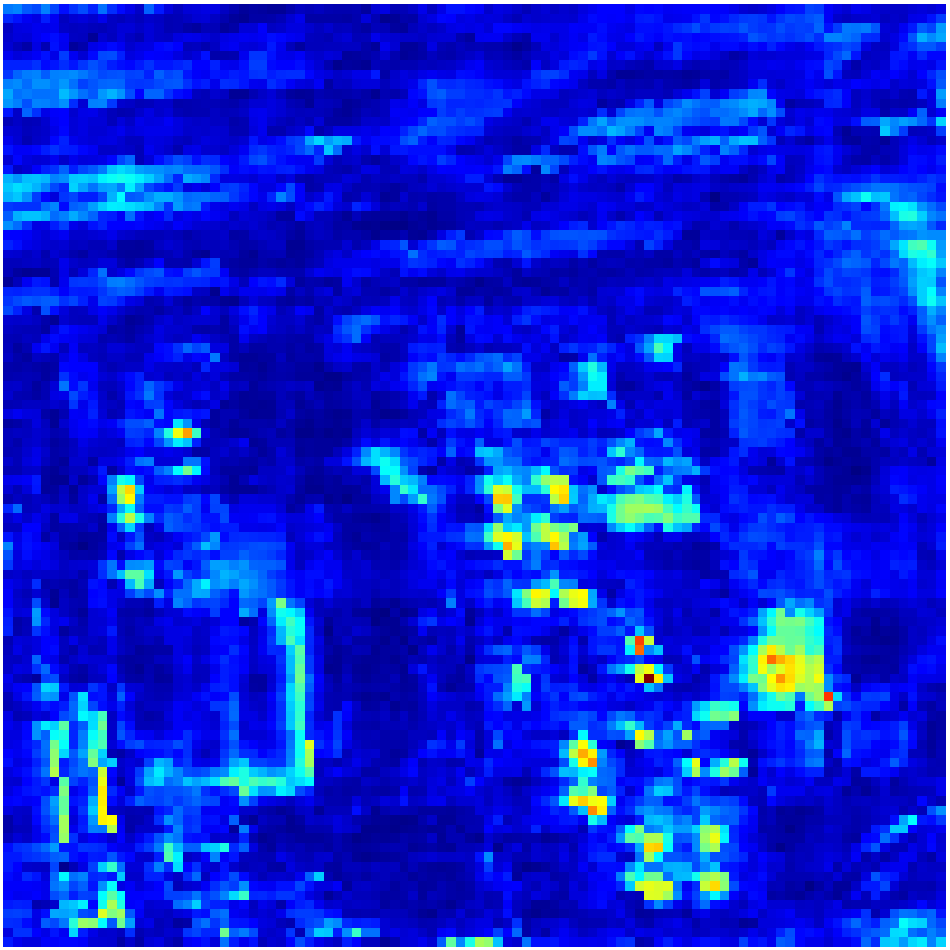}
			\includegraphics[width=\linewidth]{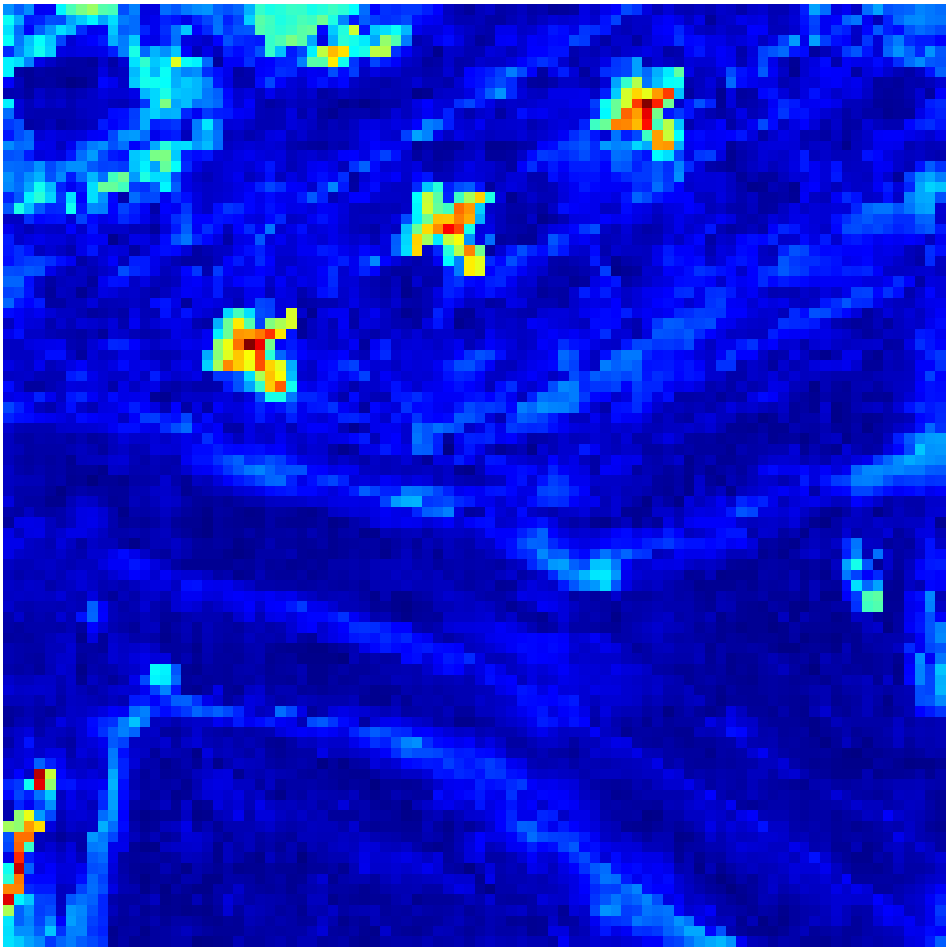}
			\caption{}
		\end{subfigure}\hfill
		\begin{subfigure}[b]{0.09\linewidth}
			\includegraphics[width=\linewidth]{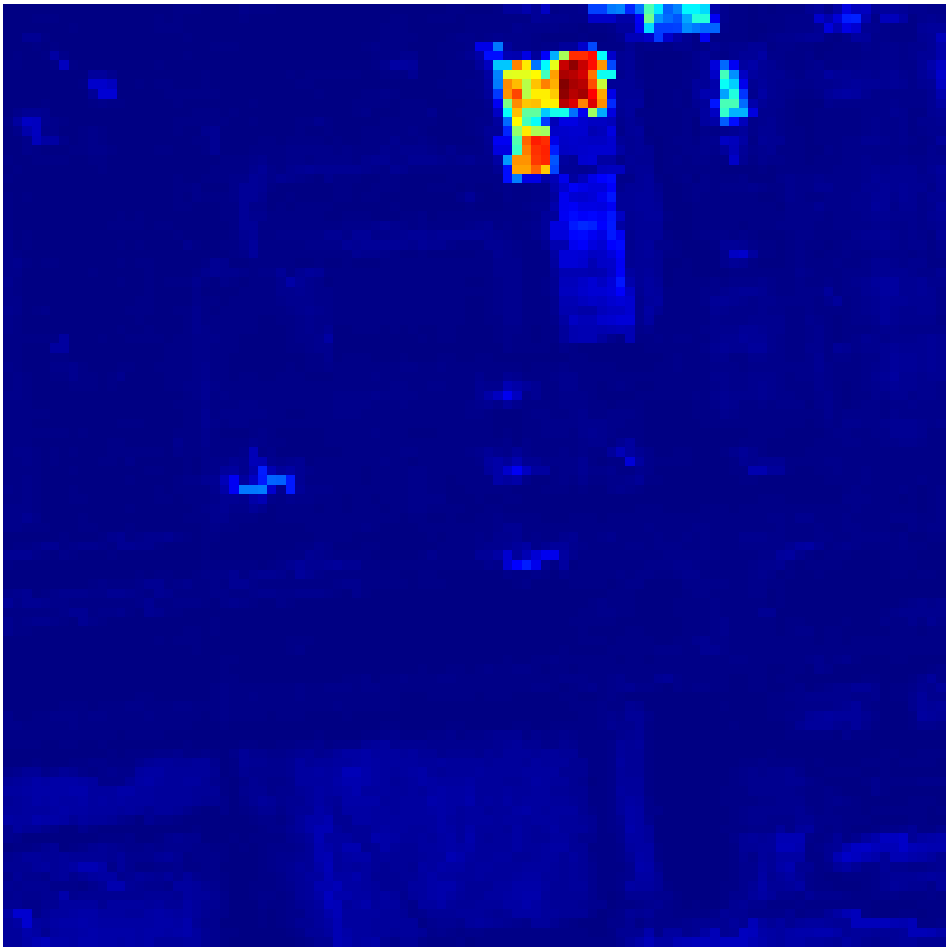}
			\includegraphics[width=\linewidth]{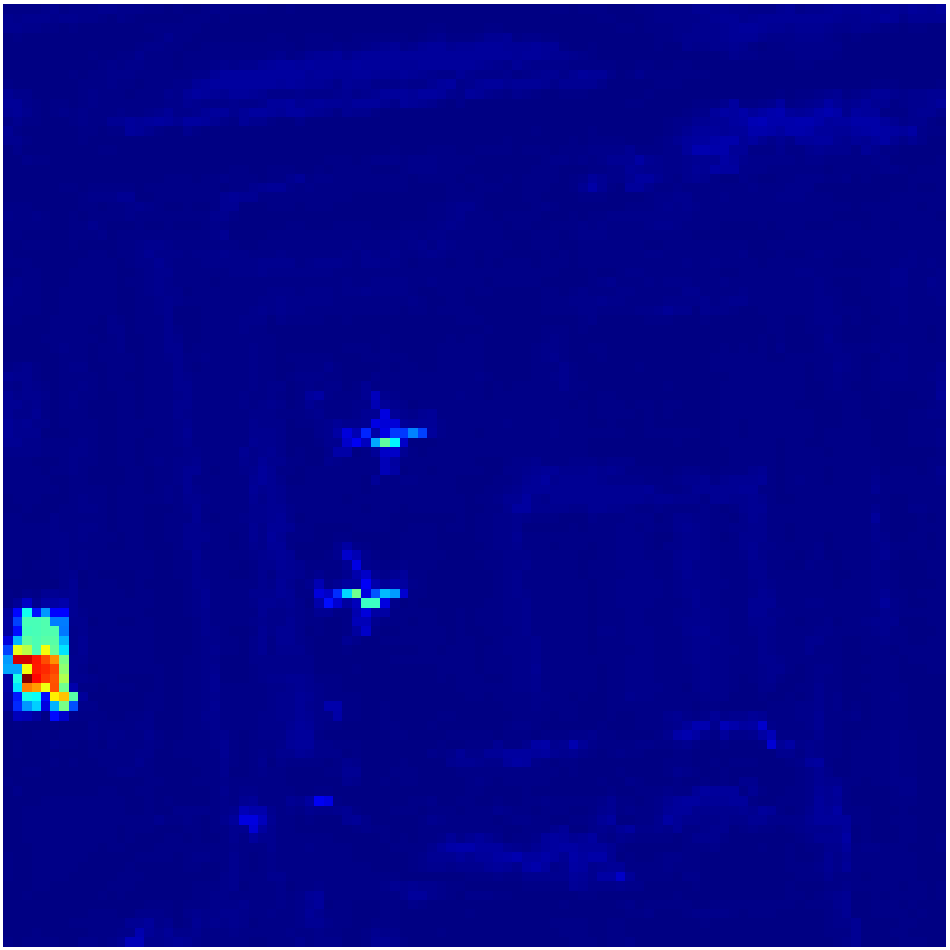}
			\includegraphics[width=\linewidth]{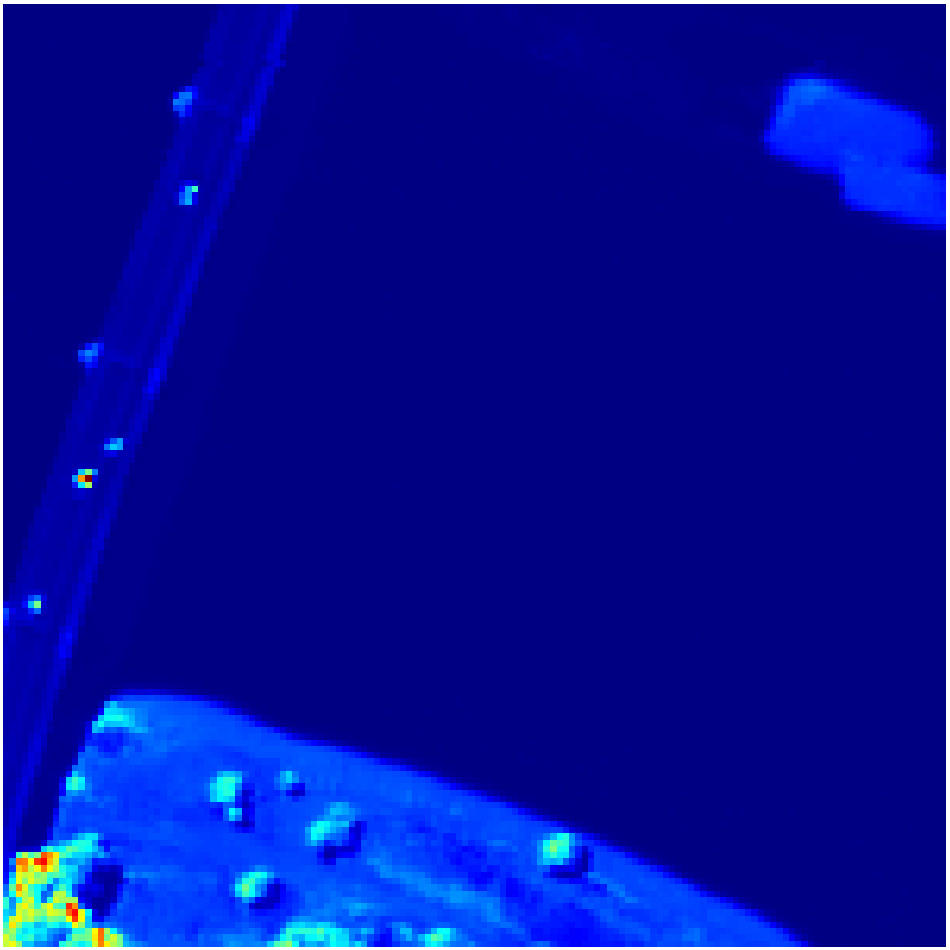}
			\includegraphics[width=\linewidth]{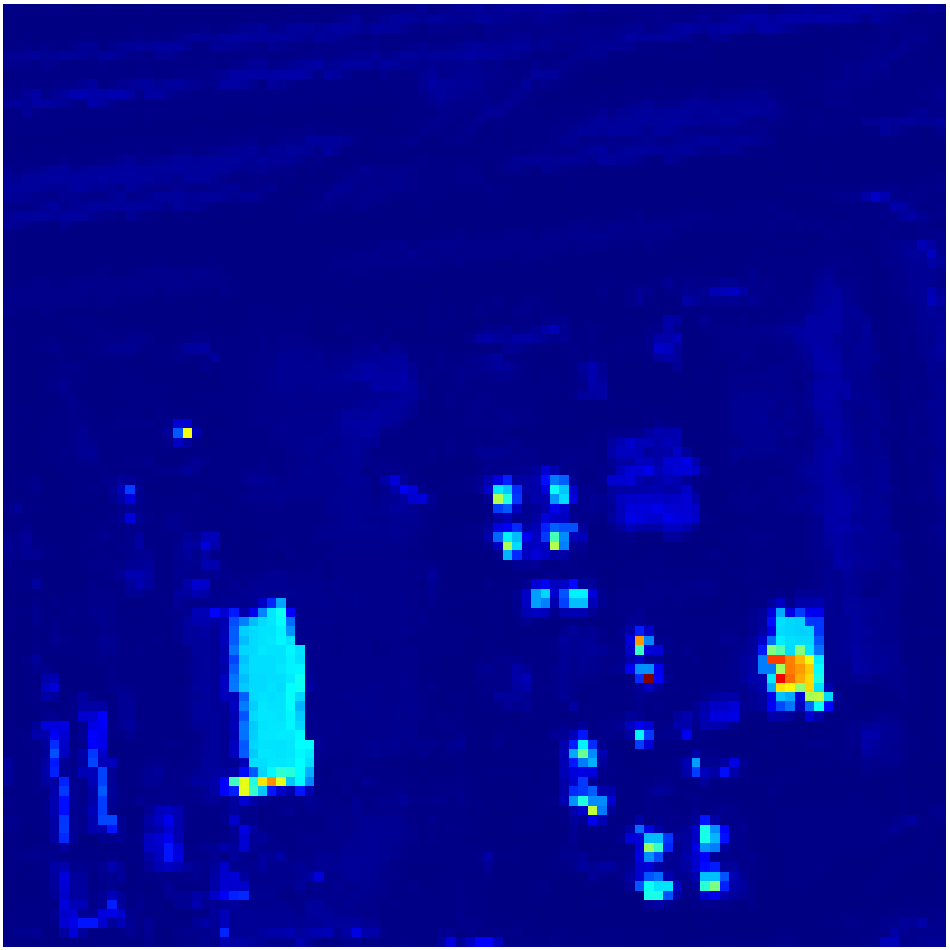}
			\includegraphics[width=\linewidth]{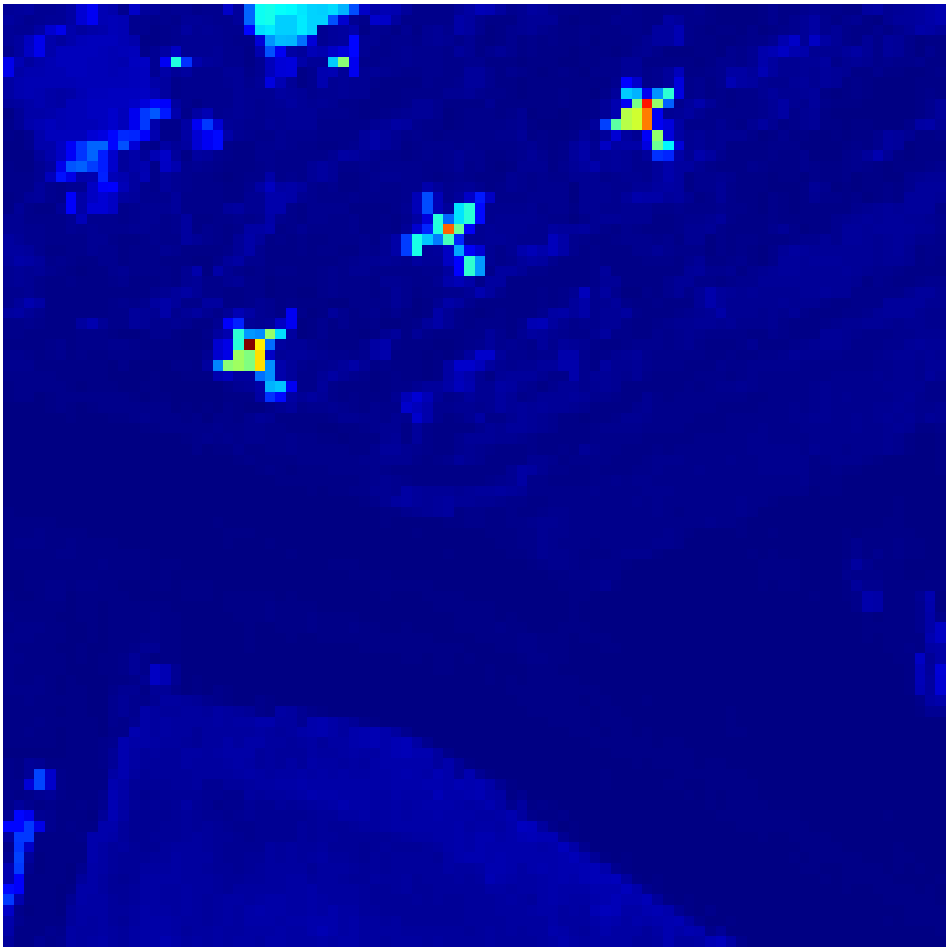}
			\caption{}
		\end{subfigure}\hfill
		\begin{subfigure}[b]{0.09\linewidth}
			\includegraphics[width=\linewidth]{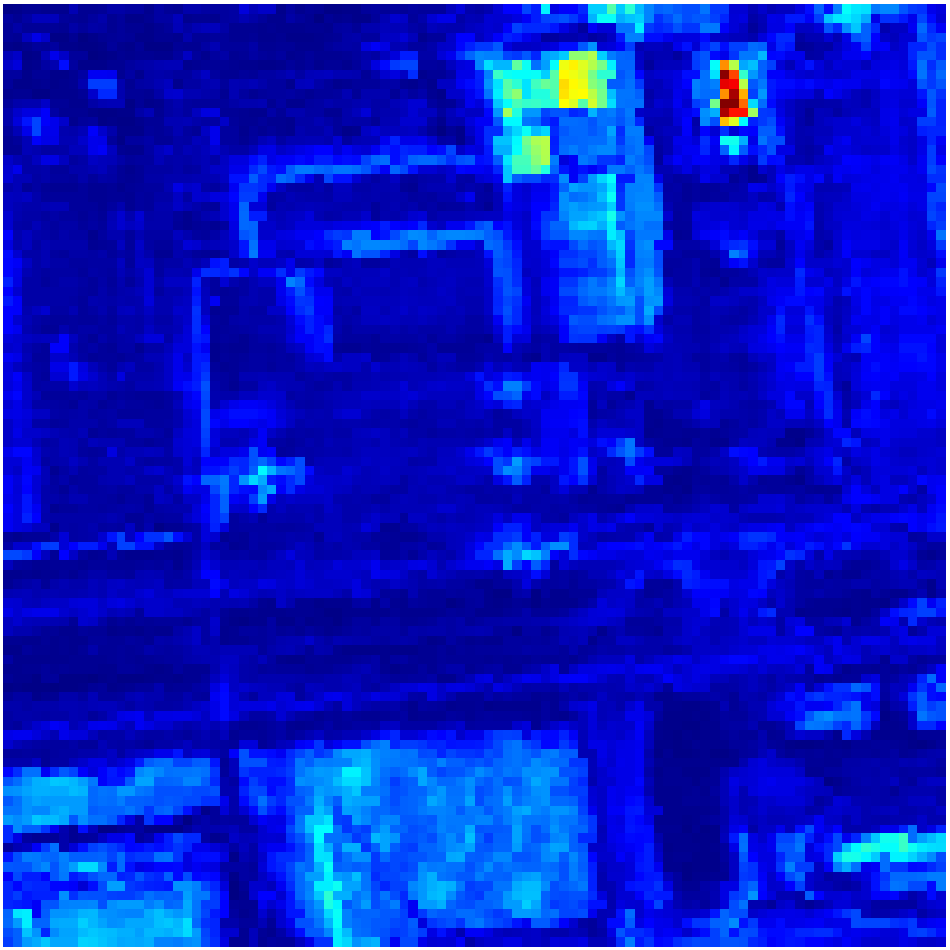}
			\includegraphics[width=\linewidth]{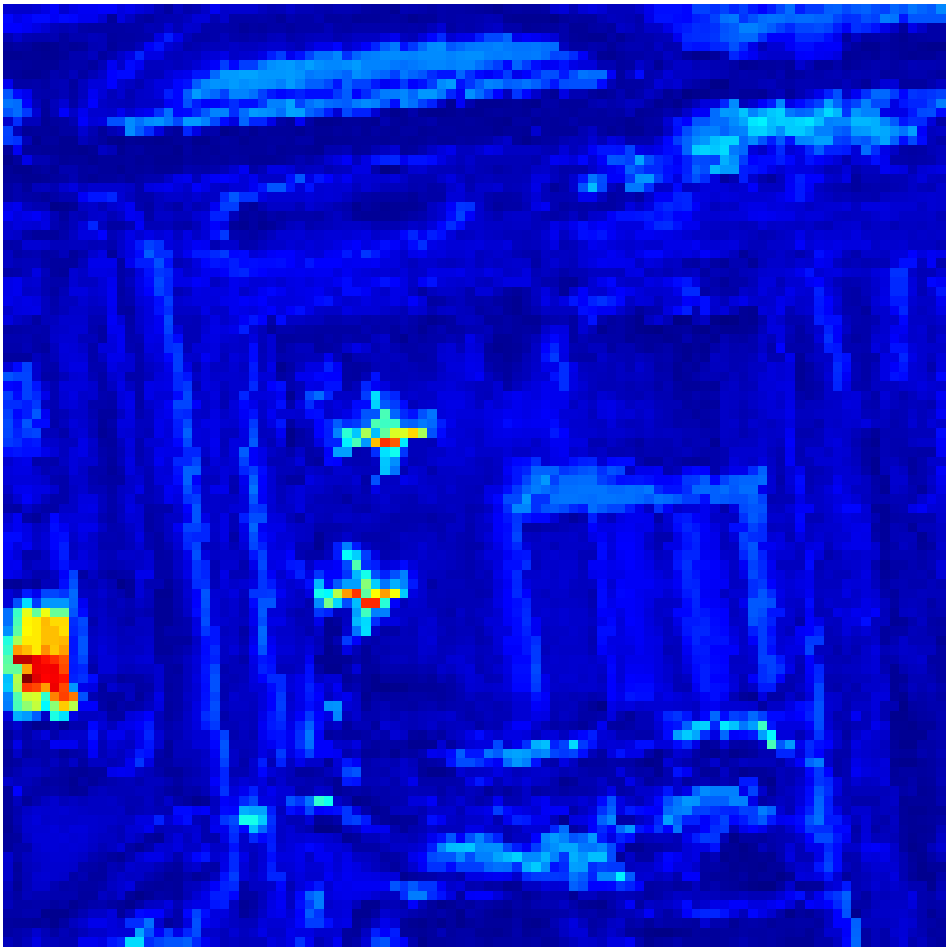}
			\includegraphics[width=\linewidth]{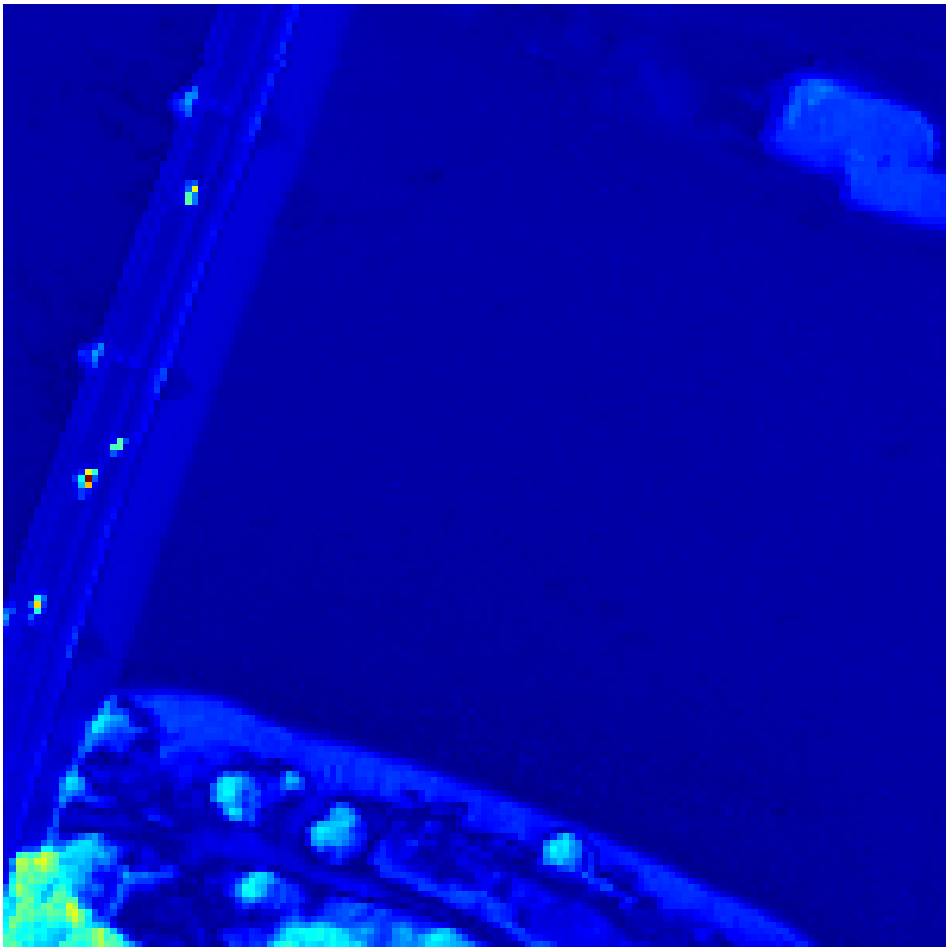}
			\includegraphics[width=\linewidth]{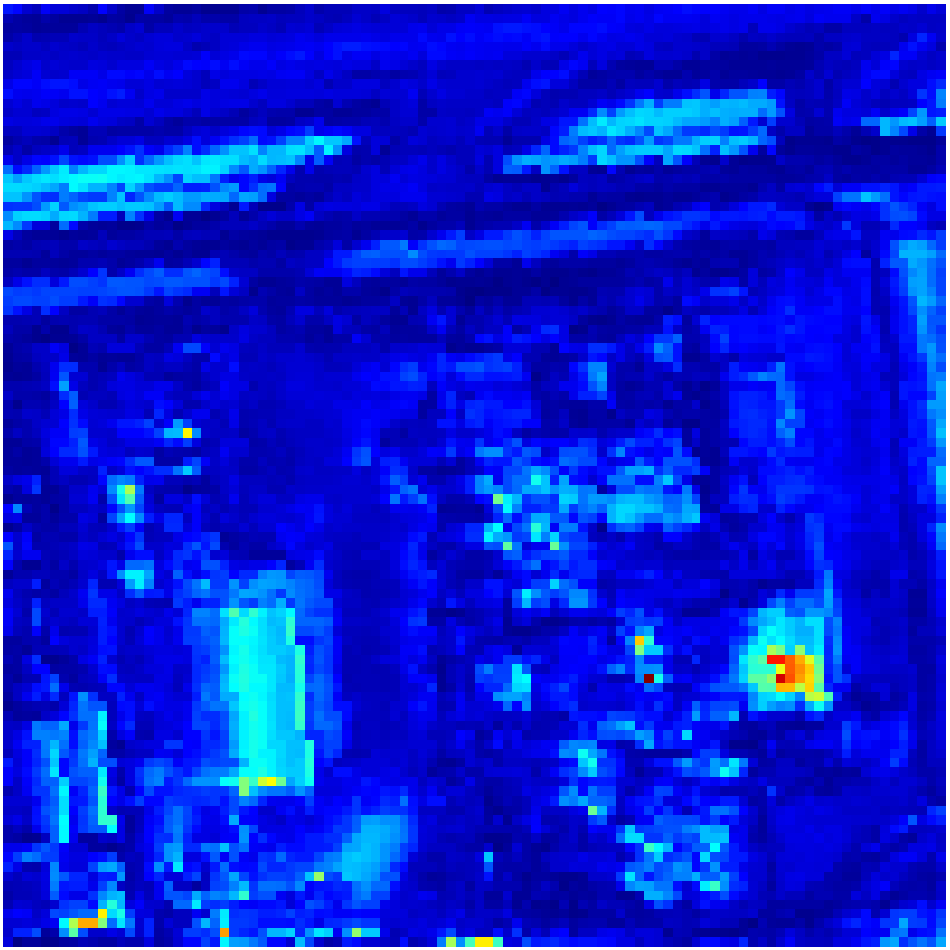}
			\includegraphics[width=\linewidth]{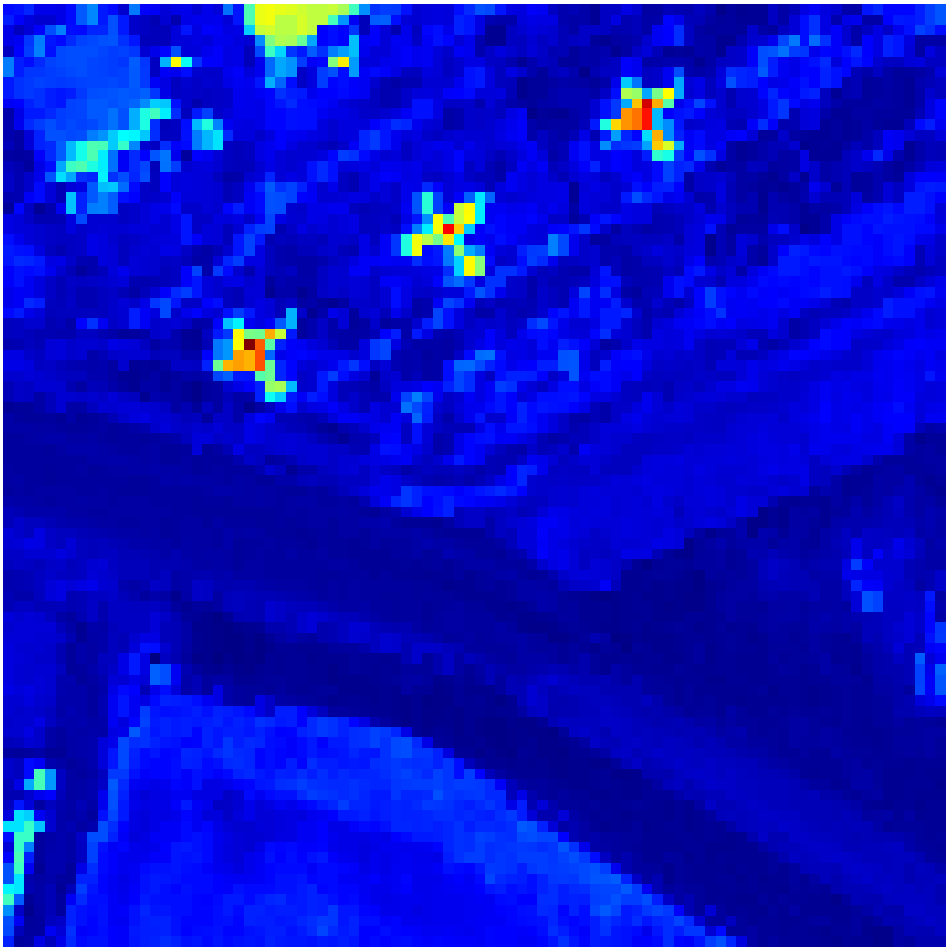}
			\caption{}
		\end{subfigure}\hfill
		\begin{subfigure}[b]{0.09\linewidth}
			\includegraphics[width=\linewidth]{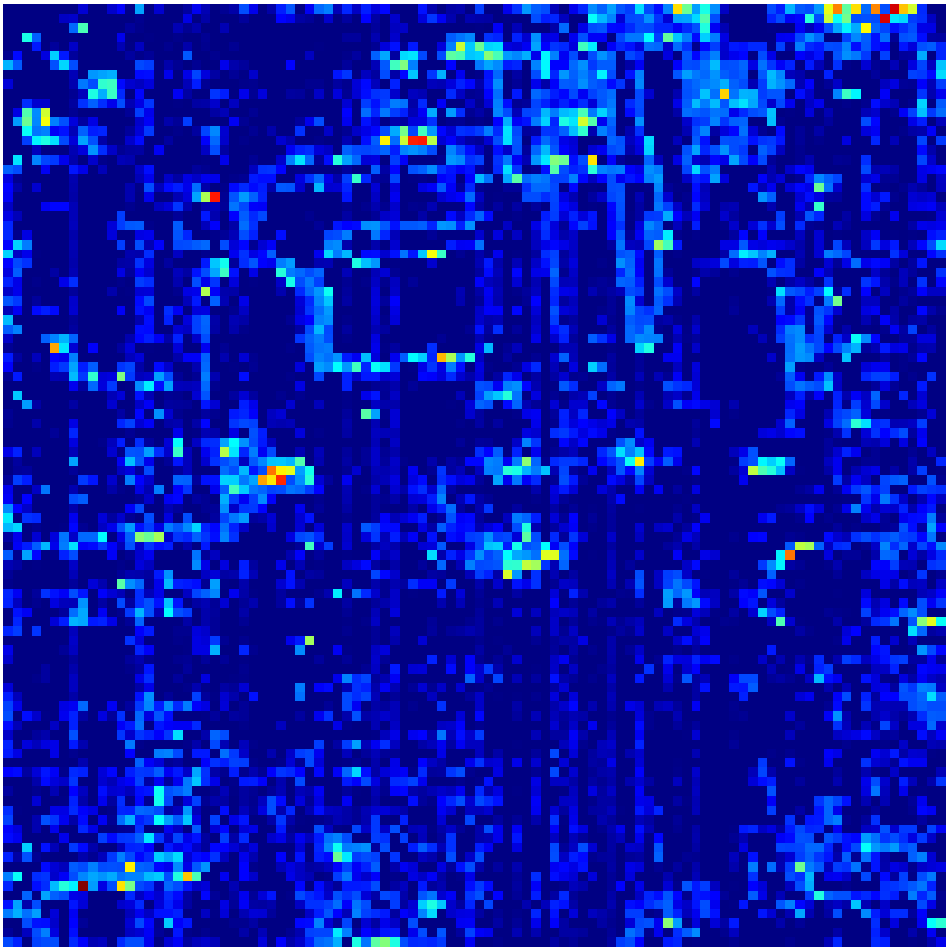}
			\includegraphics[width=\linewidth]{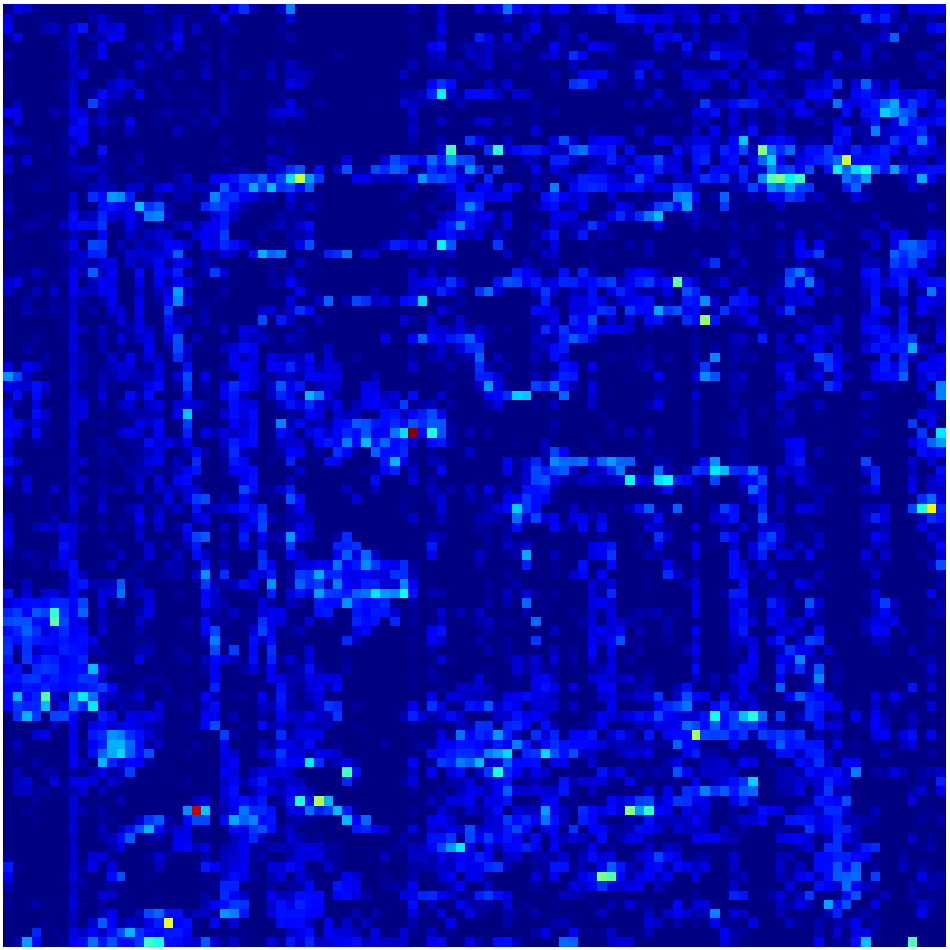}
			\includegraphics[width=\linewidth]{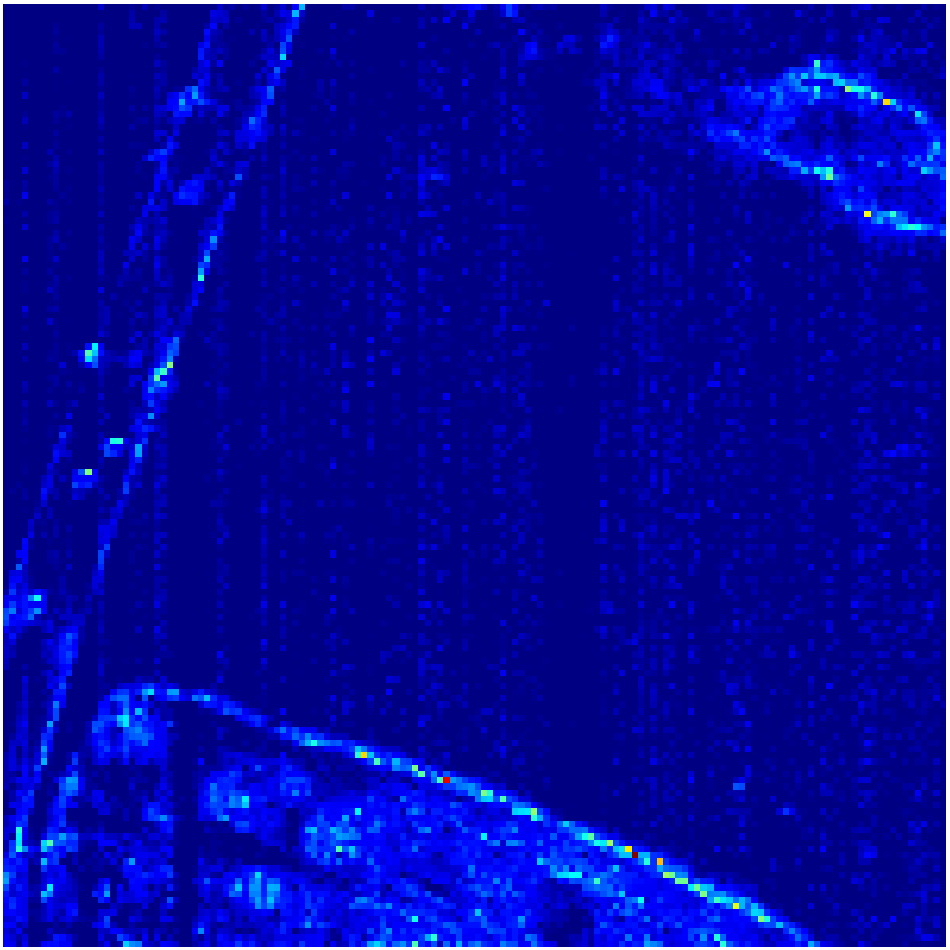}
			\includegraphics[width=\linewidth]{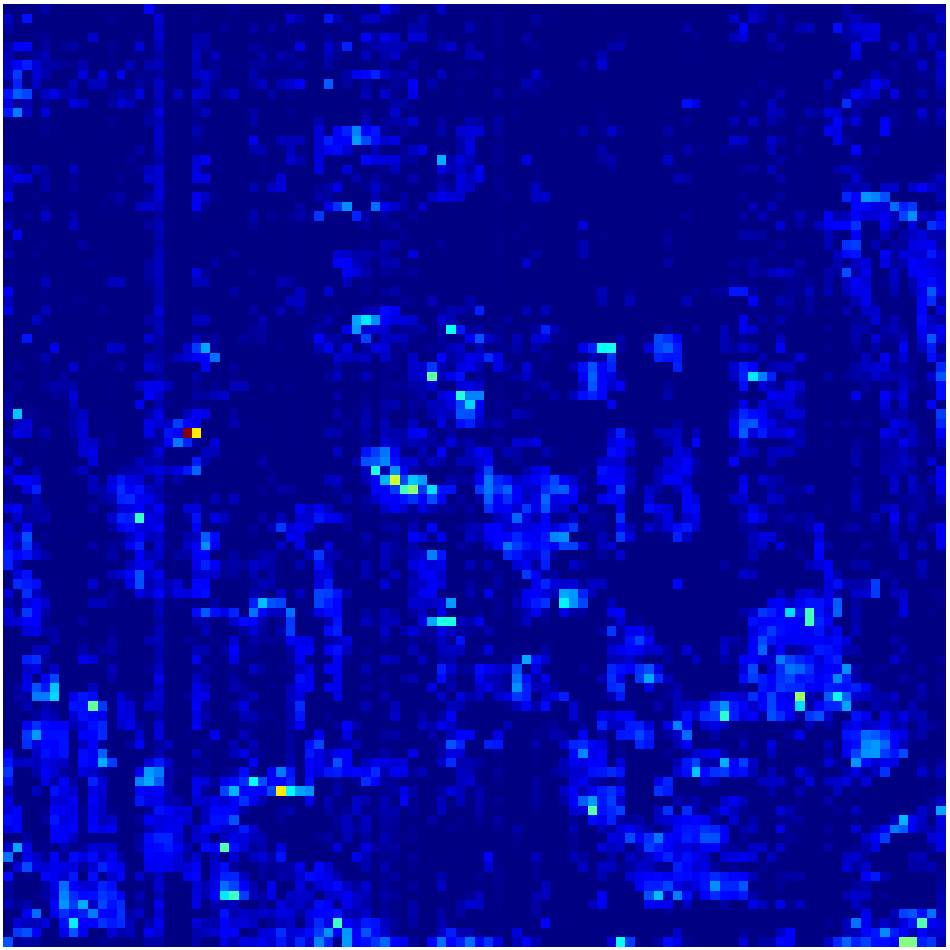}
			\includegraphics[width=\linewidth]{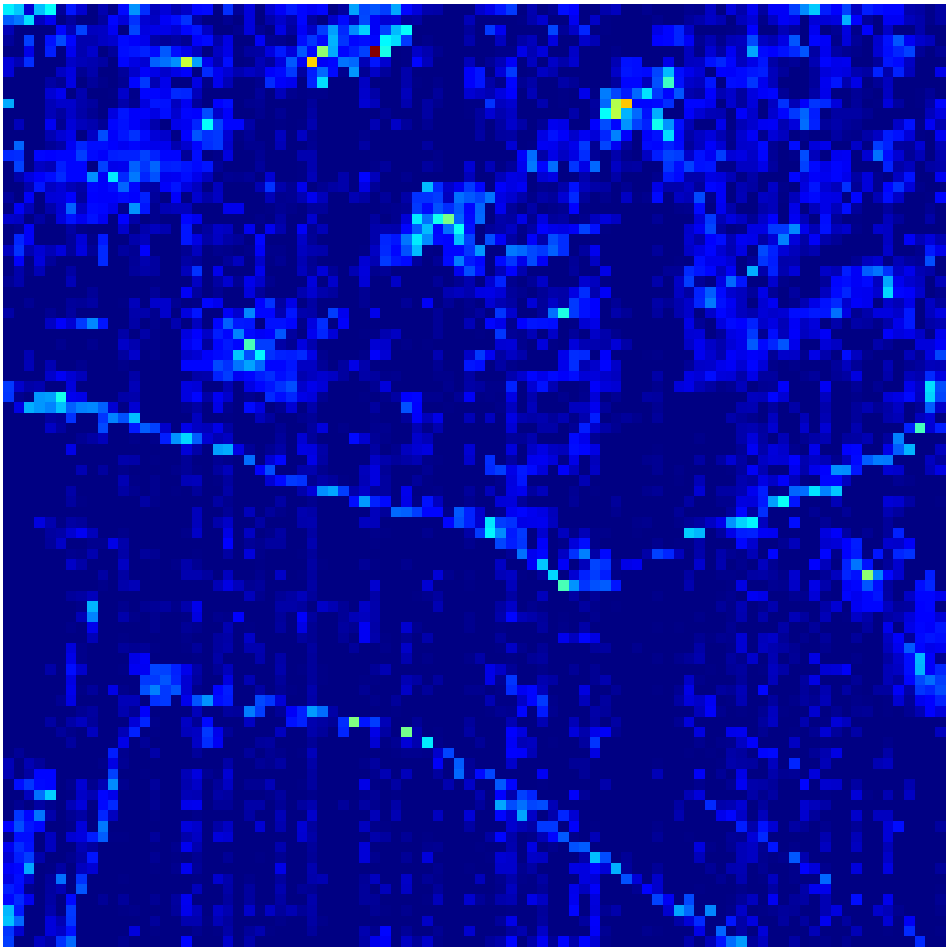}
			\caption{}
		\end{subfigure}\hfill
		\begin{subfigure}[b]{0.09\linewidth}
			\includegraphics[width=\linewidth]{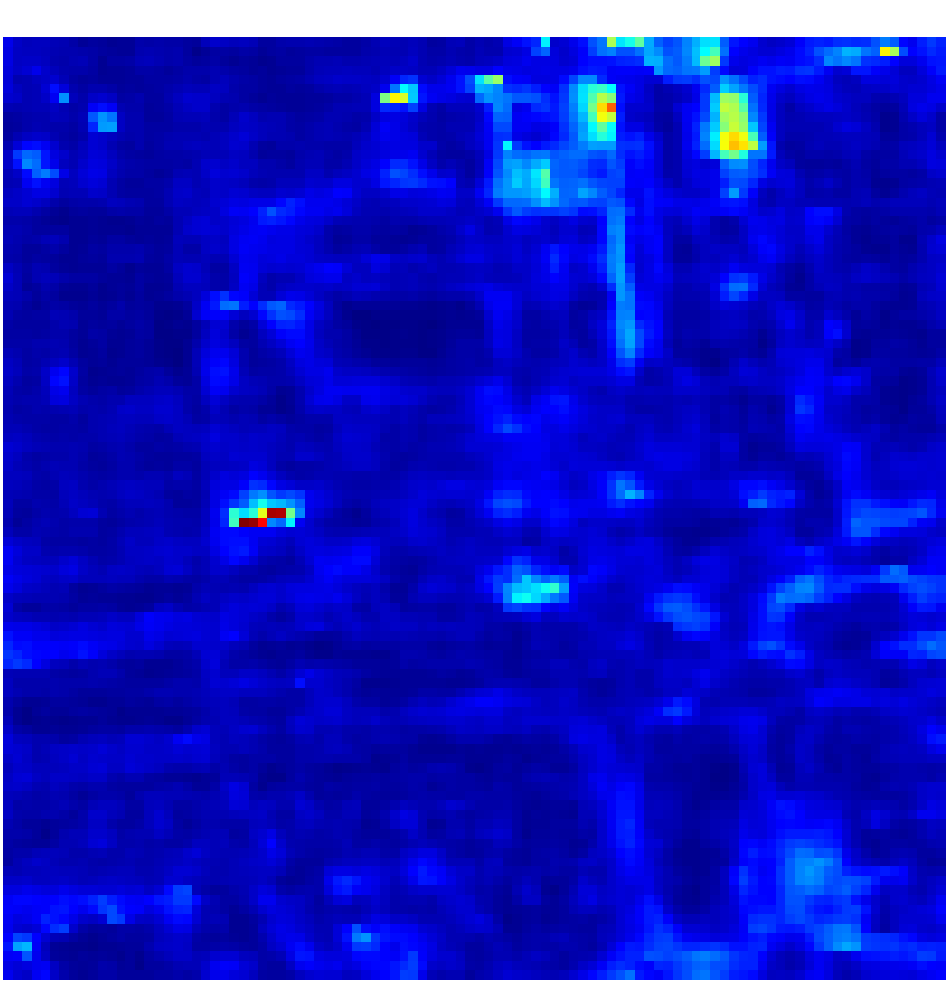}
			\includegraphics[width=\linewidth]{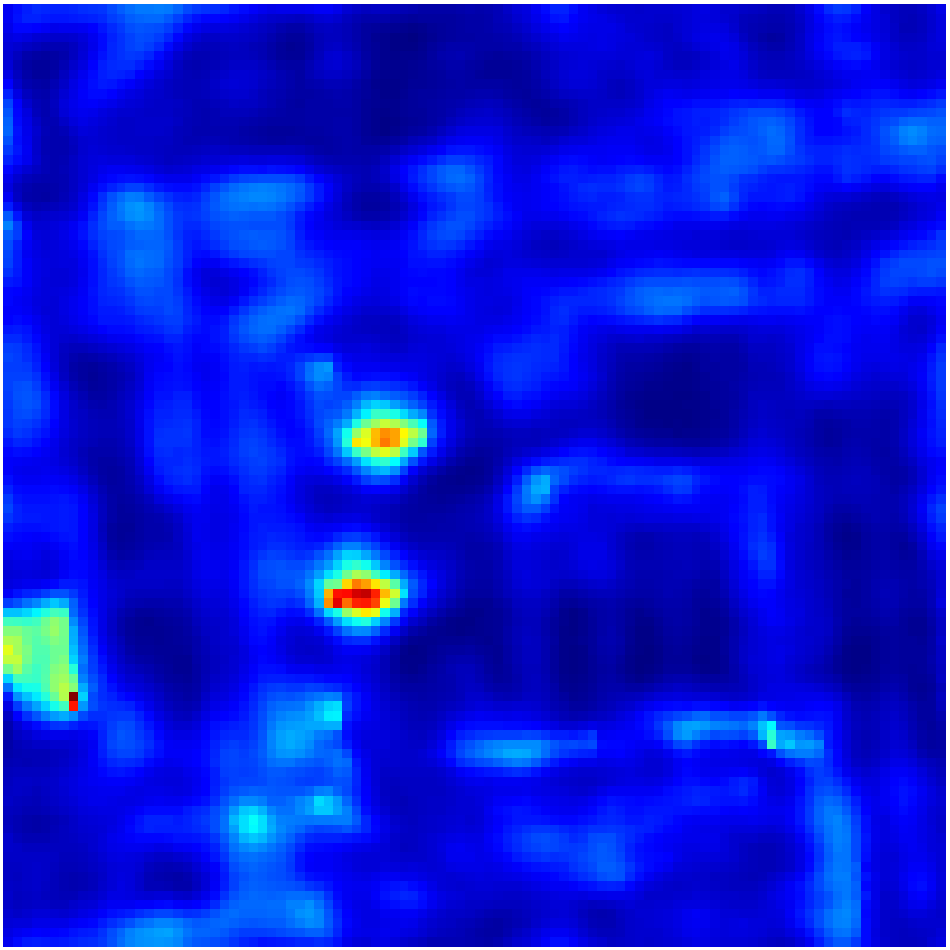}
			\includegraphics[width=\linewidth]{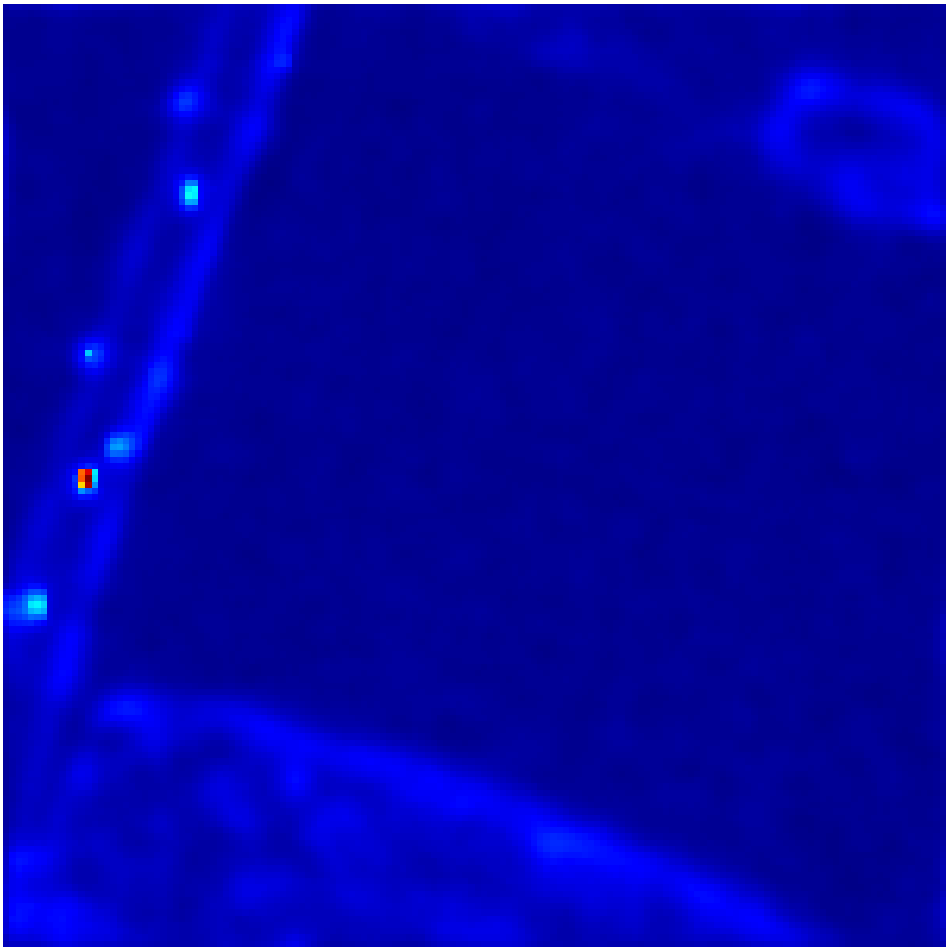}
			\includegraphics[width=\linewidth]{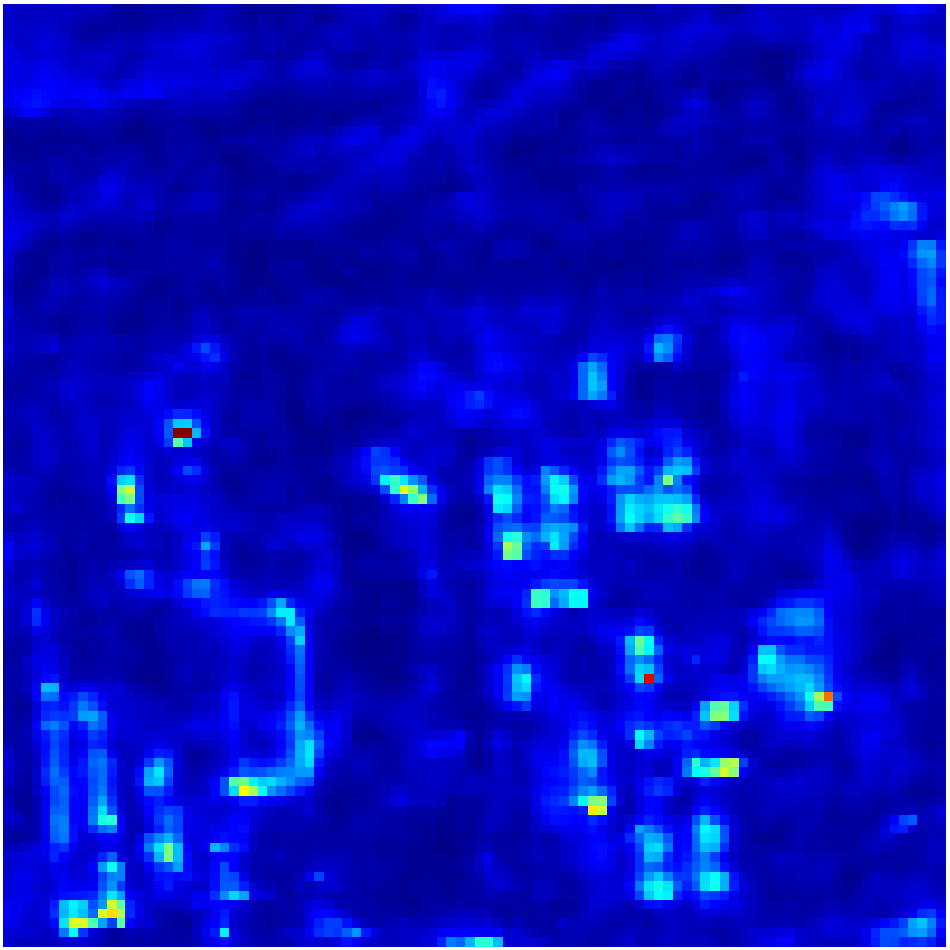}
			\includegraphics[width=\linewidth]{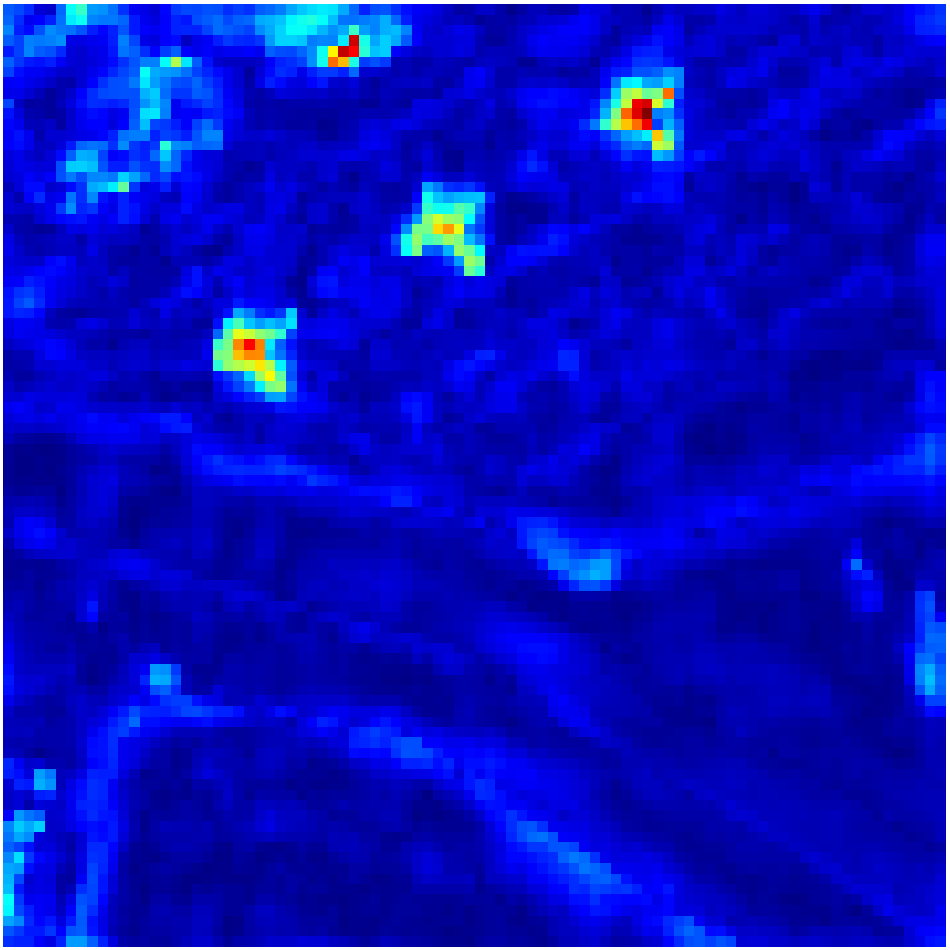}
			\caption{}
		\end{subfigure}
	\end{subfigure}
	\caption{Detection results of different methods on five datasets: (a) RX; (b) RPCA; (c) LRASR; (d) Turbo-GoDec; (e) PTA; (f) TPCA; (g) PCA-TLRSR; (h) RGAE; (i) GAED; (j) LCRS; (k) GSAA-SS.}
	\label{fig:2D}
\end{figure*}

Fig. \ref{fig:ROC} reports the ROC curves on the five datasets. On Airport2, GSAA-SS achieves a higher detection probability than the competing methods across almost the entire false-alarm range. On Beach and San Diego, its superiority is particularly pronounced when $P_F>10^{-3}$; this advantage is also evident on Urban when $P_F>10^{-2}$ and on Airport1 when $P_F>10^{-1}$. These ROC results show that the proposed detector maintains favorable detection capability under different false-alarm constraints.

\begin{figure*}[htbp]
	\centering
	\begin{subfigure}[b]{1\linewidth}
		\begin{subfigure}[b]{1\linewidth}
			\centering
			\includegraphics[width=0.95\linewidth]{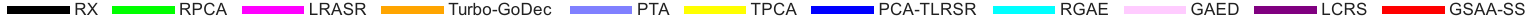}
		\end{subfigure} 
		\begin{subfigure}[b]{0.195\linewidth}
			\centering
			\includegraphics[width=\linewidth]{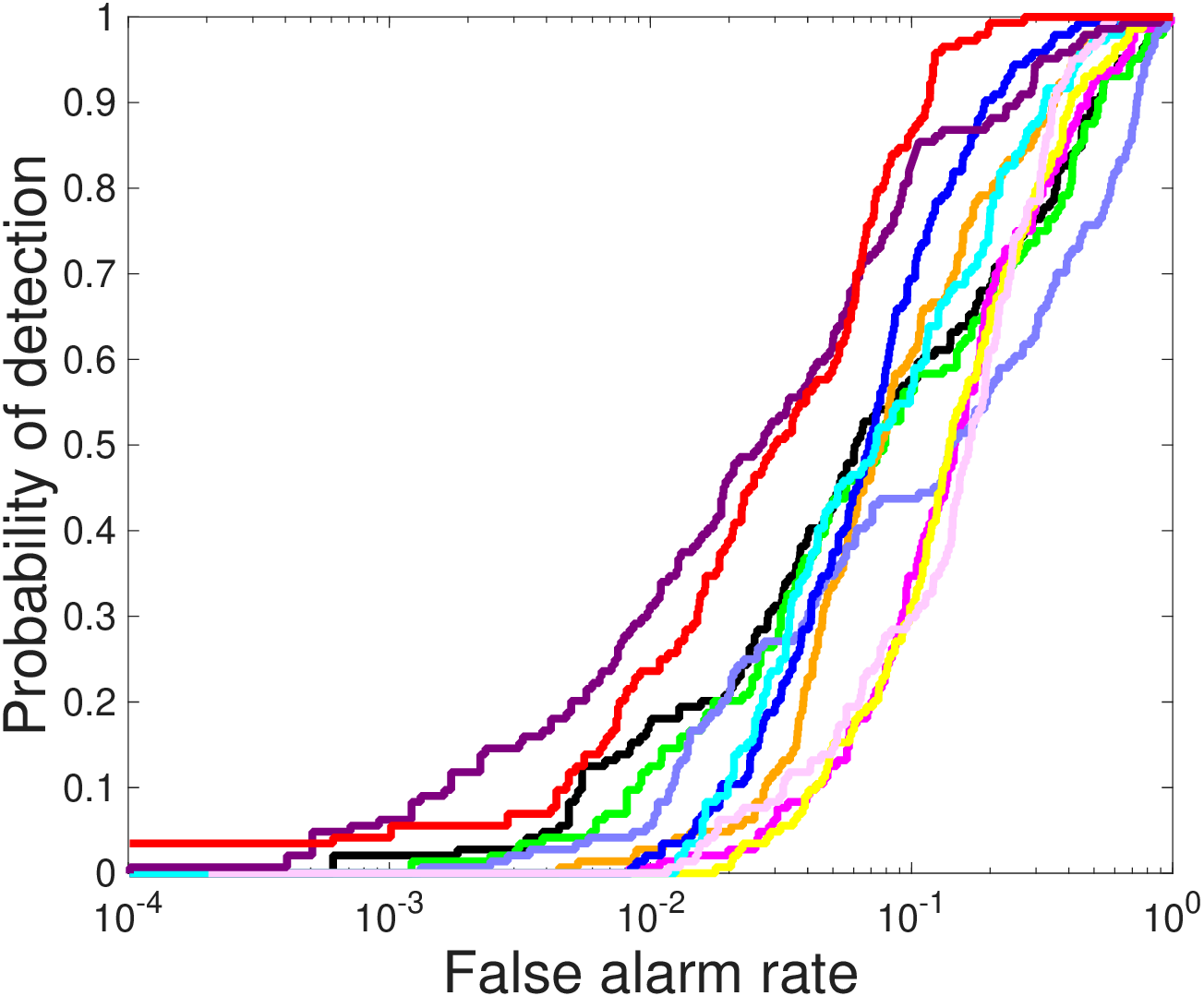}
			\caption{Airport1}
		\end{subfigure}   	
		\begin{subfigure}[b]{0.195\linewidth}
			\centering
			\includegraphics[width=\linewidth]{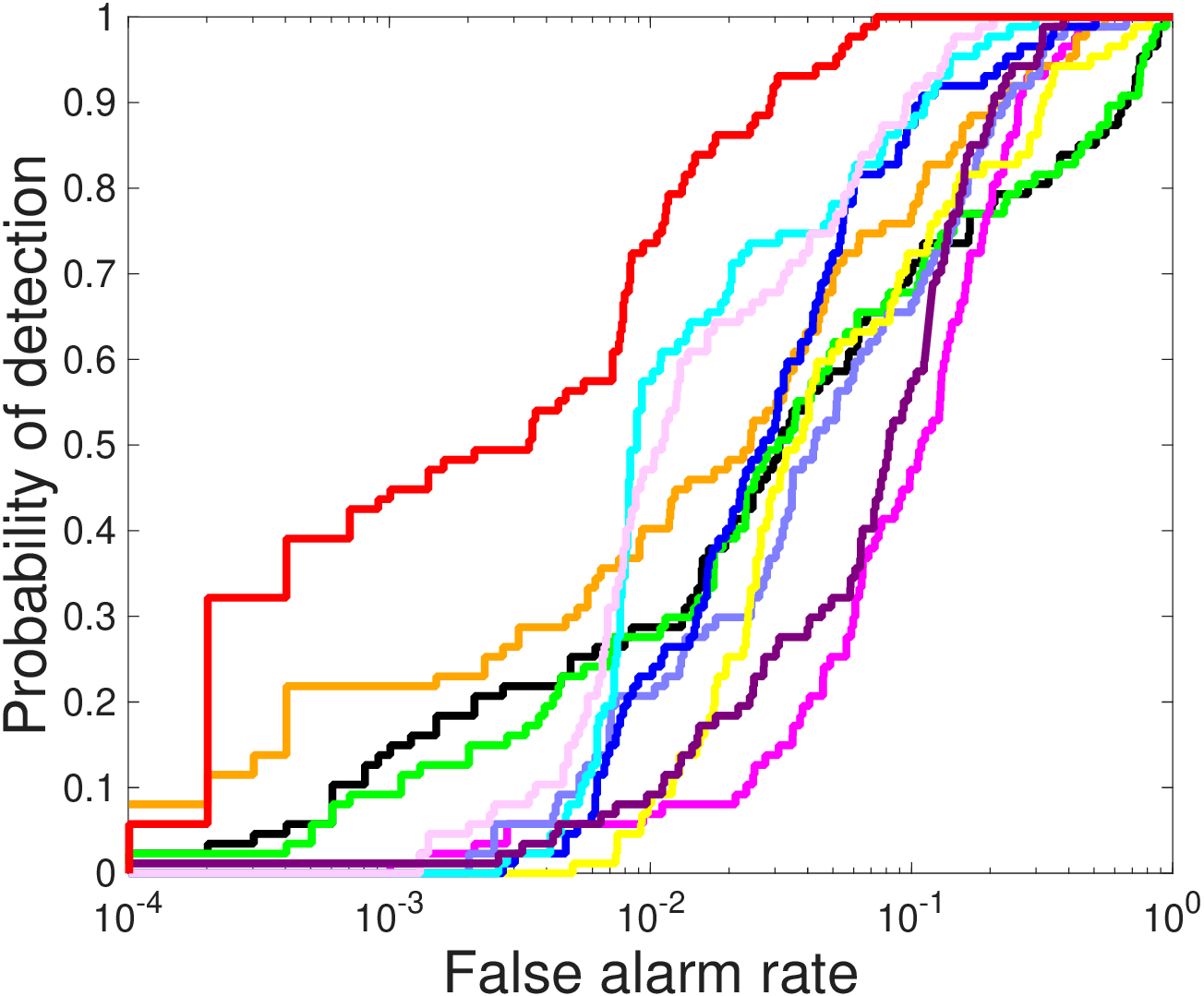}
			\caption{Airport2}
		\end{subfigure}
		\begin{subfigure}[b]{0.195\linewidth}
			\centering
			\includegraphics[width=\linewidth]{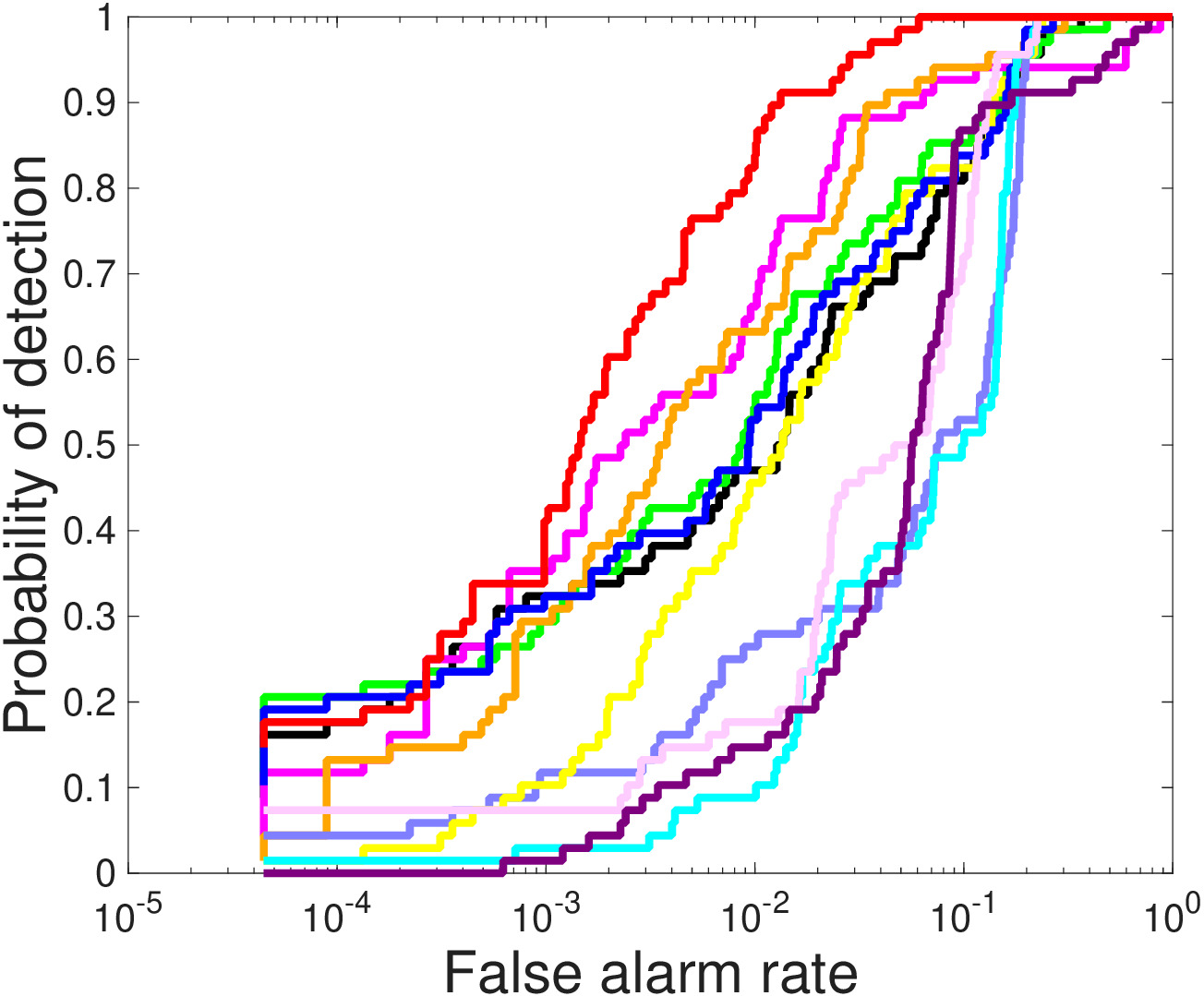}
			\caption{Beach}
		\end{subfigure}
		\begin{subfigure}[b]{0.195\linewidth}
			\centering
			\includegraphics[width=\linewidth]{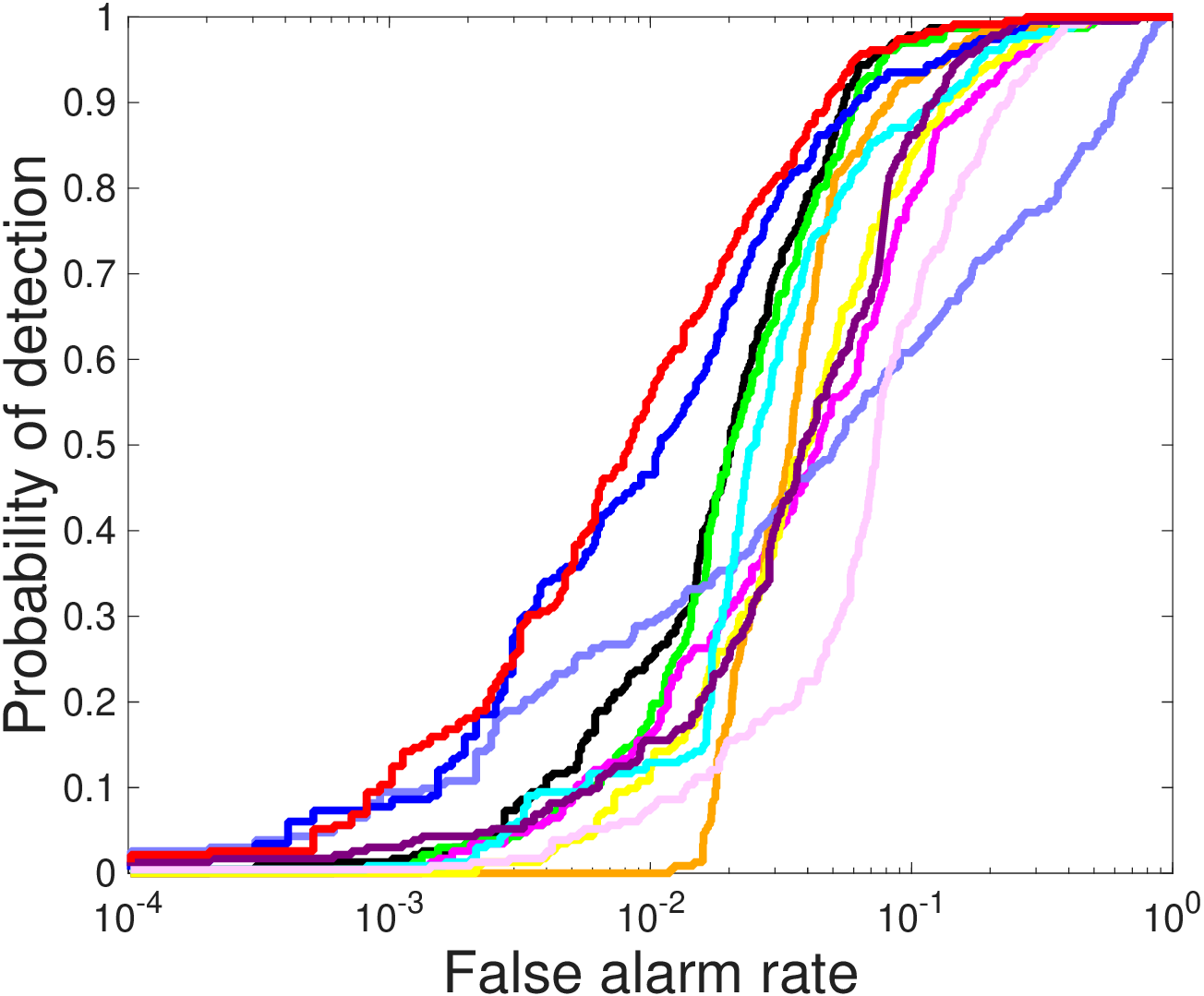}
			\caption{Urban}
		\end{subfigure}
		\begin{subfigure}[b]{0.195\linewidth}
			\centering
			\includegraphics[width=\linewidth]{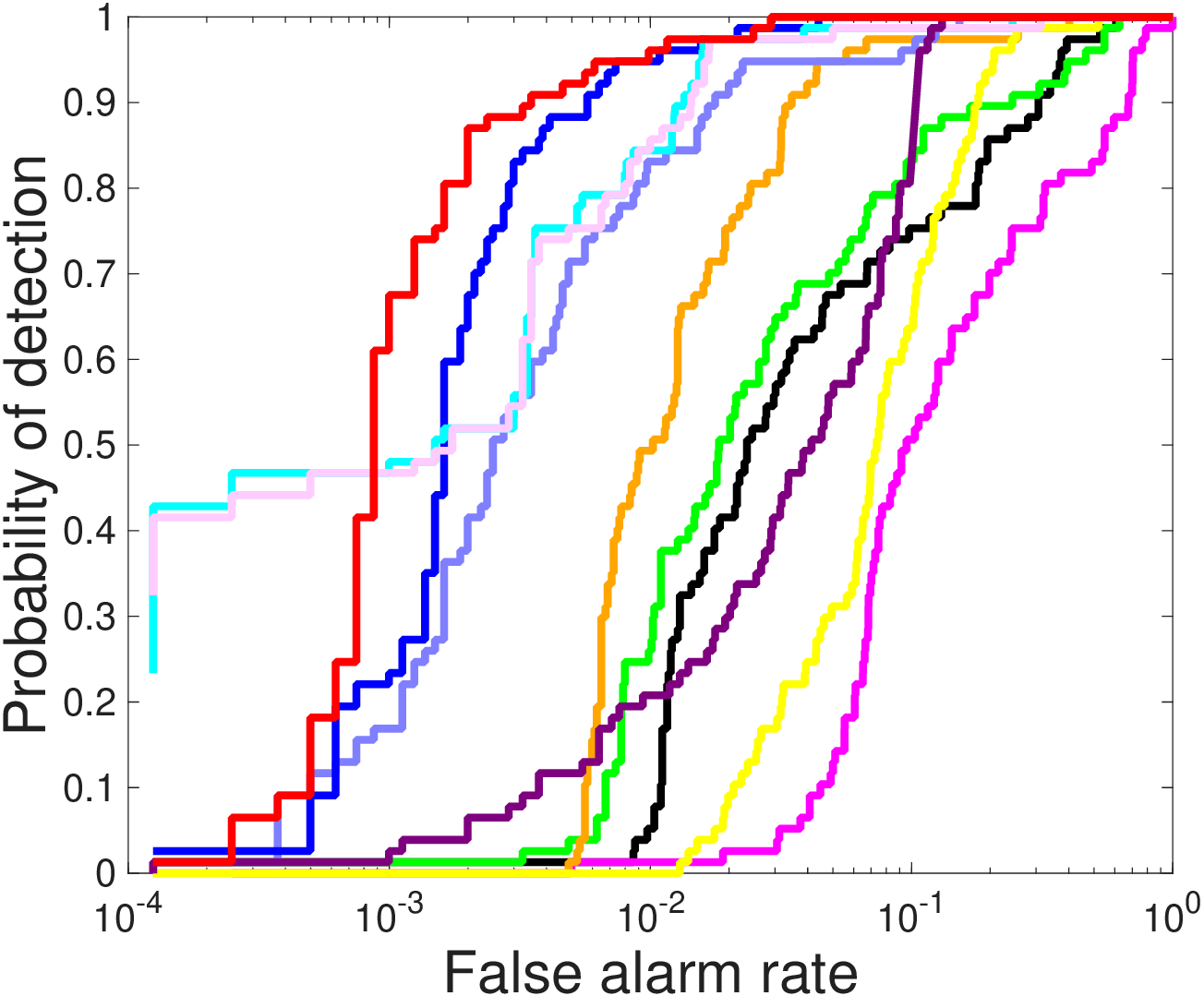}
			\caption{San Diego}
		\end{subfigure}
	\end{subfigure}
	\vfill
	\caption{ROC curves obtained by different methods.}
	\label{fig:ROC}
\end{figure*}

Table \ref{tab:AUC} lists the AUC values and running time (in seconds) of all methods. GSAA-SS obtains the highest AUC on all five datasets. Compared with the second-best result on each dataset, the improvements are 3.17\%, 2.41\%, 1.71\%, 1.02\%, and 0.09\%, respectively. The corresponding second-best methods are LCRS, GAED, Turbo-GoDec, PCA-TLRSR, and PCA-TLRSR. These results indicate that
GSAA-SS can obtain  robust performance across scenes with varying background complexity and anomaly distributions.
The runtime results in Table \ref{tab:AUC} further show the computational advantage of the proposed factorized model. Although RX is the fastest method, its detection accuracy is much lower than that of GSAA-SS. Among the remaining methods, GSAA-SS achieves the lowest runtime on average across the tested datasets, while maintaining highly competitive efficiency on each individual dataset. It is about twice as fast as PCA-TLRSR and is at least twenty-five times faster than RGAE and GAED on average. This efficiency is consistent with the complexity analysis in Section \ref{Sec:Alg}, because the proposed model avoids repeated large-scale SVDs and adaptively reduces the factor dimension during optimization.

\begin{table*}[htbp]
	\centering
	\caption{Comparison of AUC values and running time(s) of different methods.}
	\label{tab:AUC}
	\resizebox{\linewidth}{!}{
	\begin{tabular}{cc*{11}{c}} 
		\toprule
		\multirow{2.5}*{HSI} & \multirow{2.5}*{Index} & 
		\multicolumn{11}{c}{Method} \\
		\cmidrule(lr){3-13}
		& & RX & RPCA & LRASR & Turbo-GoDec & PTA & TPCA & PCA-TLRSR & RGAE & GAED & LCRS & GSAA-SS \\
		\midrule
		\multirow{2}{*}{Airport1} 
		& AUC & 0.8221 & 0.8088 & 0.7942 & 0.8672 & 0.7331 & 0.8023 & 0.9057 & 0.8684 & 0.8133 & 0.9237 & 0.9530 \\
		& Time (s) & 0.06 & 4.23 & 18.45 & 15.36 & 16.04 & 19.51 & 2.56 & 52.79 & 42.80 & 1.91 & 1.94 \\
		\midrule
		\multirow{2}{*}{Airport2}
		& AUC & 0.8404 & 0.8428 & 0.8689 & 0.9299 & 0.9096 & 0.8891 & 0.9458 & 0.9645 & 0.9673 & 0.8998 & 0.9906 \\
		& Time (s) & 0.05 & 4.50 & 19.75 & 14.98 & 16.32 & 19.58 & 2.56 & 52.02 & 41.58 & 1.94 & 1.70 \\
		\midrule
		\multirow{2}{*}{Beach}
		& AUC & 0.9538 & 0.9603 & 0.9504 & 0.9776 & 0.9061 & 0.9583 & 0.9599 & 0.9062 & 0.9368 & 0.9045 & 0.9943 \\
		& Time (s) & 0.04 & 2.05 & 43.72 & 26.43 & 17.53 & 22.33 & 6.66 & 219.39 & 73.75 & 3.11 & 2.50 \\
		\midrule
		\multirow{2}{*}{Urban}
		& AUC & 0.9692 & 0.9658 & 0.9293 & 0.9536 & 0.8258 & 0.9370 & 0.9711 & 0.9510 & 0.8986 & 0.9429 & 0.9810 \\
		& Time (s) & 0.06 & 3.96 & 19.91 & 15.51 & 15.64 & 20.08 & 2.56 & 57.09 & 42.38 & 2.02 & 1.67 \\
		\midrule
		\multirow{2}{*}{San Diego}
		& AUC & 0.9106 & 0.9264 & 0.7847 & 0.9768 & 0.9895 & 0.9075 & 0.9970 & 0.9930 & 0.9918 & 0.9500 & 0.9979 \\
		& Time (s) & 0.38 & 5.53 & 18.22 & 12.49 & 14.68 & 20.95 & 2.92 & 100.88 & 40.98 & 1.66 & 1.41 \\
		\bottomrule
	\end{tabular}}
\end{table*}

To further examine separability, Fig. \ref{fig:AB} presents normalized background--anomaly separation maps. A larger gap between the anomaly and background distributions indicates better separability. On Airport1, Airport2, and Beach, GSAA-SS produces a clearer separation gap than the competing methods. On Urban, RX, RPCA, PCA-TLRSR, and GSAA-SS show comparable separation, while the other methods exhibit stronger overlap between anomaly and background responses. On San Diego, PCA-TLRSR and GSAA-SS achieve similar separability, but PCA-TLRSR is less stable on Airport1, Airport2, and Beach. Overall, GSAA-SS provides consistently strong separability across all tested scenes.

\begin{figure*}[htbp]
	\centering
	\begin{minipage}{0.9\textwidth}
		\begin{subfigure}[b]{0.195\linewidth}
			\centering
			\includegraphics[width=\linewidth]{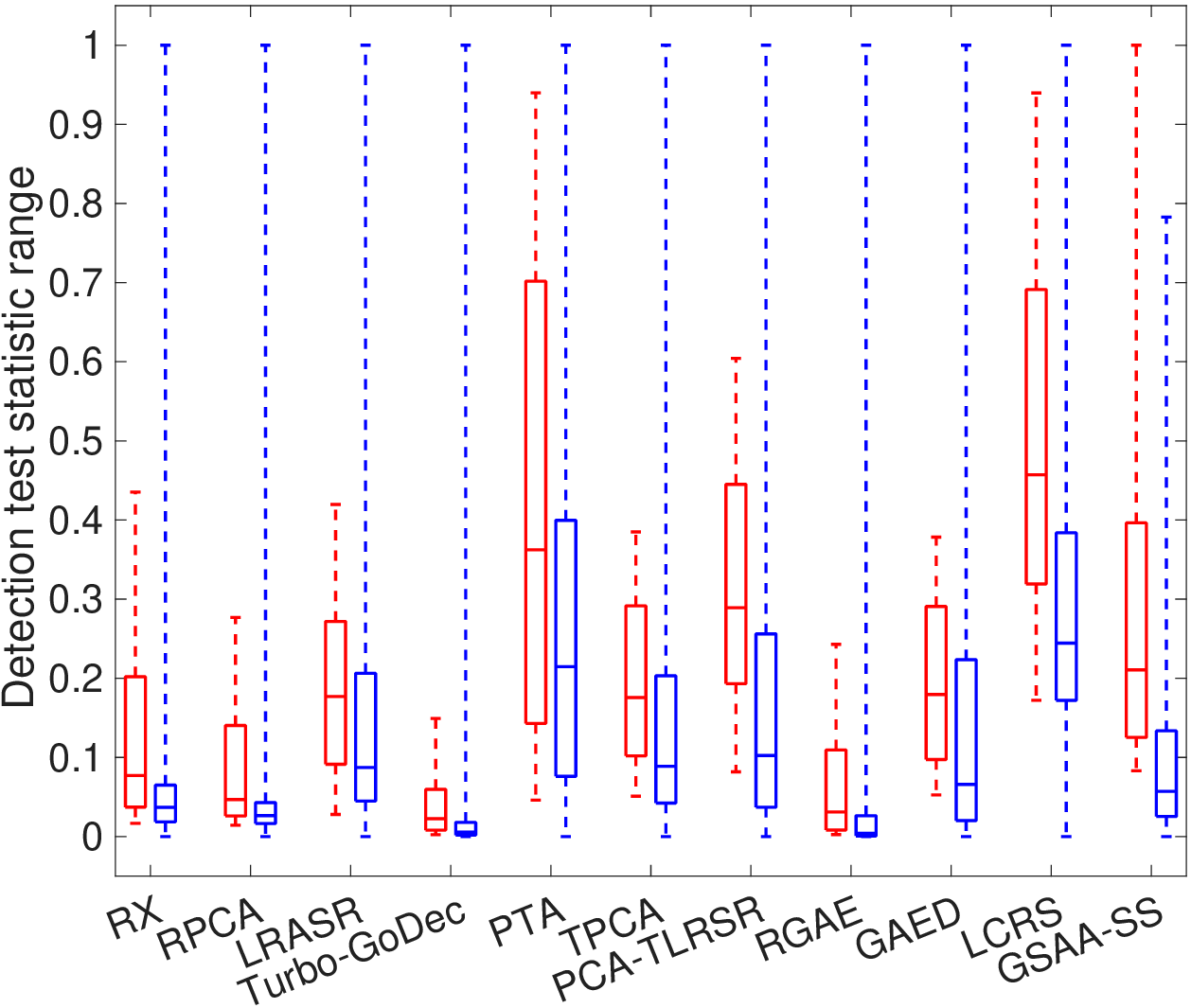}
			\caption{Airport1}
		\end{subfigure}   	
		\begin{subfigure}[b]{0.195\linewidth}
			\centering
			\includegraphics[width=\linewidth]{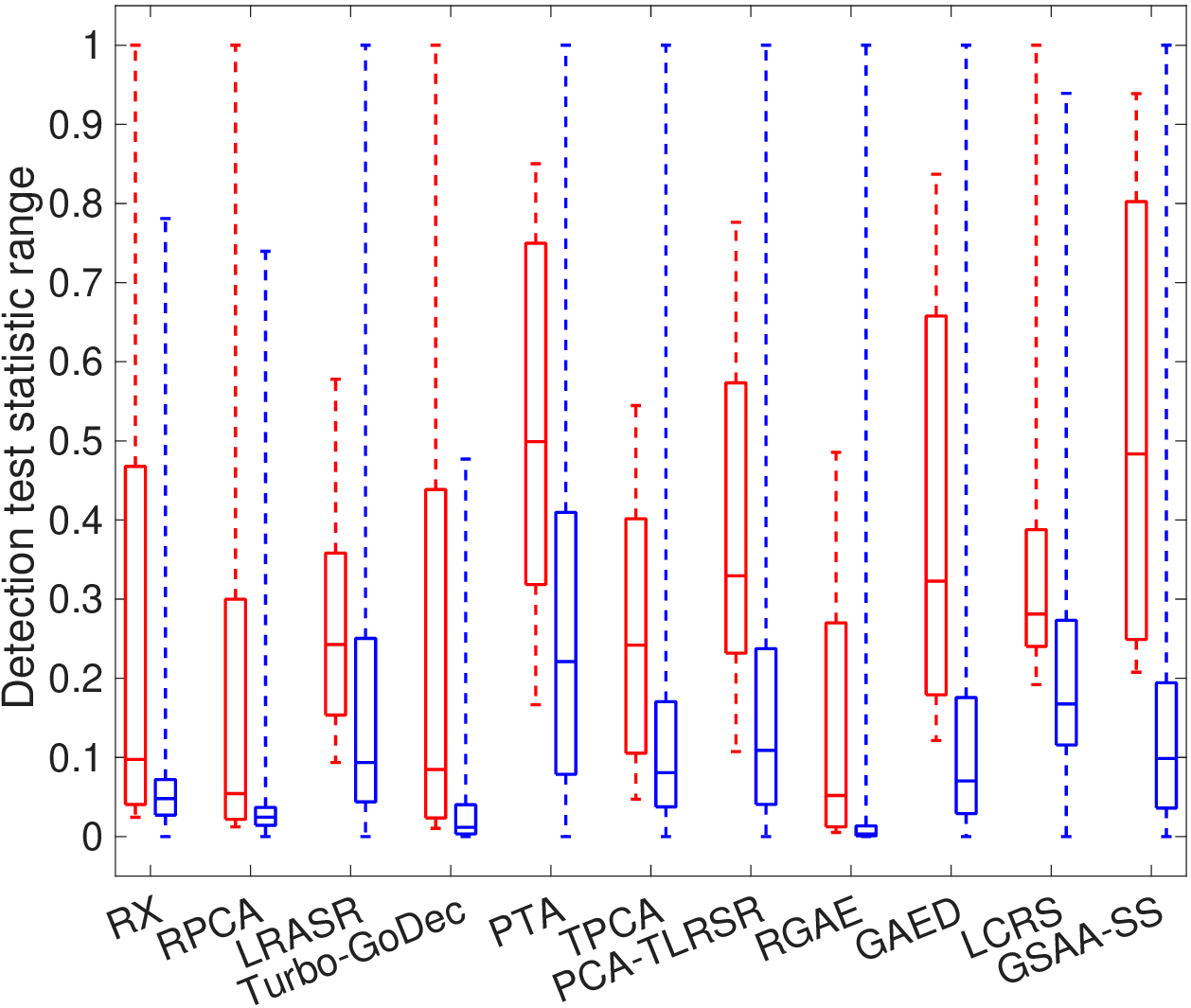}
			\caption{Airport2}
		\end{subfigure}
		\begin{subfigure}[b]{0.195\linewidth}
			\centering
			\includegraphics[width=\linewidth]{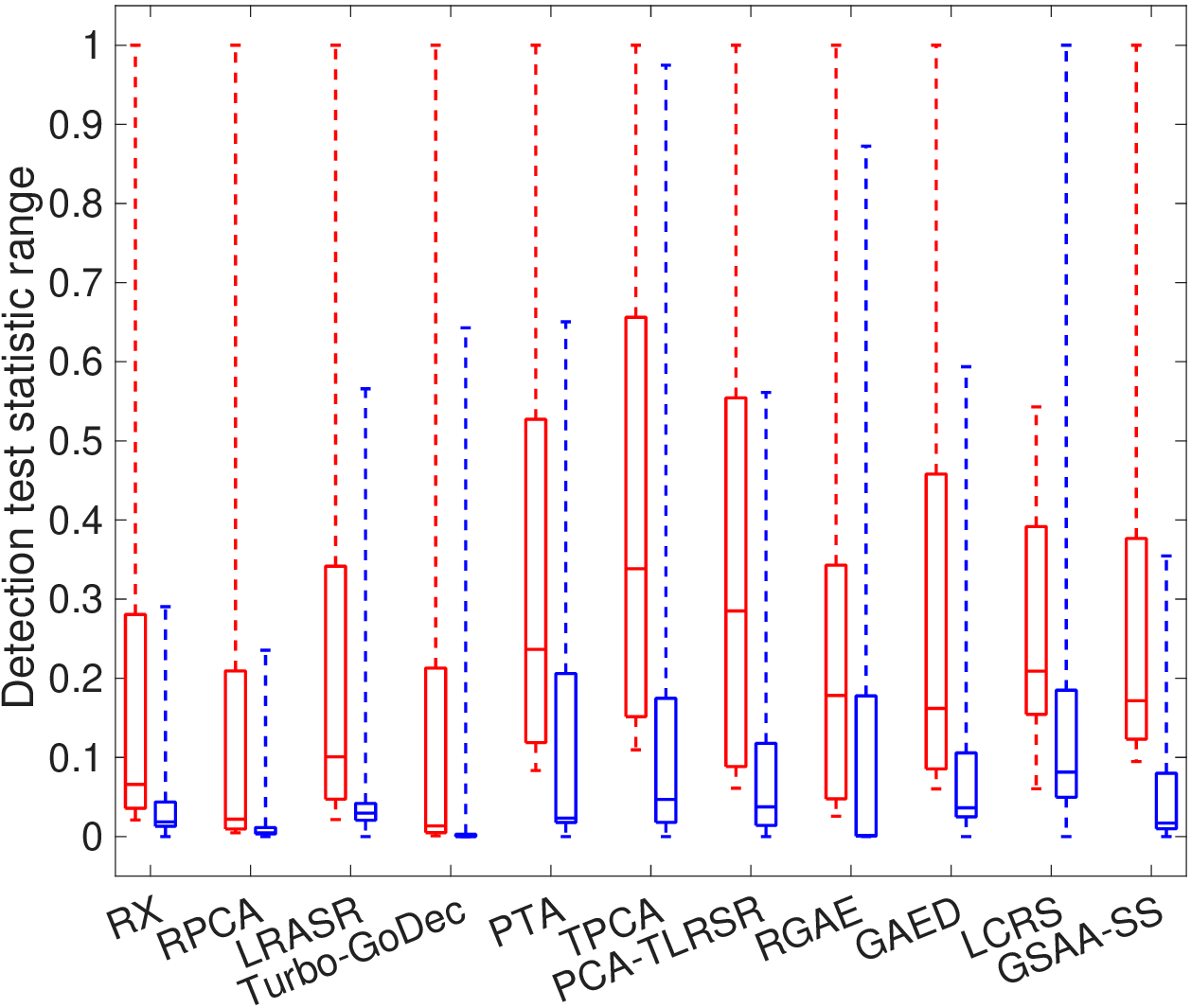}
			\caption{Beach}
		\end{subfigure}
		\begin{subfigure}[b]{0.195\linewidth}
			\centering
			\includegraphics[width=\linewidth]{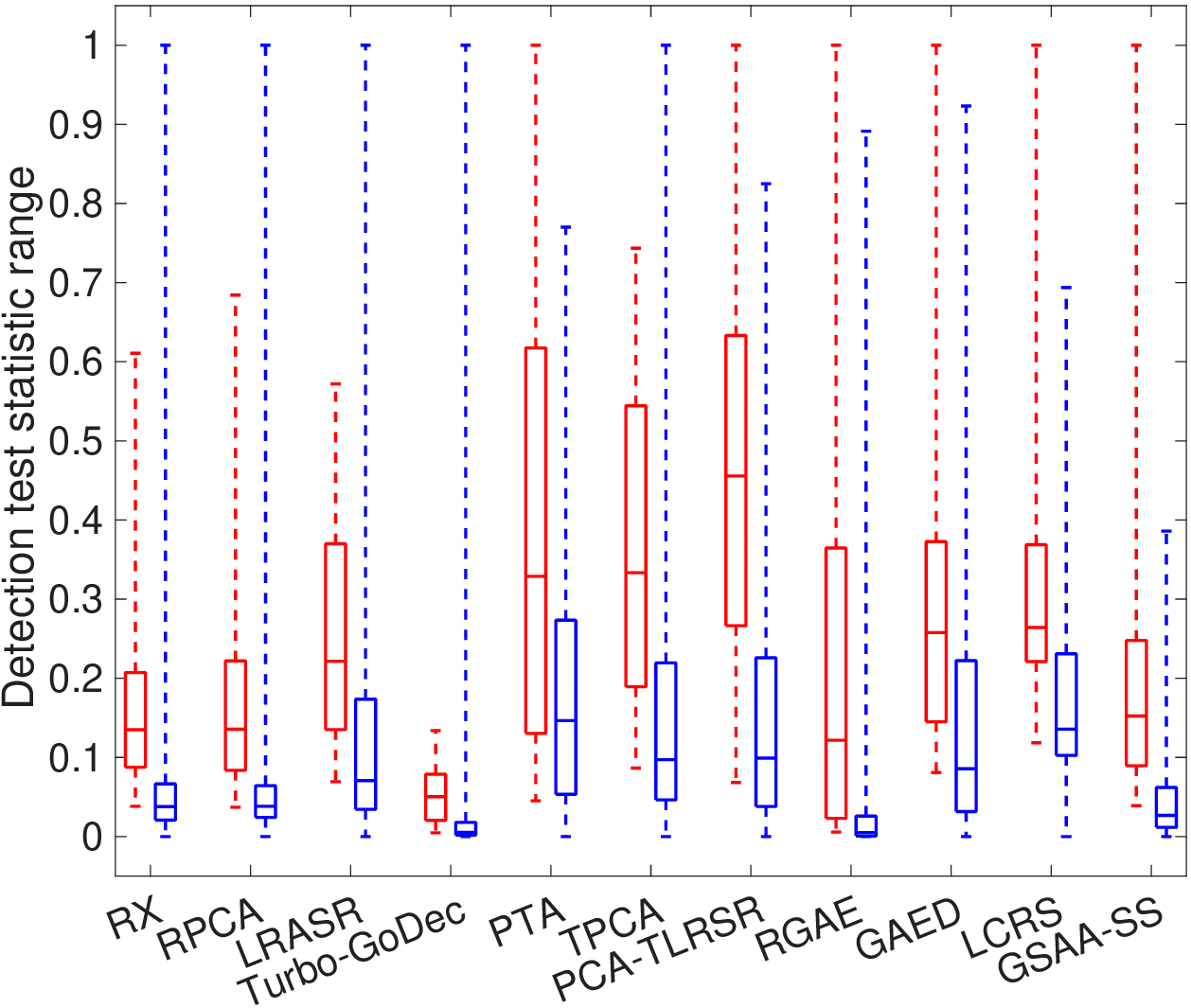}
			\caption{Urban}
		\end{subfigure}
		\begin{subfigure}[b]{0.195\linewidth}
			\centering
			\includegraphics[width=\linewidth]{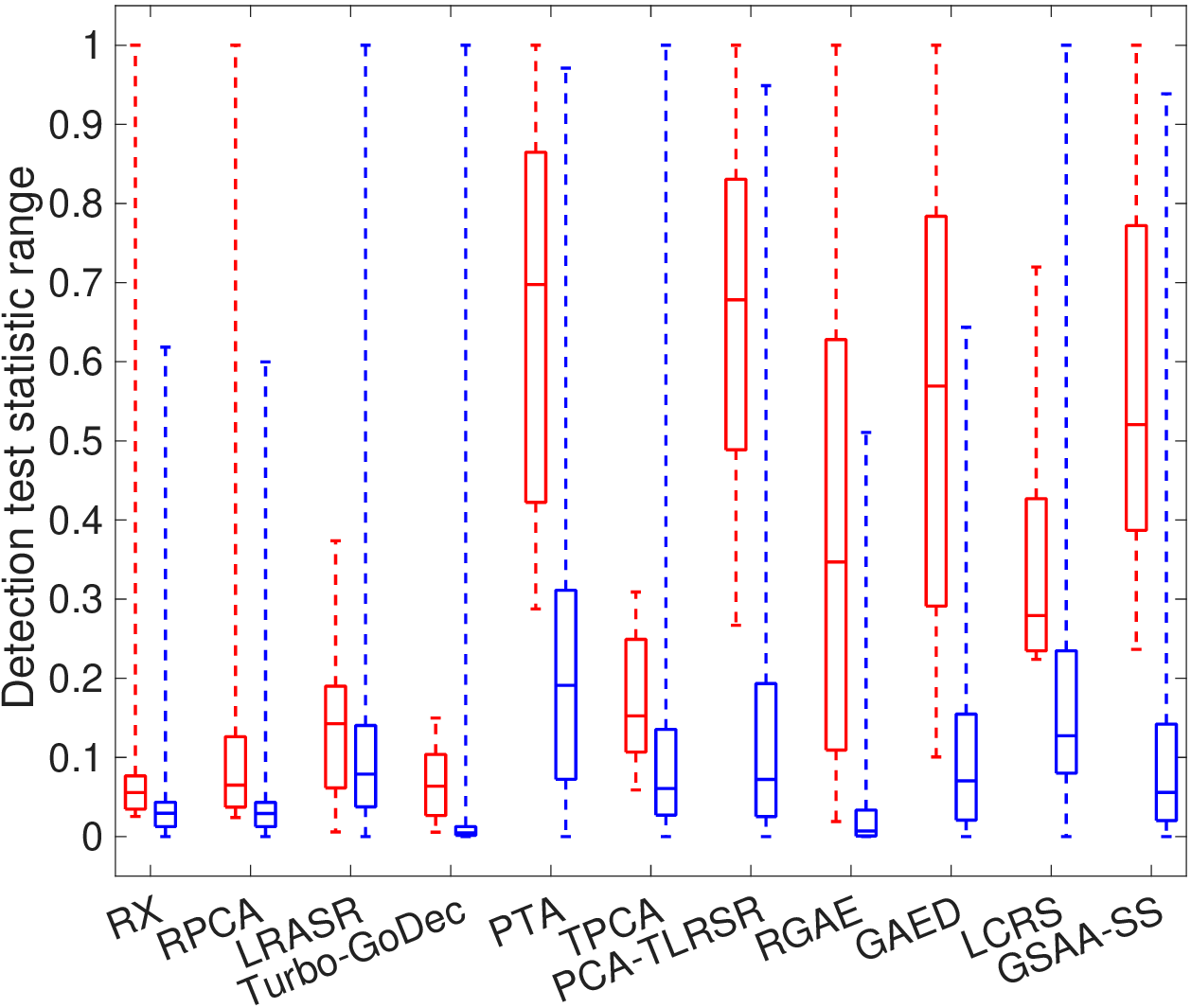}
			\caption{San Diego}
		\end{subfigure}
	\end{minipage}\hfill
	\begin{minipage}{0.1\textwidth}
		\centering
		\includegraphics[width=\linewidth]{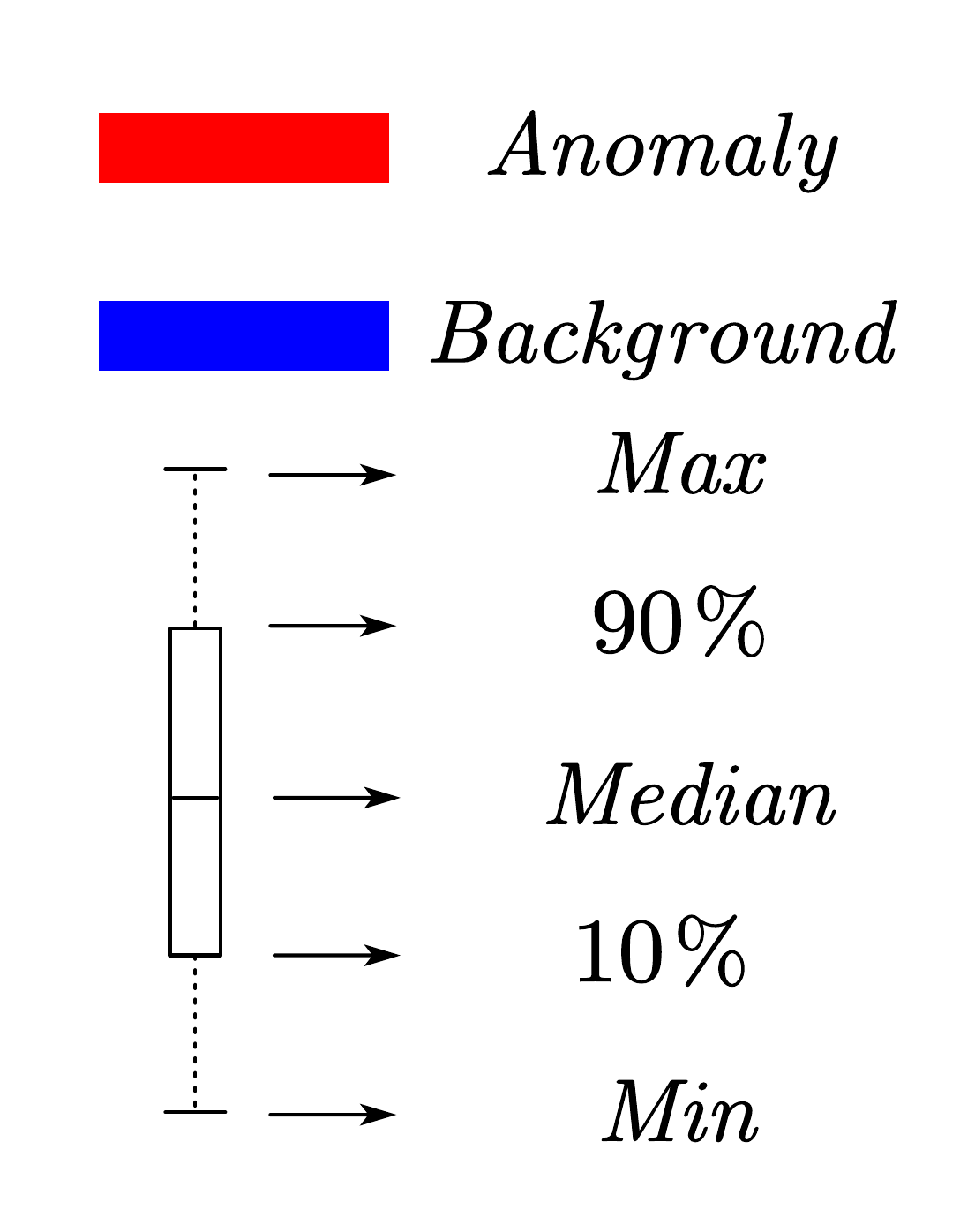}
	\end{minipage}
	\vfill
	\caption{Separability maps of different methods.}
	\label{fig:AB}
\end{figure*}

\subsection{Discussions}
This subsection begins with an analysis of how different parameters affect performance, followed by a discussion of the model’s novel contributions.

\subsubsection{Parameter analysis}\label{sec:pa}
The proposed model contains five model parameters, namely $\gamma$, $\lambda$, $p$, $\alpha_1$, and $\alpha_2$, and four algorithm parameters, namely $\beta_u^0$ and $\rho_u$ for $u\in[2]$. In all experiments, the initial penalty parameters are set as $\beta_1^0=5\times10^{-2}$ and $\beta_2^0=10^{-2}$. The parameter $\rho_1$ is selected from $\{1.3,1.4,1.5\}$, while $\rho_2$ is fixed at 1.2. Following the convergence requirement, $\lambda$ is updated as $\lambda=\min\{10\beta_1,10^{10}\}$.

We first evaluate the sensitivity to $\gamma$ and $p$. The parameter $\gamma$ is selected from $\{10^{-2},5\times10^{-2},10^{-1},5\times10^{-1},1\}$, and $p$ varies from 0.2 to 1 with an interval of 0.1. Fig. \ref{fig:gamma-p} shows the resulting AUC surfaces. The proposed method is stable when $\gamma \in \{5\times10^{-2},10^{-1}\}$ and $p\in[0.4,0.7]$. Therefore, we set $\gamma=5\times10^{-2}$ and $p=0.6$ in the following experiments.

\begin{figure*}[htbp]
	\centering
	\begin{subfigure}[b]{1\linewidth}
		\begin{minipage}{0.195\textwidth}
			\centering
			\includegraphics[width=1\textwidth]{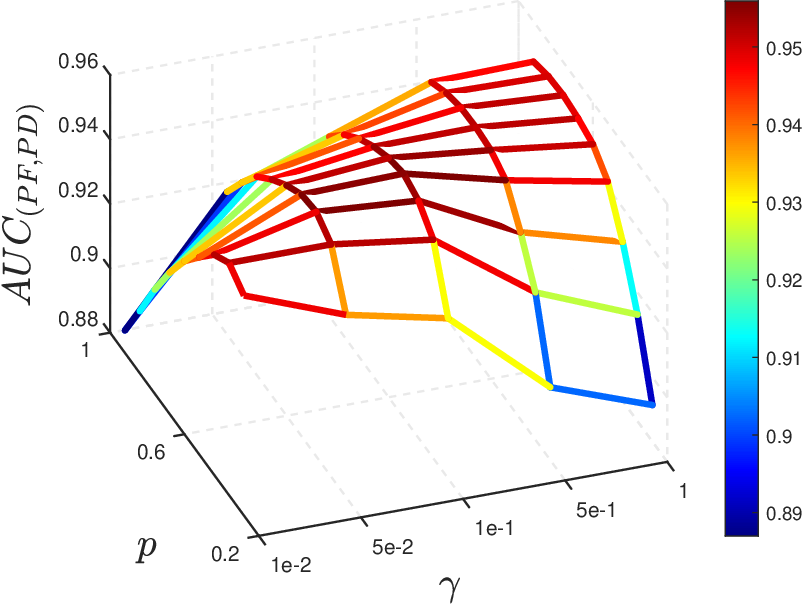}
			\caption{Airport1}
		\end{minipage}
		\begin{minipage}{0.195\textwidth}
			\centering
			\includegraphics[width=1\textwidth]{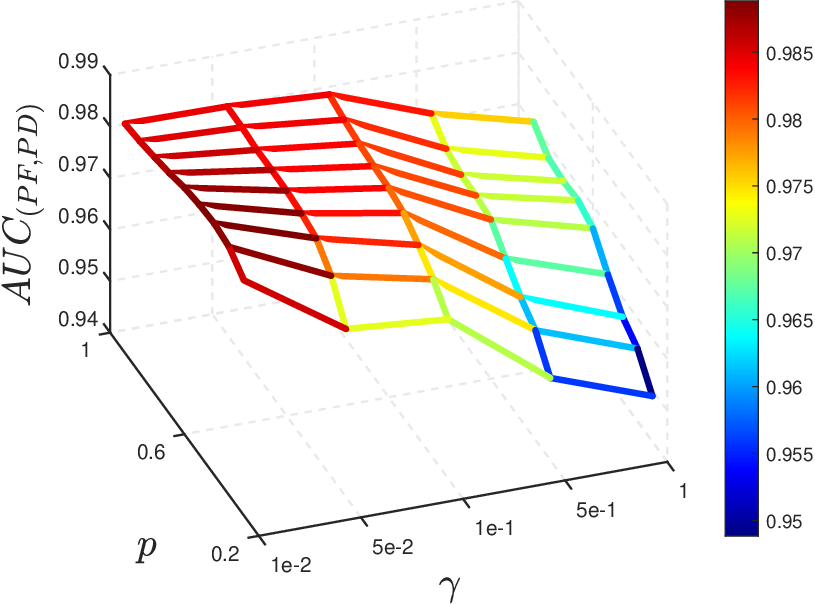}
			\caption{Airport2}
		\end{minipage}
		\begin{minipage}{0.195\textwidth}
			\centering
			\includegraphics[width=1\textwidth]{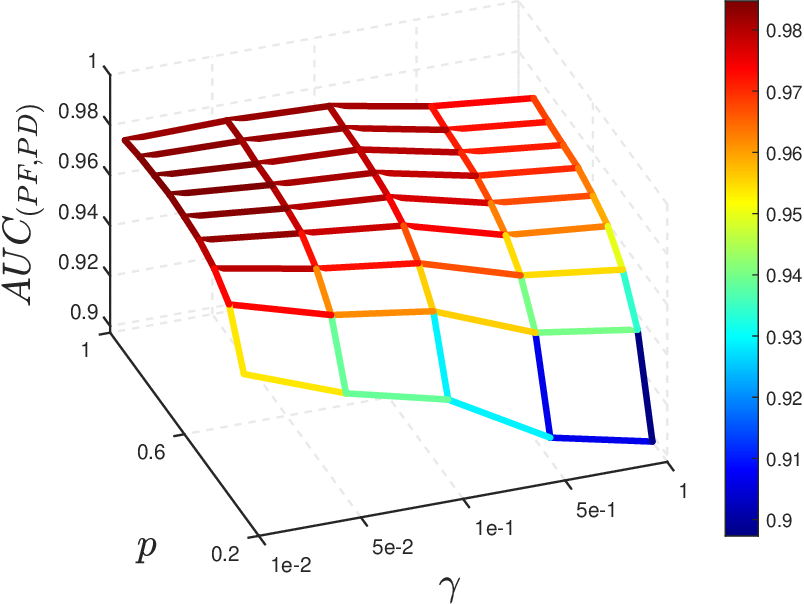}
			\caption{Beach}
		\end{minipage}
		\begin{minipage}{0.195\textwidth}
			\centering
			\includegraphics[width=1\textwidth]{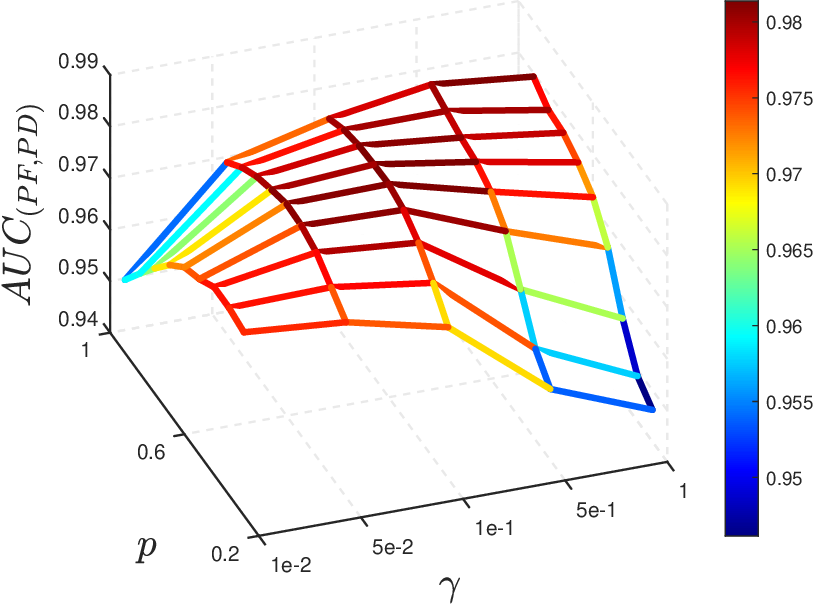}
			\caption{Urban}
		\end{minipage}
		\begin{minipage}{0.195\textwidth}
			\centering
			\includegraphics[width=1\textwidth]{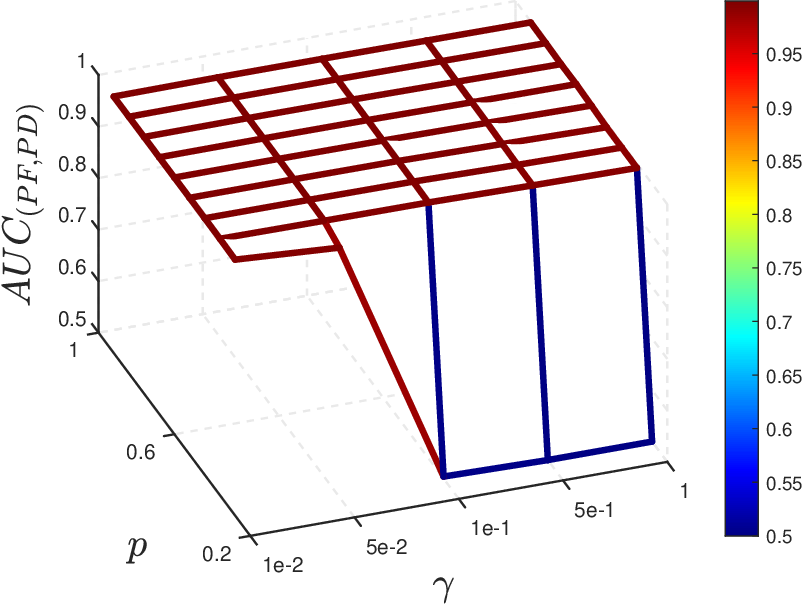}
			\caption{San Diego}
		\end{minipage}
	\end{subfigure}
	\vfill
	\caption{Surfaces of AUC values of the result by our method with different $\gamma$ and $p$.}
	\label{fig:gamma-p}
\end{figure*}

We then examine the influence of $\alpha_1$ and $\alpha_2$, which control the number of vertical and horizontal changes in the learned grouping map. Both parameters are selected from $\{1,5,10,20,30,40,50\}$. Fig. \ref{fig:alpha} reports the corresponding AUC surfaces. The performance is insensitive to moderate changes in $\alpha_1$ and $\alpha_2$: for all datasets except Beach, the variation of AUC is below 0.5\%, and the largest variation on Beach is about 2\%. We also observe that the preferred value of $\alpha_1$ tends to increase when more anomaly blocks are present. Therefore, we fix $\alpha_2=10$ and set $\alpha_1=1$ for scenes with few anomaly blocks and $\alpha_1=50$ for scenes with more anomaly blocks.
\begin{figure*}[htbp]
	\centering
	\begin{subfigure}[b]{1\linewidth}
		\begin{minipage}{0.195\textwidth}
			\centering
			\includegraphics[width=1\textwidth]{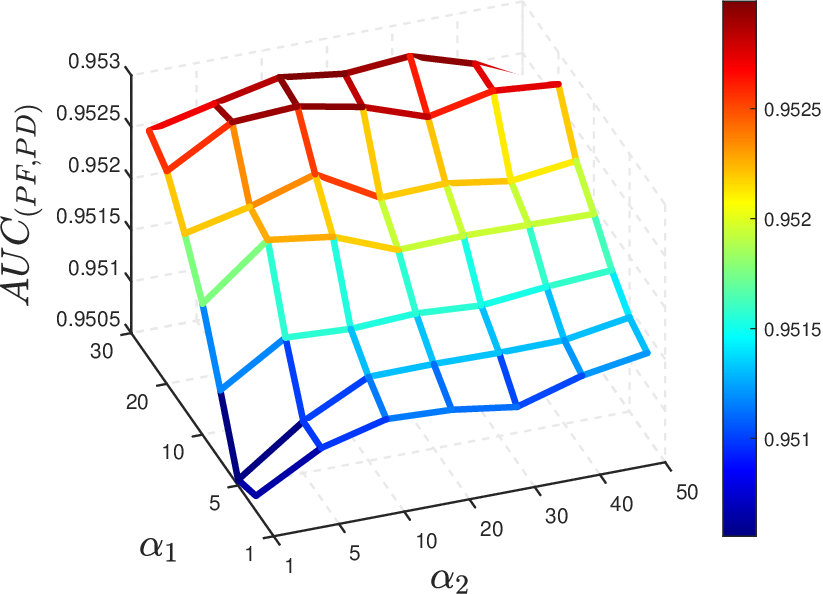}
			\caption{Airport1}
		\end{minipage}
		\begin{minipage}{0.195\textwidth}
			\centering
			\includegraphics[width=1\textwidth]{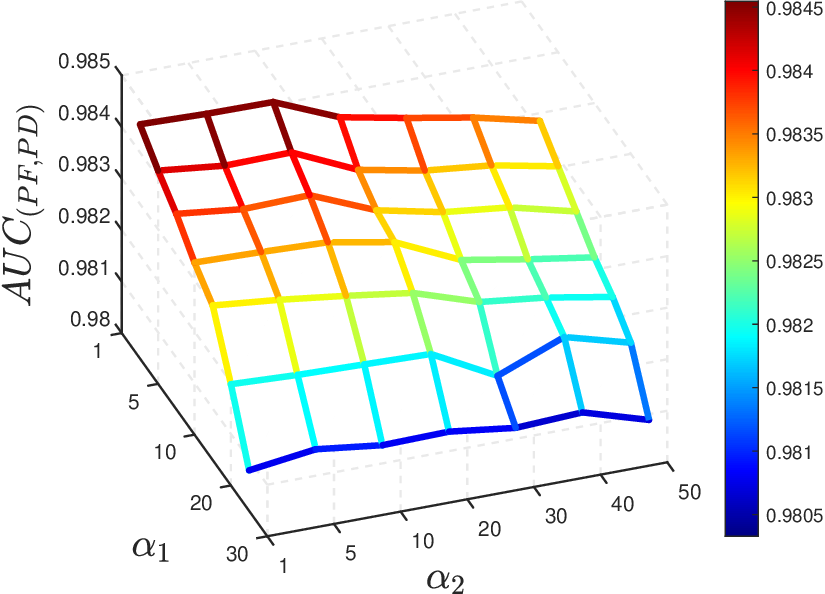}
			\caption{Airport2}
		\end{minipage}
		\begin{minipage}{0.195\textwidth}
			\centering
			\includegraphics[width=1\textwidth]{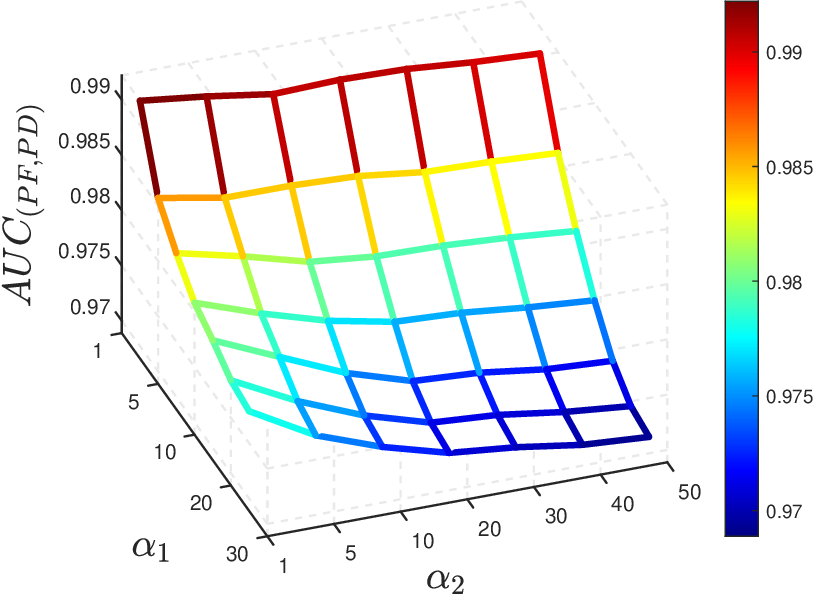}
			\caption{Beach}
		\end{minipage}
		\begin{minipage}{0.195\textwidth}
			\centering
			\includegraphics[width=1\textwidth]{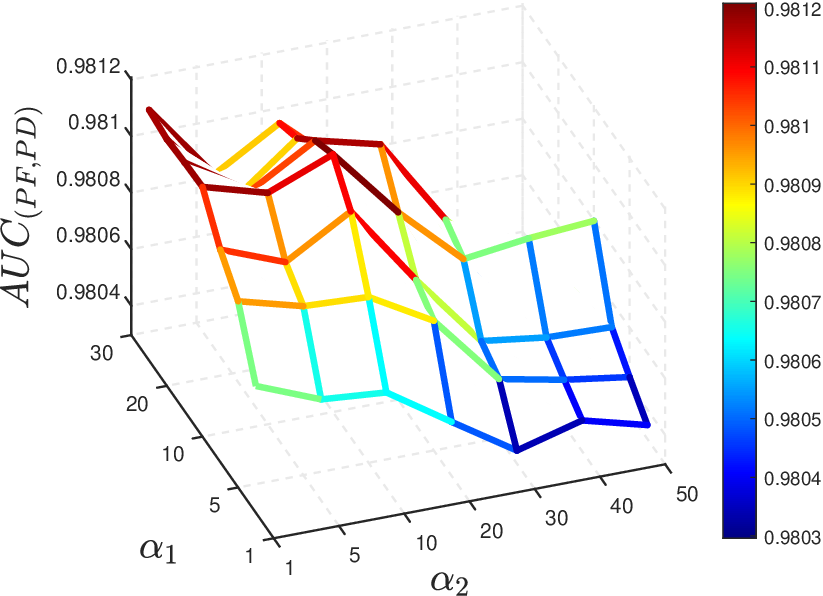}
			\caption{Urban}
		\end{minipage}
		\begin{minipage}{0.195\textwidth}
			\centering
			\includegraphics[width=1\textwidth]{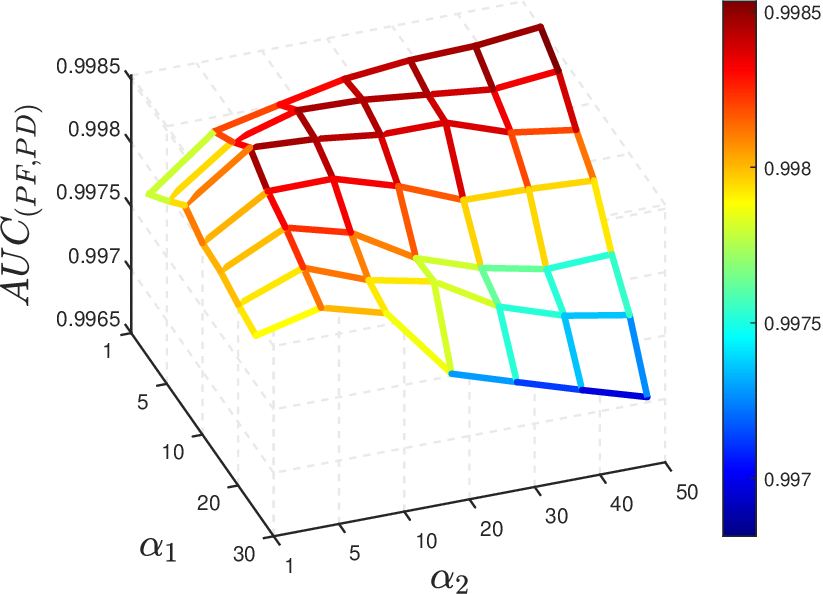}
			\caption{San Diego}
		\end{minipage}
	\end{subfigure}
	\vfill
	\caption{Surfaces of AUC  values of the result by our method with different $\alpha_1$ and $\alpha_2$.}
	\label{fig:alpha}
\end{figure*}

\subsubsection{Adaptive dimensional reduction during iterations} In our model, the large tensor $\mathcal{Z}\in\mathbb{R}^{n_1\times n_2\times n_3}$ is factorized via the t-product into two smaller tensors, $\mathcal{X}\in\mathbb{R}^{n_1\times d\times n_3}$ and $\mathcal{Y}\in\mathbb{R}^{n_2\times d\times n_3}$, where the shared dimension $d$ controls the complexity of the decomposition. In Algorithm \ref{Alg:LADMM}, $d$ is adaptively reduced at Step 11. Fig. \ref{fig:r} illustrates the evolution of the auxiliary dimension $d$ during each iteration of Algorithm \ref{Alg:LADMM} for both the spatial domain and the spectral domain. Initially $d$ equals the full tensor size but collapses sharply within a few iterations to a small stable value. After this rapid drop, $d$ remains fixed for the rest of the optimization. This early convergence of $d$ demonstrates that the algorithm immediately eliminates redundant dimensions and thereafter incurs only the low-rank cost determined by the converged value of $d$, yielding substantial savings in computation time.
\begin{figure}[htbp]
	\centering
	\begin{subfigure}[b]{1\linewidth}
		\begin{minipage}{0.495\textwidth}
			\centering
			\includegraphics[width=1\textwidth]{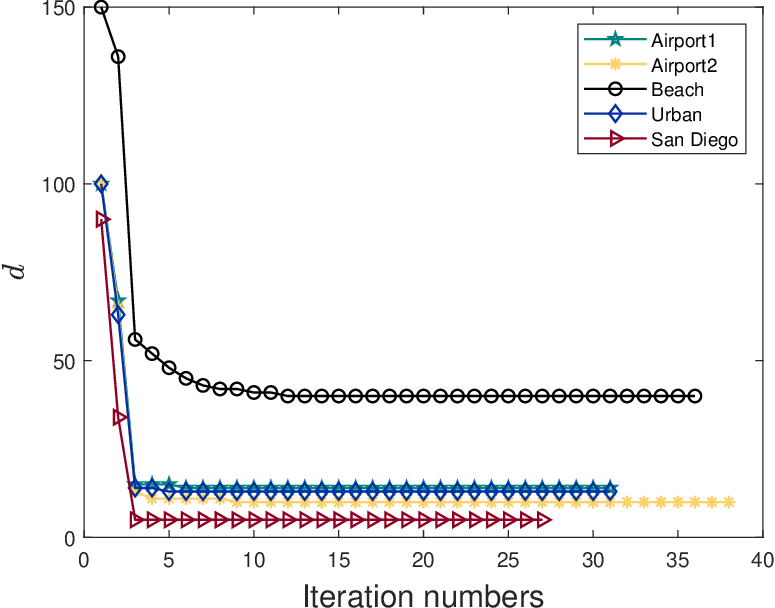}
			\caption{Spatial domain}
		\end{minipage}
		\begin{minipage}{0.495\textwidth}
			\centering
			\includegraphics[width=1\textwidth]{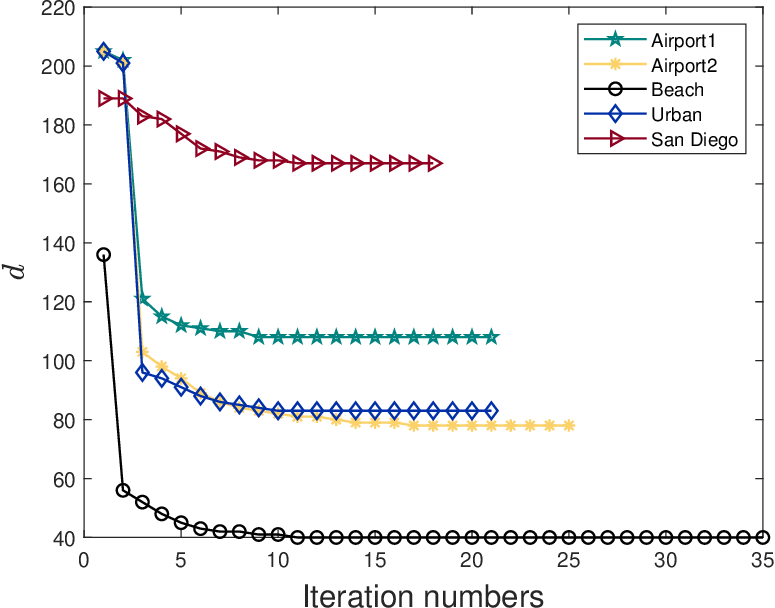}
			\caption{Spectral domain}
		\end{minipage}
	\end{subfigure}
	\vfill
	\caption{Variation of $d$ values across iterations for Algorithm \ref{Alg:LADMM}, shown separately for the spatial and spectral domains.}
	\label{fig:r}
\end{figure}

\subsubsection{Effects of automatic anomaly grouping and spectral--spatial fusion}
Table \ref{Tab:SS} verifies the effects of AAG and spectral--spatial fusion by reporting the $AUC$ values of different detection maps. Algorithm \ref{Alg:LADMM} produces $\mathcal{E}_{spa}$ and $\varTheta_{spa}$ in the spatial domain, and $\mathcal{E}_{spe}$ and $\varTheta_{spe}$ in the spectral domain. The anomaly magnitude maps are computed as
$E_{spa}(i,j)=\sqrt{\sum_k|\mathcal{E}_{spa}(i,j,k)|^2}$ and
$E_{spe}(i,j)=\sqrt{\sum_k|\mathcal{E}_{spe}(i,j,k)|^2}$.
The fused maps are obtained by element-wise multiplication, i.e.,
$E_{fus}=E_{spa}\odot E_{spe}$ and
$\varTheta_{fus}=\varTheta_{spa}\odot \varTheta_{spe}$.

In both spatial and spectral domains, $\varTheta$ consistently outperforms $E$ on all datasets. This result indicates that the learned grouping map provides a more spatially coherent anomaly response than the anomaly magnitude map alone. Combined with the parameter analysis in Fig. \ref{fig:alpha}, the results also show that finite $\alpha_1$ and $\alpha_2$ improve anomaly modeling compared with the limiting case $\alpha_1,\alpha_2\to\infty$, where Theorem \ref{Thm:LOP} reduces the AAG penalty to the conventional $\|\cdot\|_{2,1}$ form.

The fusion results further demonstrate the benefit of combining spectral and spatial information. Both $E_{fus}$ and $\varTheta_{fus}$ generally outperform their single-domain counterparts, with only minor exceptions. In particular, $\varTheta_{fus}$ achieves the highest AUC on four of the five datasets, and on Airport1, it is only 0.0029 lower than the best single-domain result. These observations confirm that AAG and spectral--spatial fusion are both important to the final performance of GSAA-SS.

\begin{table}[htbp]
	\centering
	\caption{Comparison of AUC values of different anomaly detection maps.}
	\label{Tab:SS}
	\begin{tabular}{c cc cc cc}
		\toprule
		\multirow{2}{*}{HSI} & 
		\multicolumn{2}{c}{Spatial Domain} & 
		\multicolumn{2}{c}{Spectral Domain} & 
		\multicolumn{2}{c}{Fusion Method} \\
		\cmidrule(lr){2-3} \cmidrule(lr){4-5} \cmidrule(lr){6-7}
		& $E_{spa}$ & $\varTheta_{spa}$ 
		& $E_{spe}$ & $\varTheta_{spe}$ 
		& $E_{fus}$ & $\varTheta_{fus}$ \\
		\midrule
		Airport1    & 0.9178 & 0.9559 & 0.8288 & 0.8541 & 0.9303 & 0.9530 \\
		Airport2    & 0.9198 & 0.9872 & 0.8896 & 0.9678 & 0.9476 & 0.9906 \\
		Beach       & 0.9268 & 0.9857 & 0.9720 & 0.9825 & 0.9535 & 0.9943 \\
		Urban       & 0.9299 & 0.9731 & 0.9662 & 0.9674 & 0.9545 & 0.9810 \\
		San Diego   & 0.9824 & 0.9978 & 0.9688 & 0.9752 & 0.9878 & 0.9979 \\
		\bottomrule
	\end{tabular}
	
	
\end{table}

\section{Conclusions} \label{Sec:con}

This paper presented GSAA, an HAD model based on group sparse low-rank tensor factorization with automatic anomaly grouping. In this framework, the background is modeled by imposing group sparsity on the tensor factors, which provides an efficient alternative to direct tensor rank regularization and avoids repeated large-scale SVD computations. For anomaly modeling, the AAG penalty introduces a latent grouping map, enabling spatially coherent anomaly structures to be learned adaptively from the data rather than specified by a predefined pixel-wise partition. To further leverage complementary spectral and spatial information, GSAA was extended to a spectral--spatial framework, leading to GSAA-SS. An efficient LADMM algorithm was developed to solve the resulting optimization problem, and its convergence was theoretically analyzed. Experimental results on five real hyperspectral scenes demonstrate that GSAA-SS achieves accurate anomaly detection and competitive computational efficiency relative to representative existing methods.

\ifCLASSOPTIONcaptionsoff
  \newpage
\fi



%
\nocite{HJ91, ZLLZ18, Com18}
\bibliographystyle{IEEEtran}
\bibliography{GSAA}

%




\end{document}